\documentclass{article}

 \usepackage[preprint]{neurips_2026}

\usepackage[utf8]{inputenc} 
\usepackage[T1]{fontenc}    
\usepackage[pagebackref=true]{hyperref}       
\renewcommand*{\backref}[1]{}
\renewcommand*{\backrefalt}[4]{%
    \ifcase #1 \relax
    \or
        (Cited on page #2.)%
    \else
        (Cited on pages #2.)%
    \fi
}
\usepackage{url}            
\usepackage{booktabs}       
\usepackage{makecell}       
\usepackage{amsfonts}       
\usepackage{nicefrac}       
\usepackage{microtype}      
\usepackage{xcolor}         

\title{
    {\color{oscar} LOSCAR}-SGD: Local SGD 
    with Communication-Computation Overlap\\
    and Delay-Corrected Sparse Model Averaging
}

\author{%
  Yassine~Maziane \\
  KAUST \\
  \texttt{yassine.maziane@kaust.edu.sa} \\
  \And
  Ammar~Mahran \\
  KAUST \\
  \texttt{majiedammar.mahran@kaust.edu.sa} \\
  \And
  Artavazd~Maranjyan \\
  KAUST \\
  \texttt{artavazd.maranjyan@kaust.edu.sa} \\
  \And
  Peter~Richt\'{a}rik \\
  KAUST \\
  \texttt{peter.richtarik@kaust.edu.sa}
}

\usepackage{multirow}
\usepackage[utf8]{inputenc} 
\usepackage[T1]{fontenc} 
\usepackage{calligra}
\usepackage{layout}

\usepackage{amsmath}
\usepackage{amssymb}
\usepackage{amsfonts} 
\usepackage{amsthm}

\usepackage{mathtools}
\usepackage{thmtools} 
\usepackage{thm-restate}

\usepackage{latexsym}

\usepackage{xspace}
\usepackage{xcolor} 
\usepackage{graphicx}

\usepackage{algorithmic}
\usepackage{algorithm}

\usepackage{hyperref}  
\usepackage{url}           
\usepackage{booktabs} 
\usepackage{microtype} 

\hypersetup{colorlinks=false, urlcolor=magenta}

\usepackage{cleveref}

\usepackage{bbm}
\usepackage{bm} 

\usepackage{nicefrac} 
\usepackage{adjustbox}
\usepackage{threeparttable}

\usepackage{import, ifthen}

\usepackage{tcolorbox}
\usepackage{pifont}
\definecolor{mydarkred}{RGB}{192,25,25}
\definecolor{mydarkgreen}{RGB}{25,192,25}
\definecolor{mydarkblue}{RGB}{25,25,192}

\definecolor{RedOrange}{cmyk}{0,0.77,0.87,0}
\definecolor{Orchid}{cmyk}{0.32,0.64,0,0}

\definecolor{checkmarkcolor}{rgb}{0.3,0.25,0.2}
\definecolor{crossmarkcolor}{rgb}{0.3,0.25,0.2}
\newcommand{\mycheckmark}{\textbf{\color{checkmarkcolor}\ding{52}}}
\newcommand{\mycrossmark}{{\color{crossmarkcolor}\ding{56}}}

\newcommand{\norm}[1]{{\left\| #1 \right\|}}

\newcommand{\rbr}[1]{\left(#1\right)} 
\newcommand{\nbr}[1]{\left\{#1\right\}} 

\newcommand{\anglebr}[1]{\left\langle#1\right\rangle} 

\newcommand{\inner}[2]{\anglebr{#1, #2}}

\newcommand{\cD}{\mathcal{D}}

\newcommand{\cF}{\mathcal{F}}

\newcommand{\bbN}{\mathbb{N}}

\newcommand{\R}{\mathbb{R}}  

\newcommand{\del}[1]{}

\newcommand{\eqdef}{\coloneqq}

\newcommand{\Exp}[1]{{\mathbb{E}} \left[ #1 \right]} 
\newcommand{\ExpCond}[2]{{\mathbb{E}}\left[\left.#1\right\vert#2\right]}
\newcommand{\ExpSub}[2]{{\mathbb{E}}_{#1}\left[#2\right]}

\newcommand{\Proj}{{\rm Proj}} 

\newcommand{\pr}[1][]{
  \ifthenelse { \equal{#1}{} }
  { \ensuremath{\mathrm{P}} }
  { \ensuremath{\mathrm{P}\left(#1\right)} }
}

\usepackage[colorinlistoftodos,bordercolor=orange,backgroundcolor=orange!20,linecolor=orange,textsize=scriptsize]{todonotes}

\newcommand{\ppeter}[1]{}

\newcommand{\hidesolutions}[1]{} 

\definecolor{pastelred}{HTML}{FFB3BA}
\definecolor{pastelgreen}{HTML}{BDECB6}
\definecolor{pastelblue}{HTML}{AEC6CF}

\definecolor{thmbackground}{rgb}{240, 240, 255}

\declaretheoremstyle[
  spaceabove = 10pt, 
  spacebelow = 10pt,
  postheadspace=\newline,  
    postheadhook={\textcolor{black}{\rule[.6ex]{\linewidth}{0.4pt}}\\},  
  headfont = \bfseries,
  bodyfont = \normalfont\itshape,
  notefont = \mdseries\bfseries,
  notebraces = (),
 mdframed={
            frametitlebackgroundcolor=blue!60,
            backgroundcolor=red!0!white, 
           linecolor=black!100!white, 
            linewidth=0.4pt,     
            roundcorner=5pt,                     
            innertopmargin=6pt,
            innerbottommargin=20pt,   
            skipabove=2\parsep, 
            skipbelow=2\parsep } 
]{mythmstyle}

\declaretheorem[
  name=Corollary,
  style=mythmstyle,
  numberwithin=section
]{corollary}

\declaretheoremstyle[
  spaceabove = 10pt, 
  spacebelow = 10pt,
  postheadspace=\newline,  
    postheadhook={\textcolor{black}{\rule[.6ex]{\linewidth}{0.4pt}}\\},  
  headfont = \bfseries,
  bodyfont = \normalfont\itshape,
  notefont = \mdseries\bfseries,
  notebraces = (),
 mdframed={
            backgroundcolor=green!0!white, 
            linecolor=black!100!white, 
            innertopmargin=6pt,
            roundcorner=5pt, 
            innerbottommargin=20pt, 
            skipabove=2\parsep, 
            skipbelow=2\parsep } 
]{mylemmastyle}

\declaretheorem[
  name=Lemma,
  style=mylemmastyle,
  numberwithin=section
]{lemma}

\declaretheoremstyle[
  spaceabove = 10pt, 
  spacebelow = 10pt,
  postheadspace=\newline,  
    postheadhook={\textcolor{black}{\rule[.6ex]{\linewidth}{0.4pt}}\\},  
  headfont = \bfseries,
  bodyfont = \normalfont\itshape,
  notefont = \mdseries\bfseries,
  notebraces = (),
 mdframed={
            backgroundcolor=black!0!white, 
            linecolor=black!100!white, 
            innertopmargin=6pt,
            roundcorner=5pt, 
            innerbottommargin=20pt, 
            skipabove=2\parsep, 
            skipbelow=2\parsep } 
]{myexamplestyle}

\theoremstyle{plain}
\newtheorem{assumption}{Assumption}\numberwithin{assumption}{section}
\numberwithin{claim}{section}
\numberwithin{fact}{section}
 \numberwithin{exercise}{section}

\theoremstyle{definition}
\numberwithin{definition}{section}

\definecolor{alggray}{gray}{0.35} 
\newcommand{\algname}[1]{{\color{alggray}\small\sf#1}\xspace}

\definecolor{oscar}{RGB}{145,105,25}
\newcommand{\loscar}{\algname{\color{oscar} LOSCAR}}
\newcommand{\loscarsgd}{\loscar-\sgd}

\newcommand{\minibatch}{\algname{Minibatch~SGD}}
\newcommand{\minibatchtitle}{Minibatch~SGD}

\newcommand{\sgd}{\algname{SGD}}

\newcommand{\fedavg}{\algname{FedAvg}}
\newcommand{\fedprox}{\algname{FedProx}}
\newcommand{\localsgd}{\algname{Local~SGD}}

\usepackage{thmtools}
\usepackage{thm-restate}
\usepackage{cleveref}
\crefname{assumption}{assumption}{assumptions}
\crefname{theorem}{theorem}{theorems}
\crefname{corollary}{corollary}{corollaries}
\crefname{lemma}{lemma}{lemmas}

\definecolor{linkcolor}{rgb}{0.45,0,0.05} 
\hypersetup{ %
  colorlinks=true,
  linkcolor=linkcolor,
  citecolor=linkcolor,
  filecolor=linkcolor,
  urlcolor=linkcolor,
}

\begin{document}

\maketitle

\begin{abstract}
    Communication is a major bottleneck in distributed learning, especially in large-scale settings and in federated learning environments with slow links. Three standard ways to reduce this cost are communication compression, local training, and communication-computation overlap. Methods that combine these ingredients are used in practice and have been found to be effective for large-scale training, but there is little theory for methods that combine all three. We study a heterogeneous-compute setting in which different workers may take different numbers of local steps, and we propose \loscarsgd, a \localsgd method that communicates only a sparse subset of model coordinates and continues optimizing while communication is in flight. A key ingredient is a delay-corrected merge rule that incorporates delayed synchronized information without discarding the progress made during the overlap phase. We give convergence guarantees for smooth non-convex objectives and show how sparsity, overlap, and worker heterogeneity affect the rate. To the best of our knowledge, this is the first theory for this combination of ingredients. Experiments further show that communication-computation overlap reduces training time and that the delay-corrected merge outperforms naive overwriting.
\end{abstract}
\section{Introduction}\label{sec:introduction}
Distributed learning enables the training of large models across many workers.
The most common setup is \textit{data parallelism}, where the data is split across workers, each worker keeps a copy of the model, and the workers communicate to train one shared model \citep{goyal2017accurate, li2020pytorch, zhao2023pytorch}.
A simple baseline in this setting is synchronized \minibatch: at every iteration, each worker computes one stochastic gradient, the gradients are averaged, and one global update is taken.
This is simple, but it requires communication at every iteration.
As models become larger and the number of workers grows, communication can become the main factor limiting training speed.
Federated learning is an especially clear example of this issue, since communication often happens over the internet and is much slower than local computation \citep{mcmahan2017communication, kairouz2021advances, wang2021field}.

There are several common ways to reduce this communication cost.
One is \textbf{local training}: instead of synchronizing after every stochastic-gradient step, workers perform several local \sgd steps before they communicate.
Because the local steps are simple \sgd steps, this approach is often called \localsgd \citep{zinkevich2010parallelized}, and in the federated learning literature it became popular through \fedavg \citep{mcmahan2017communication}.
Another idea is \textbf{communication-computation overlap}.
In standard synchronized methods such as \minibatch and plain \fedavg-style local training, workers may sit idle while messages are being sent, aggregated, and returned.
With overlap, workers keep optimizing locally while communication is in flight, which can hide part of the communication delay and reduce training time \citep{wang2020overlaplocalsgd, sun2024co2, kale2025eager}.
A third idea is to reduce the amount of communicated information through \textbf{compression}, for example via quantization or sparsification \citep{alistarh2017qsgd, wangni2018gradient}.
In local training, workers often synchronize model weights by averaging them; with sparsification, only a subset of model coordinates is synchronized \citep{fournier2024wash, beton2025sparta}.

These techniques are each well studied in both theory and practice.
There is also practical evidence that combining them can work well \citep{douillard2025streaming, ajanthan2026asyncmesh}.
However, to the best of our knowledge, there is no theory that covers the full combination of local training, sparse synchronization, and communication-computation overlap.
The problem is further complicated when workers have different computation speeds, as in heterogeneous clusters or federated settings, because different workers may perform different amounts of local work in the same amount of time.
A natural way to handle this is to let different workers take different numbers of local steps so that they remain roughly aligned in time \citep{li2020federated,
maranjyan2025ata, maranjyan2025gradskip}.

In this work, we study the data-homogeneous version of this setting and propose \loscarsgd, a \localsgd-type method with infrequent sparse model averaging, communication-computation overlap, and worker-specific local-step counts.
We also study a delay-corrected merge rule that preserves the progress made during the overlap phase instead of discarding it through a naive averaging overwrite.

\Cref{tbl:intro_methods} summarizes representative methods from the literature that incorporate at least one of the ingredients discussed above and have accompanying theoretical analysis.

\begin{table*}[t]
\centering
\setlength{\tabcolsep}{5pt}
\begin{tabular}{@{}ccccc@{}}
    \toprule
    Methods & \makecell{Local\\steps} & Compression & \makecell{Comm./comp.\\overlap} & \makecell{Worker-specific\\local-step counts} \\
    \midrule
    \minibatch & \mycrossmark & \mycrossmark & \mycrossmark & \mycrossmark \\
    \cmidrule(r){1-1}
    \makecell[c]{\localsgd, \fedavg\\{\footnotesize\citep{zinkevich2010parallelized,mcmahan2017communication}}} & \mycheckmark & \mycrossmark & \mycrossmark & \mycrossmark \\
    \cmidrule(r){1-1}
    \makecell[c]{\algname{SPARTA}, \algname{LoCoDL}\\{\footnotesize\citep{beton2025sparta, condat2025locodl}}} & \mycheckmark & \mycheckmark & \mycrossmark & \mycrossmark \\
    \cmidrule(r){1-1}
    \makecell[c]{\algname{Overlap-Local-SGD}, \algname{CO2}\\{\footnotesize\citep{wang2020overlaplocalsgd,sun2024co2}}} & \mycheckmark & \mycrossmark & \mycheckmark & \mycrossmark \\
    \cmidrule(r){1-1}
    \makecell[c]{\fedprox / \algname{GradSkip}\\{\footnotesize\citep{li2020federated,maranjyan2025gradskip}}} & \mycheckmark & \mycrossmark & \mycrossmark & \mycheckmark \\
    \cmidrule(r){1-1}
    \makecell[c]{\loscarsgd (\Cref{alg:general})} & \mycheckmark & \mycheckmark & \mycheckmark & \mycheckmark \\
    \bottomrule
\end{tabular}
\caption{
    Comparison of representative distributed learning methods with theoretical analysis related to our setting. 
    The columns indicate whether a method explicitly incorporates local steps, compression, communication-computation overlap, and worker-specific local-step counts.
    A checkmark indicates the presence of the corresponding feature.
    To the best of our knowledge, our method is the first analyzed method that combines all four features.}
\label{tbl:intro_methods}
\end{table*}
%
%
\subsection{Contributions}\label{sub:contributions}
\begin{itemize}
    \item
    We introduce a general \localsgd framework (\Cref{alg:general}) for data-homogeneous distributed optimization that combines four ingredients in one model:
    local training, sparse model averaging, communication-computation overlap, and worker-specific local-step counts.

    \item
    We propose a delay-corrected merge rule for sparse synchronization.
    When delayed synchronized information arrives, the worker should not discard the local progress it made during communication.
    Instead of simply overwriting coordinates with a delayed average, the rule keeps the overlap progress and corrects only the synchronization disagreement.

    \item
    In \Cref{sec:theory}, we prove convergence guarantees for smooth non-convex objectives.
    The bounds show how sparsity, overlap, and heterogeneity affect convergence.
    They also show that the extra error caused by overlap-induced staleness appears as a controlled higher-order term.


    \item
    Experiments in \Cref{sec:experiments} show that communication-computation overlap reduces training time, and that the delay-corrected merge is better than naive overwriting.
\end{itemize}
\subsection{Related work}\label{sub:related_literature}
\paragraph{Local training}
Local training reduces communication by performing multiple stochastic-gradient steps between synchronizations.
It was popularized in federated learning by \fedavg \citep{mcmahan2017communication}, building on earlier distributed-training ideas \citep{zinkevich2010parallelized, povey2014parallel, moritz2015sparknet}.
Federated-learning analyses in the heterogeneous-data regime include \citep{li2020onfedavg, karimireddy2020scaffold, mishchenko2022proxskip, douillard2023diloco}, while for our setting the most relevant theory is the homogeneous-data \localsgd line \citep{yu2019parallel, haddadpour2019localsgd, wang2021cooperative, woodworth2020local}; for broader background, see \citet{malinovsky2022variance}.

\paragraph{Compression and sparsification}
Another standard way to reduce communication is to transmit compressed information.
Classical work studies gradient compression, sparsification, and quantization \citep{alistarh2017qsgd, wangni2018gradient, lin2020deepgradientcompression, beznosikov2023biased}, but for methods based on periodic model averaging the closer references are those that sparsify the communicated model itself \citep{beton2025sparta, filippova2025partial, douillard2025streaming, condat2025locodl, condat2026bicolor}.

\paragraph{Communication-computation overlap}
Communication-computation overlap is one of the main ingredients of our method.
Instead of leaving workers idle while messages are sent, aggregated, and returned, overlap lets them continue taking useful local updates while communication is in flight, which can reduce training time.
This idea is studied in overlap-based local-training and distributed-training methods such as \citep{wang2020overlaplocalsgd, sun2024co2, kale2025eager, douillard2025streaming}.

\paragraph{Asynchronous methods}
Asynchronous \sgd also avoids worker idling, but it does so by removing round-level synchronization altogether \citep{agarwal2011distributed, recht2011hogwild, maranjyan2025thesis}.
The motivation is therefore similar to overlap, but the mechanism is different.
Our method is more structured: it keeps explicit rounds and a shared server message at the end of each round, while allowing controlled staleness through delayed sparse synchronization.
Because only part of the model may be synchronized after a delay, the method is closer to asynchronous training than plain \localsgd, while still remaining more organized than fully asynchronous \sgd.
This kind of structure is useful to have, and it can even be necessary for optimal asynchronous methods, as shown in recent asynchronous-\sgd theory \citep{tyurin2023optimal, maranjyan2025mindflayer, maranjyan2025ringmaster, maranjyan2026ringleader, mahran2026rescaled, sadiev2026ringmaster_lmo, tovmasyan2026rennala_mvr}.
In this sense, our method sits between \localsgd and asynchronous \sgd.
The closest comparison in this direction is \citet{tyurin2026birch}, which combines local training with asynchronous model averaging but does not include model sparsification.
\section{Problem setup}\label{sec:problem_setup}
We consider the stochastic optimization problem
\begin{equation}\label{eq:objective}
    \min_{x \in \R^d} \nbr{ f(x) \eqdef \ExpSub{\xi \sim \cD}{F(x;\xi)} } ,
\end{equation}
where $x \in \R^d$ is the model parameter,
$\xi$ is a random sample drawn from a distribution $\cD$,
$F(x;\xi)$ is the sample loss,
and $f$ is the population objective.

We work in a distributed, data-homogeneous setting with \(n\) workers.
Each worker keeps its own copy of the model and draws data from the same distribution \(\cD\).
Therefore, every worker can in principle solve \eqref{eq:objective} on its own.
We still use multiple workers because the goal is to speed up training by parallelizing computation.

We study first-order stochastic optimization, so the basic local computation at each worker is the evaluation of a stochastic gradient.
We write \(g_i(x,\xi_i)\) for the stochastic gradient computed by worker \(i\) at point \(x\) using a fresh random sample \(\xi_i\).
Our goal is to measure progress over time.
For that reason, we also need a simple model for computation time and communication time.
\begin{assumption}[Worker timing model]\label{ass:worker_times}
    For each worker \(i \in \{1,\dots,n\}\), one stochastic-gradient computation takes \(\tau_i\) seconds, where \(\tau_i\) is a positive integer.
    Let
    \[
        \tau \eqdef \operatorname{lcm}(\tau_1,\dots,\tau_n)
    \]
    be the least common multiple of these computation times.
    We also assume that one communication phase, from the moment workers send their messages until the aggregated message is available again at the workers, takes \(\zeta\) seconds, where
    \[
        \zeta \in \{0,\tau,2\tau,\dots\}.
    \]
\end{assumption}
\Cref{ass:worker_times} is a modeling simplification.
In a real system, computation times are not exact integers and communication delays are not perfectly constant.
We do not use this assumption because we believe real systems behave in such an exact way.
We use it only as a clean abstraction that puts all workers on a common time grid.
Since every \(\tau_i\) divides \(\tau\), every \(\tau\) seconds is a time at which all workers can be viewed on the same clock.
This lets us define rounds and count local steps without introducing much heavier timing notation.
What matters for our analysis is not the integer-valued assumption itself.
What matters is that different workers can complete different amounts of local work in the same amount of time, that communication takes a nonzero amount of time, and that workers may continue taking local steps during communication.
These are the effects we want to capture.
In practice, one can approximate this model by choosing explicit round boundaries in time.
Then all local work completed before a boundary is assigned to the current round, and communication is handled after that boundary.
If some worker is slightly early or slightly late, the system can either wait briefly for the current local step to finish or ignore a partially completed step.
The same idea can be used during communication: one can assign a fixed communication window to the round and absorb small timing mismatches inside that window.

We make the following standard assumptions on the objective:
\begin{assumption}[Lower boundedness]\label{ass:lower_bounded}
    The objective $f:\R^d \to \R$ is bounded from below, i.e.,
    there exists $f^\star \in \R$ such that
    \[
        f(x) \ge f^\star \qquad \forall x \in \R^d.
    \]
\end{assumption}
\begin{assumption}[Smoothness]\label{ass:smoothness}
    The objective $f$ is differentiable and $L$-smooth, that is,
    there exists $L>0$ such that
    \[
        \|\nabla f(x)-\nabla f(y)\| \leq L\|x-y\|
        \qquad \forall x,y \in \R^d.
    \]
\end{assumption}
Randomness enters our model through the random data samples used to compute stochastic gradients, as well as the random masks sampled for the coordinate sparsifier (cf. \Cref{sec:method}).
We make the natural assumption that these sources of randomness are independent.
\begin{assumption}[Independent Samples]\label{ass:independence}
    The collection of all data samples $\{\xi_{i,t}^r\}$ and all sparsification masks $\{S_r\}$ across all workers, local steps, and rounds are mutually independent.
\end{assumption}
Moreover, we make the following assumptions about the stochastic gradients:
\begin{assumption}[Unbiased stochastic gradients]\label{ass:unbiased}
    For every worker $i$ and every $x \in \R^d$,
    \[
        \ExpCond{g_i(x,\xi_i)}{x} = \nabla f(x).
    \]
\end{assumption}
\begin{assumption}[Bounded variance]\label{ass:bounded_variance}
    There exists $\sigma^2 \ge 0$ such that for every worker $i$ and every $x \in \R^d$,
    \[
        \ExpCond{\|g_i(x,\xi_i)-\nabla f(x) \|^2}{x} \leq \sigma^2.
    \]
\end{assumption}
\begin{assumption}[Bounded second moment]\label{ass:bounded_second_moment}
    There exists $G>0$ such that for every worker $i$ and every $x \in \R^d$,
    \[
        \ExpCond{ \|g_i(x,\xi_i)\|^2 }{ x } \leq G^2.
    \]
\end{assumption}
\Cref{ass:unbiased,ass:bounded_variance,ass:bounded_second_moment} are written worker by worker, but the constants \(L\), \(\sigma^2\), and \(G^2\) do not depend on \(i\) because all workers optimize the same objective in the data-homogeneous setting.

The next notation is only needed in the proofs.
In round \(r\), worker \(i\) performs \(H_i\) local steps, where \(H_i\) is defined later in \eqref{eq:H_i}.
For round \(r\) and local step index \(t\), let \(w_{i,t}^r\) denote the point at which worker \(i\) evaluates its \(t\)-th stochastic gradient in that round.
We also define the active worker set at local time \(t\) by
\begin{equation}\label{eq:A_t-89yf908df}
    A_t \eqdef \{i\in\{1,\dots,n\}: H_i>t\},
    \qquad
    m_t \eqdef |A_t|.
\end{equation}
So \(A_t\) is exactly the set of workers that still perform a local step at local time \(t\).

With the sigma-field
\[
    \cF_{r,t}
    \eqdef
    \sigma\!\left(
    \{x_i^r\}_{i=1}^n,\;
    \{w_{i,s}^r,\xi_{i,s}^r:\ i\in A_s,\ s=0,\dots,t-1\},\;
    \{w_{i,t}^r:\ i\in A_t\}
    \right) ,
\]
This sigma-field contains the history up to local time \(t\) in round \(r\), except for the fresh samples drawn at that time.

%
Under \Cref{ass:independence,ass:unbiased}, conditioned on \(\cF_{r,t}\), the random vectors
\[
    g_i(w_{i,t}^r,\xi_{i,t}^r)-\nabla f(w_{i,t}^r),
    \qquad i\in A_t ~,
\]
are independent and have conditional mean zero.
%
%
\section{Local SGD with communication-computation overlap and infrequent sparse model averaging}\label{sec:method}
Here we describe the method studied in this paper.
It combines three ideas:
local SGD, sparse model communication, and communication-computation overlap.
The main goal is to use all the available worker computation.
In particular, workers should keep taking local SGD steps instead of sitting idle while communication is happening.

As explained in \Cref{sec:problem_setup}, the workers may have different computation times, and \(\tau\) is the common clock used to define rounds.
The method uses two timing parameters.
The integer \(M \ge 1\) controls how long workers compute locally before communication starts.
More precisely, the first phase of each round lasts \(M\tau\) seconds.
After that, the workers communicate sparse model information with the server.
This communication takes \(\zeta\) seconds.
During those \(\zeta\) seconds, workers keep computing locally.
When the server response arrives, each worker merges the delayed sparse average with its current local model and starts the next round.

The full method, \loscarsgd, is summarized in \Cref{alg:general}.
\begin{algorithm}[thbp]
\caption{\loscarsgd ({\color{oscar} L}ocal \sgd with {\color{oscar}O}verlapped {\color{oscar} S}parse {\color{oscar} C}orrected {\color{oscar} A}ve{\color{oscar} R}aging)}
\label{alg:general}
\begin{algorithmic}[1]
    \STATE Input: Number of workers $n$;  worker compute times $\tau_1,\dots,\tau_n \in \mathbb N$; $\tau \gets \operatorname{lcm}(\tau_1,\dots,\tau_n)$ (least common multiple of the worker compute times)
    \STATE integer $M \ge 1$ (local-computation parameter)
    \STATE communication duration $\zeta \in \{0,\tau,2\tau,\dots\}$
    \STATE sparsification level $K \in \{1,\dots,d\}$
    \STATE stepsize $\eta > 0$
    \STATE $N_i \gets M\tau/\tau_i$ (\#local steps before communication starts) for each worker $i$
    \STATE $Q_i \gets \zeta/\tau_i$ (\# local steps during communication) for each worker $i$
    \STATE Initialize $x_i^0 \in \mathbb R^d$ for $i=1,\dots,n$
    \FOR{$r=0,1,2,\dots$}
        \STATE Sample a common Rand-$K$ mask $S_r \subseteq [d]$
        \FORALL{workers $i=1,\dots,n$ \textbf{in parallel}}
            \STATE Starting from $x_i^r$, run $N_i$ local SGD steps with stepsize $\eta$ to obtain $y_i^r$
            \STATE Form compressed message $m_i^r \gets \mathcal C_{S_r}(y_i^r)$
            \STATE Send $m_i^r$ to the server
            \STATE While communication is in flight, run $Q_i$ further local SGD steps, using stepsize $\eta$, from $y_i^r$ to obtain $z_i^r$
        \ENDFOR
        \STATE Server computes
        $$
            m^r \gets \frac1n \sum_{i=1}^n m_i^r
        $$
        \STATE Server broadcasts $m^r$ to all workers
        \FORALL{workers $i=1,\dots,n$ \textbf{in parallel}}
            \STATE Merge delayed average with latest local model:
            \[
                x_{i,j}^{r+1} \gets
                    \begin{cases}
                        m_j^r + (z_{i,j}^r-y_{i,j}^r), & j \in S_r,\\
                        z_{i,j}^r, & j \notin S_r
                    \end{cases}
                \qquad \forall  j \in [d]
            \]
        \ENDFOR
    \ENDFOR
\end{algorithmic}
\end{algorithm}

We now walk through the main parts of one round in more detail.
\paragraph{One round in detail} 
\label{par:local_computation_within_one_round}
 
At the beginning of round $r \in \bbN_0$, worker $i$ holds their local model $x_i^r \in \R^d$.
During the first \(M\tau\) seconds of the round, worker $i$ performs 
\begin{equation}\label{eq:N_i}
    N_i \eqdef \frac{M\tau}{\tau_i}
\end{equation}
local SGD steps and reaches an intermediate iterate $y_i^r$.
This is the model that worker $i$ prepares for communication.
Each worker compresses \(y_i^r\) and sends the resulting sparse message to the server.

While this communication is taking place, the workers do not remain idle.
Instead, they continue running SGD locally.
In particular, worker $i$ runs SGD for another 
\begin{equation}\label{eq:Q_i}
    Q_i \eqdef \frac{\zeta}{\tau_i}
\end{equation}
steps over the next \(\zeta\) seconds and reaches a newer iterate \(z_i^r\).
So \(y_i^r\) is the model at the moment communication starts, while \(z_i^r\) is the model held by worker \(i\) when the server message returns.

Hence, in one round, worker \(i\) performs a total of
\begin{equation}\label{eq:H_i}
    H_i \eqdef N_i+Q_i \overset{\eqref{eq:N_i}+\eqref{eq:Q_i}}{=} \frac{M\tau+\zeta}{\tau_i}
\end{equation}
local SGD steps. 
By construction, \(N_i\), \(Q_i\), and \(H_i\) are integers, and \(H_i \geq 1\).

At the same time, the server averages the sparse messages received from all workers and sends the result back.
When this message arrives, each worker combines it with its current local model \(z_i^r\) to form the next starting point \(x_i^{r+1}\).

The algorithm is written using \(N_i\) and \(Q_i\) because this makes the round structure easy to state.
In practice, one does not need to precompute these numbers and force workers to stop after exactly \(N_i\) or \(Q_i\) steps.
Instead, workers can simply keep running local SGD, and the system can decide when communication starts.
If \(\tau\) is known, then starting communication after \(M\tau\) seconds gives the same behavior as \Cref{alg:general}.
If \(\tau\) is not known, one can instead use a time-based synchronization interval as a hyperparameter.
As discussed under \Cref{ass:worker_times}, the integer-time assumption is a modeling device for the analysis, not a literal implementation requirement.

\paragraph{Compressed synchronization} 
\label{par:compressed_synchronization}

We use sparsification to reduce the amount of data communicated in each round.
This helps keep the communication window short, which in turn reduces how stale the communicated model becomes before the merged model is formed.

We fix the sparsification level $K \in [d] \eqdef \{1,\dots,d\}$.
In every round $r \in \bbN_0$, a mask $S_r \subseteq [d]$ is sampled uniformly at random among all subsets of cardinality $K$, independently of the algorithmic history up to the start of round $r$ and the stochastic samples drawn within round $r$.
The same mask \(S_r\) is used by all workers in that round, so the server receives the same set of coordinates from every worker.
We denote the level of sparsification and its complement by
\begin{equation}\label{eq:pq}
    p \eqdef \frac{K}{d},
    \qquad
    q \eqdef 1-p ~.
\end{equation}

We define the compressor associated with $S_r$ by $\mathcal C_{S_r}(x) \eqdef \Proj_{S_r}x$, where $\Proj_{S_r}$ denotes the coordinate projection onto the coordinates in $S_r$. 
This operator keeps the coordinates in \(S_r\) and drops the rest.
Since $S_r$ is a uniform random subset of $[d]$ of cardinality $K$, each coordinate is selected with probability $p=\nicefrac{K}{d}$, and therefore
$$
    \Exp{\mathcal C_{S_r}(x)} = \Exp{\Proj_{S_r}(x)} = \frac Kd x.
$$
Worker \(i\) sends the compressed message
$$
    m_i^r \eqdef \mathcal C_{S_r}(y_i^r)=\Proj_{S_r}y_i^r.
$$
The server averages these sparse messages and broadcasts the result back to all workers:
$$
    m^r = \frac{1}{n}\sum_{i=1}^n m_i^r.
$$
Since the same mask $S_r$ is used by all workers, this is equivalent to
\begin{equation*}
    m^r = \Proj_{S_r} (\bar{y}^r),
    \qquad
    \bar y^r \eqdef \frac{1}{n}\sum_{i=1}^n y_i^r.    
\end{equation*}
The important point is that \(m^r\) is formed from the older models \(y_i^r\), not from the newer models \(z_i^r\).
When the averaged message arrives after \(\zeta\) seconds, each worker combines it with its current local model \(z_i^r\) to form the next round's starting point \(x_i^{r+1}\).
\paragraph{Merge rule}\label{par:merge_rules}
At the end of round \(r\), worker \(i\) has two relevant objects:
the newest local model \(z_i^r\), and the delayed server message \(m^r\).
The merge rule says how these two objects are combined.

One possible merge rule is to overwrite the communicated coordinates with the server value:
\begin{equation*}
    x_{i,j}^{r+1}
    =   \begin{cases}
            m_j^r, & j \in S_r ~, \\
            z_{i,j}^r, & j \notin S_r ~.
        \end{cases}
\end{equation*}

Under this rule, the coordinates in \(S_r\) are replaced by the delayed average, while the other coordinates stay equal to the most recent local iterate \(z_i^r\).
This rule is simple, but it throws away the progress made on the communicated coordinates during the overlap window.

A better alternative is to keep that progress.
This is the rule used in \Cref{alg:general}:
\begin{equation*}
    x_{i,j}^{r+1}
    =
    \begin{cases}
        \bar y_j^r + (z_{i,j}^r-y_{i,j}^r), & j\in S_r ~,\\
        z_{i,j}^r, & j\notin S_r ~.
    \end{cases}    
\end{equation*}
Under this rule, \(z_i^r\) is the baseline.
On the communicated coordinates, we add the correction \(\bar y^r-y_i^r\).
So the worker keeps the local progress \(z_i^r-y_i^r\) made during communication, while still moving the communicated coordinates toward the average model.
\paragraph{Special cases} \label{par:special_cases}
Our method combines three algorithmic ingredients:
local training, sparse communication, and communication-computation overlap.
Each of them can be turned off, yielding variants that are covered by existing methods in the literature; see \Cref{tbl:intro_methods}

Setting \(\zeta=0\) removes communication-computation overlap.
Indeed, then \(Q_i=0\) for every worker, so no local steps are taken while communication is in flight.
This does not remove communication itself; rather, it specializes the model to the case in which communication contributes no extra time that could be overlapped with computation.
Setting \(K=d\) removes sparsification, since then all coordinates are communicated.
If, in addition, all workers have the same logical step time, \(\tau_i=\tau\), then each worker takes the same number \(N_i=M\) of pre-communication local steps, which recovers the usual homogeneous \localsgd regime.
The further special case \(M=1\) yields one local step per round; together with \(K=d\) and \(\zeta=0\), this recovers synchronized \minibatch.
The choice \(\tau_i=\tau\) should be understood only as a specialization of the model to equal step counts per round, not as a literal requirement that the underlying hardware be homogeneous.

\paragraph{Connection to asynchronous SGD}
Our method is motivated by the same systems principle as asynchronous \sgd \citep{agarwal2011distributed, maranjyan2025thesis}:
workers should keep computing whenever possible instead of idling during synchronization or communication.
The main difference is that our method retains an explicit round structure.
During a round, workers continue optimizing while communication is in flight, but at the end of the round they all receive the same server message and update from it.
In asynchronous \sgd, by contrast, there is no such round-level synchronization:
updates are applied whenever a worker finishes a stochastic-gradient computation, which leads to stale information.

For this reason, our method is more structured than asynchronous \sgd, yet it still contains a controlled form of staleness.
The server message is formed from the older iterates \(y_i^r\), while by the time that message is used each worker has already advanced to \(z_i^r\).
When \(K<d\), only a subset of coordinates is synchronized in each round, which makes the method closer in spirit to asynchronous training than classical \localsgd.
When \(K=d\), this sparsity effect disappears, and the method becomes \localsgd with communication-computation overlap and delayed full-model averaging.
In this sense, the method sits between \localsgd and asynchronous \sgd.
A closely related work is that of \citet{tyurin2026birch}, which allows local training and asynchronous model averaging, but does not include model sparsification.
\section{Theory}\label{sec:theory}
Before we state our main result, we collect the aggregate quantities that appear in the rate.
\begin{equation}\label{eq:-0-0-098uy8ygugf}
    \bar N \eqdef \frac{1}{n}\sum_{i=1}^n N_i,
    \qquad
    \bar H \eqdef \frac{1}{n}\sum_{i=1}^n H_i,
    \qquad
    H_{\max} \eqdef \max_{1\leq i\leq n} H_i,
\end{equation}
\begin{equation}\label{eq:98yf9d8yf0d_90uf09d}
    S_N \eqdef \sum_{i=1}^n N_i^2,
    \qquad
    S_Q \eqdef \sum_{i=1}^n Q_i^2,
\end{equation}
These quantities summarize the amount and distribution of local computation in one round:
\(\bar N\) and \(\bar H\) are the average numbers of pre-communication and total local steps,
\(H_{\max}\) is the largest total number of local steps performed by any worker,
and \(S_N\) and \(S_Q\) are the corresponding squared aggregates for the pre-communication and overlap phases.
Finally, define
\begin{equation}\label{eq:Psi_H-defn}
    \Psi_H
    \eqdef
    \sum_{i=1}^n\sum_{t=0}^{H_i-1} t^2
    =
    \sum_{i=1}^n \frac{(H_i-1)H_i(2H_i-1)}{6}.
\end{equation}
The quantity \(\Psi_H\) captures the cumulative within-round drift created by taking multiple local SGD steps before the next synchronization.
In particular, \(\Psi_H=0\) when every worker performs only one local step per round.
We can now state the main convergence guarantee for \Cref{alg:general}.
\begin{restatable}[Convergence guarantee]{theorem}{maintheorem}\label{thm:main}
    Consider \Cref{alg:general}, initialized with
    $$
        x_i^0 = x^0 \in \R^d,
        \qquad i=1,\dots,n.
    $$
    Let \Cref{ass:lower_bounded,ass:smoothness,ass:independence,ass:unbiased,ass:bounded_variance,ass:bounded_second_moment} hold.
    Choose a stepsize \(\eta\) satisfying
    $$
        0<\eta \leq \frac{1}{8LH_{\max}},
    $$
    where \(H_{\max}\) is defined in \eqref{eq:-0-0-098uy8ygugf}.
    Choose the sparsification coefficient \(K \in \{1,\dots,d\}\), let \(q=1-\frac{K}{d}\) as in \eqref{eq:pq}, and choose parameters \(\alpha>0\) and \(\beta>0\) such that
    $$
        c \eqdef q(1+\alpha)(1+\beta) < 1.
    $$
    Further, define
    $$
        B \eqdef q(1+\beta)\left(1+\frac1\alpha\right),
        \qquad
        D \eqdef 1+\frac{q}{\beta}.
    $$
    Then for every \(R\ge 1\), the average iterate \(\bar x^r \eqdef \frac{1}{n}\sum_{i=1}^n x_i^r\), satisfies
    \begin{eqnarray}\label{eq:main_rate}
        \frac1R\sum_{r=0}^{R-1}\Exp{\|\nabla f(\bar x^r)\|^2}
        &\leq& 
        \frac{4(f(x^0)-f^\star)}{\eta\bar H R}
        +
        \frac{4L\eta\sigma^2}{n}
        \notag\\
        &&+ 
        \frac{6L^2\eta^2G^2\Psi_H}{n\bar H}
        +
        \frac{6L^2\eta^2G^2H_{\max}}{n\bar H}
        \frac{BS_N+DS_Q}{1-c},
    \end{eqnarray}
    where
    \(\bar{H}\) is defined in \eqref{eq:-0-0-098uy8ygugf}, \(S_N\) and \(S_Q\) are defined in \eqref{eq:98yf9d8yf0d_90uf09d}, and
    \(\Psi_H\) is defined in \eqref{eq:Psi_H-defn}.
\end{restatable}

\paragraph{Interpretation} 
The bound in \eqref{eq:main_rate} separates four effects.
The first term is the optimization term, which scales inversely with $\eta \bar{H} R$.
Larger effective local progress per round, as captured by $\eta \bar{H}$, reduces the required number of communication rounds $R$. 
Note, however, that the permissible stepsize $\eta$ must scale with $\nicefrac{1}{H_{\max}}$ to ensure convergence, meaning the effective progress $\eta \bar{H}$ is constrained by the system heterogeneity ratio $\nicefrac{\bar{H}}{H_{\max}}$.
We account for this trade-off explicitly in the wall-clock time complexity of \Cref{corollary:time_complexity}.

The second term is the stochastic noise term, and it retains the familiar \(\nicefrac{1}{n}\) variance reduction from averaging across workers that arises from \minibatch \citep{cotter2011better, gower2019sgd}.

The third term is the intrinsic local-drift penalty.
It is present even with full communication and zero delay, simply because workers take several local steps before they resynchronize.
The fourth term captures disagreement caused by delayed and partial synchronization.
Its \(S_N\)-part comes from the fact that the communicated model is formed only after \(N_i\) local steps, while its \(S_Q\)-part is the additional price of continuing local optimization during the overlap window.
This term is also amplified when \(K\) is smaller, since then \(q=1-\nicefrac{K}{d}\) is larger, which increases the constants \(c\), \(B\), and \(D\).
In other words, synchronizing fewer coordinates makes the residual disagreement harder to control.

Two simplifications are worth keeping in mind.
When \(\zeta=0\), we have \(Q_i=0\) for all \(i \in [n]\) and \(S_Q=0\), so the overlap-induced contribution inside the fourth term disappears.
When \(K=d\), we have \(q=0\), and hence \(c=B=0\) and \(D=1\), so the \(BS_N\) part of the fourth term vanishes and only the delay part proportional to \(S_Q\) remains.
The third term involving \(\Psi_H\) is different:
it is the intrinsic local-drift term, and it remains whenever workers take more than one local step per round.

Most importantly, all disagreement terms scale as \(\eta^2\).
Thus, the leading stochastic term has the same \(\eta\)-dependence as in standard SGD, while the extra price of local training, sparsification, and overlap appears only at higher order.

\paragraph{Special cases}
Theorem~\ref{thm:main} recovers several regimes of independent interest.

\paragraph{Special case: \minibatchtitle}
Classical \minibatch is recovered in the homogeneous setting \(\tau_i=\tau\) by taking \(M=1\), \(\zeta=0\), and \(K=d\).
Then every round contains exactly one full synchronized step, so \(\bar H=H_{\max}=1\), \(\Psi_H=0\), and the disagreement terms vanish.
Only the first two terms in \eqref{eq:main_rate} remain, and we obtain
$$
    \frac1R\sum_{r=0}^{R-1}\Exp{\|\nabla f(\bar x^r)\|^2}
    \leq
    \frac{4(f(x^0)-f^\star)}{\eta R}
    +
    \frac{4L\eta\sigma^2}{n}
    \!,
$$
which is the standard non-convex \minibatch rate \citep{cotter2011better, gower2019sgd}.

\paragraph{Special case: local SGD}
As discussed already in \Cref{par:special_cases}, taking \(\tau_i=\tau\), \(\zeta=0\), and \(K=d\) recovers homogeneous \localsgd with \(M\) local steps per round.
In that regime, the \(S_Q\)-part of the fourth term disappears because there is no overlap, and the \(BS_N\)-part disappears because there is no compression.
The only higher-order correction left in \eqref{eq:main_rate} is then the usual local-drift term involving \(\Psi_H\).
Thus, the theorem reduces to the standard picture for \localsgd: increasing \(M\) improves the leading optimization term by reducing communication frequency, but it also increases the drift within each round.
\paragraph{Time complexity}\label{par:time_complexity} 
We now translate the convergence bound into round and time complexity.
\begin{corollary}[Time complexity; proof in \Cref{proof:time_complexity}]\label{corollary:time_complexity}
    Let $\Delta \eqdef f(x^0)-f^\star$ and $X \eqdef \Psi_H + H_{\max}\frac{BS_N + DS_Q}{1-c}$. 
    Let $\hat{r}$ be sampled uniformly at random from $\{0, 1, \ldots, R-1\}$, independently of the algorithmic history.
    Under the assumptions of \Cref{thm:main}, if the total number of communication rounds satisfies
    \begin{equation*}
        R \ge c_R \left( \frac{\Delta L \sigma^2}{n \varepsilon \bar{H}} + \frac{\Delta L G \sqrt{X}}{\varepsilon^{3/2} \bar{H} \sqrt{n \bar H}} + \frac{\Delta L H_{\max}}{\varepsilon \bar{H}} \right)
    \end{equation*}
    for a sufficiently large constant $c_R > 0$, then $\Exp{\norm{\nabla f(\bar x^{\hat{r}})}^2} \le \varepsilon$.
    
    The corresponding bound on the wall-clock time complexity to achieve this guarantee is
    \begin{equation*}
        \mathcal{O}\rbr{\frac{\Delta L \sigma^2}{n \varepsilon} \cdot \tau_H + \frac{\Delta L G \sqrt{X}}{\varepsilon^{3/2} \sqrt{n \bar H}} \cdot \tau_H + \frac{\Delta L H_{\max}}{\varepsilon} \cdot \tau_H},
    \end{equation*}
    where $\tau_H \eqdef \frac{M\tau+\zeta}{\bar H} = \frac{n}{\sum_{i=1}^n \tau_i^{-1}}$ is the harmonic mean of worker computation times.
\end{corollary}
The leading term in this bound keeps the standard linear-in-\(n\) minibatch speedup and depends on worker speed only through the harmonic mean \(\tau_H\).
This leading term matches the dominant stochastic term that appears in recent asynchronous-\sgd analyses such as \citet{mishchenko2022asynchronous, koloskova2022sharper}.
The effects of local training, sparsification, and overlap appear only through the higher-order correction terms involving \(X\) and \(H_{\max}\).
This is precisely the desirable regime:
overlap can improve practical efficiency without changing the leading stochastic term, provided the additional disagreement terms remain controlled.
For the \minibatch special case above, \(\tau_H=\tau\), and the leading term in the bound of \Cref{corollary:time_complexity} reduces to the standard \(\mathcal{O}\!\left(\Delta L \sigma^2 \tau /(n\varepsilon^2)\right)\) scaling.
Thus, in the fully synchronized homogeneous regime, the time complexity is exactly the expected minibatch baseline.

\section{Experiments}\label{sec:experiments}

We evaluate the proposed overlap mechanism in a PyTorch logical-time simulator.
Workers are separate model copies inside one deterministic process, so the experiments isolate the algorithmic effects of sparse communication, overlap, delay correction, and heterogeneous worker speeds without hardware or networking noise.
The main experiments use \texttt{a9a} logistic regression with four workers, normalized features, batch size $256$, matched initialization and seeds, and worker step times $(1,2,3,6)$ unless stated otherwise.
We focus here on the sparse methods: blocking \textbf{Local Sparse}, \textbf{Overlap overwrite}, and \textbf{Overlap delay-corrected}; full metrics and ablations are deferred to the appendix.

\Cref{fig:main-experiments} summarizes the main empirical conclusions.
First, when the methods have the same round duration and communicate the same number of coordinates per round, comparing loss against rounds is sufficient: logical time and communication are common rescalings of the same index.
In this matched setting, overlap improves substantially over blocking sparse averaging, and delay correction gives the best curve.
This supports the intended mechanism: communication time is turned into useful local computation, and the delay-corrected merge preserves that progress instead of overwriting it.
Second, the sparsity level provides the expected communication tradeoff.
For the delay-corrected method, reducing the communicated fraction $p$ shifts the curve left by orders of magnitude in logical bits while retaining similar optimization behavior, with only a mild loss penalty at the most aggressive sparsity levels.
Third, when communication delay is large relative to the local compute window, the merge rule matters: delay correction clearly and consistently improves over naive overwrite.

\begin{figure}[thb]
    \centering
    \setlength{\tabcolsep}{1pt}
    \begin{minipage}{0.32\linewidth}
        \centering
        \includegraphics[width=\linewidth]{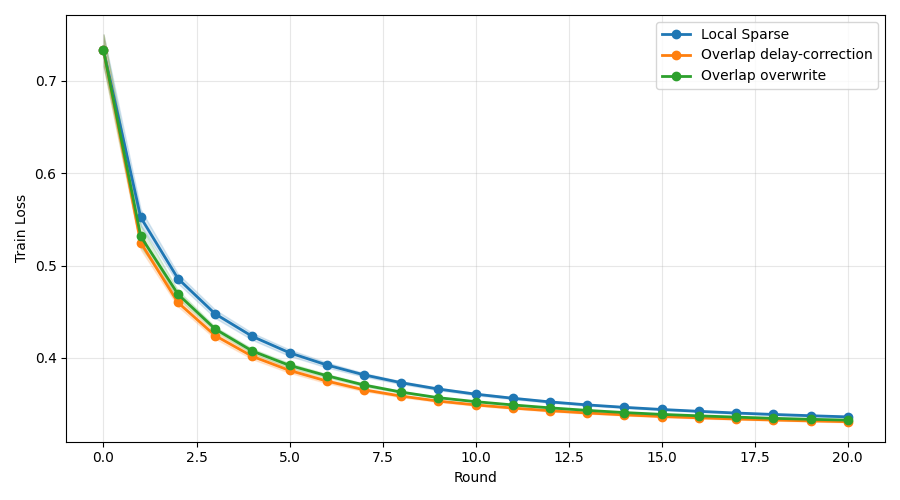}\\[-0.35em]
        {\scriptsize (a) Overlap comparison.}
    \end{minipage}\hfill
    \begin{minipage}{0.32\linewidth}
        \centering
        \includegraphics[width=\linewidth]{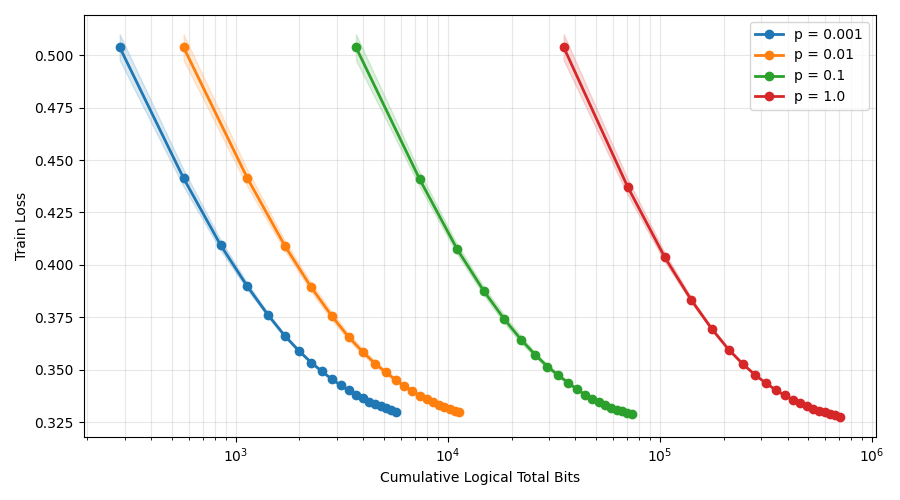}\\[-0.35em]
        {\scriptsize (b) Sparsity vs. communication.}
    \end{minipage}\hfill
    \begin{minipage}{0.32\linewidth}
        \centering
        \includegraphics[width=\linewidth]{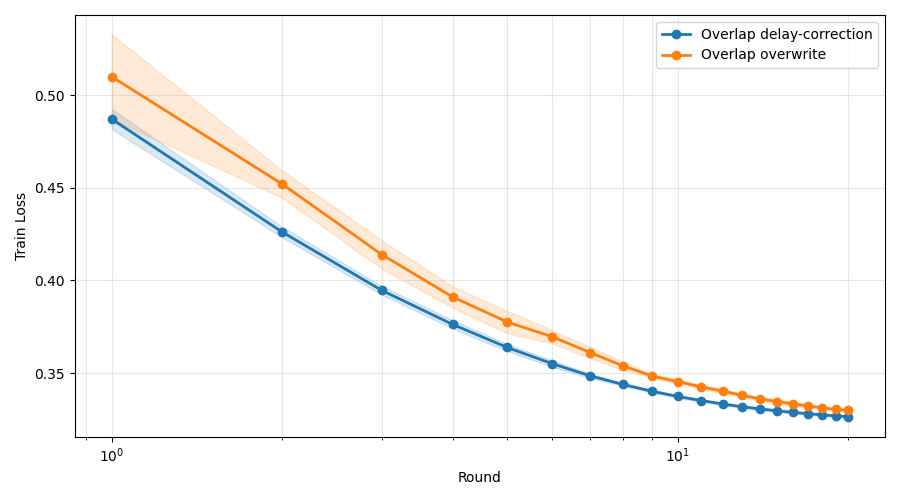}\\[-0.35em]
        {\scriptsize (c) Long-delay merge rule.}
    \end{minipage}
    \caption{Main experimental summary on \texttt{a9a}.
    (a) Overlap improves over blocking sparse averaging, and delay correction is best.
    (b) Smaller sparsity levels reduce the logical communication cost by orders of magnitude for the delay-corrected method.
    (c) Under a long communication delay, delay correction substantially outperforms overwrite.}
    \label{fig:main-experiments}
\end{figure}

The same qualitative ordering appears in additional loss, accuracy, gradient-norm, time, and communication plots.
We also ran nonconvex regularized versions of the logistic-regression experiments, but do not report them in the main paper because they did not lead to distinct conclusions.
The appendix further studies the local computation budget, communication delay, and heterogeneous-data stress tests; the latter show that very large local budgets can be harmful outside the homogeneous-data regime covered by the theory.
It also includes CIFAR-10 and Tiny ImageNet neural-network experiments, where overlap improves logical-time accuracy while processed-example plots check whether the extra steps taken during communication remain useful at a fixed amount of local computation.
\section{Conclusion}\label{sec:conclusion}
We studied local SGD with sparse model averaging, communication--computation overlap, and heterogeneous worker speeds in the data-homogeneous setting. We introduced a delay-corrected merge rule that incorporates delayed sparse synchronized information without discarding the local progress made during the overlap phase.
For smooth non-convex objectives, we proved convergence guarantees and a time complexity bound showing that the leading stochastic term retains the standard linear-in-\(n\) minibatch speedup, while the additional cost of local training, sparsification, and overlap appears through higher-order disagreement terms.
To the best of our knowledge, this is the first analysis that covers this combination of ingredients.

The experiments support this picture. Across controlled comparisons and ablations, communication--computation overlap consistently improves over blocking sparse synchronization, and the delay-corrected merge consistently improves over naive overwrite, with the largest gains when communication delay is large.
The results also show that aggressive sparsification can reduce communication by orders of magnitude with only a modest optimization penalty in the homogeneous-data regime, and that the delay-corrected variant is the most stable across different local-computation budgets.

The main limitations are that the theory is restricted to the data-homogeneous setting and does not establish optimal time complexity.
Extending the method and analysis to heterogeneous objectives, and understanding when overlap remains safe in that regime, is an important direction for future work. 
\begin{ack}
    The research reported in this publication was supported by funding from King Abdullah University of Science and Technology (KAUST): i) KAUST Baseline Research Scheme, ii) CRG Grant ORFS-CRG12-2024-6460, and iii) Center of Excellence for Generative AI, under award number 5940.
\end{ack}
{
\small
\bibliographystyle{plainnat}
\bibliography{bib}
}
\newpage
\appendix
\tableofcontents
\newpage
\section{Proofs}\label{sec:proofs}
This section collects the proofs of the main technical results.
\subsection{Within-round dynamics}
At the beginning of round \(r\), worker \(i\) holds a model \(x_i^r\in\R^d\).
Its local trajectory within the round is
\begin{equation}\label{eq:j099u0f9d8y98gfg}
    w_{i,0}^r \eqdef x_i^r,
    \qquad
    w_{i,t+1}^r
    =
    w_{i,t}^r-\eta_r\, g_i(w_{i,t}^r,\xi_{i,t}^r),
    \qquad
    t=0,1,\dots,H_i-1.
\end{equation}
We define
\begin{equation}\label{eq:ni90yfd87y98fdig}
    y_i^r \eqdef w_{i,N_i}^r,
    \qquad
    z_i^r \eqdef w_{i,H_i}^r ~.
\end{equation}
We also define the round averages
\begin{equation}
\label{eq:averages}
    \bar x^r \eqdef \frac{1}{n}\sum_{i=1}^n x_i^r,
    \qquad
    \bar y^r \eqdef \frac{1}{n}\sum_{i=1}^n y_i^r,
    \qquad
    \bar z^r \eqdef \frac{1}{n}\sum_{i=1}^n z_i^r.
\end{equation}
The delay-corrected merge is
\begin{equation}\label{eq:corrected_merge}
    x_i^{r+1}
    =
    \Proj_{S_r}\bigl(\bar y^r+z_i^r-y_i^r\bigr)+(I-\Proj_{S_r})z_i^r
    =
    z_i^r + \Proj_{S_r}(\bar y^r-y_i^r),
\end{equation}
where \(S_r\subseteq[d]\) is a uniformly random subset of cardinality \(K\), shared by all workers.

\subsection{Exact evolution of the average}

We now state and prove our first lemma, characterizing the exact evolution of the average $\bar x^r$.

\begin{lemma}[Exact evolution of the average]
\label{lem:avg_evol_09vdf}For every round \(r\),
\begin{equation}\label{eq:no8yfd9y89fd08ygf}
\bar x^{r+1}=\bar z^r = \bar x^r-\eta_r G_r,
\end{equation}
where
\begin{equation}\label{eq:oi-f9ud0fdg}
G_r \eqdef \sum_{t=0}^{H_{\max}-1} \bar g_t^r,
\qquad
\bar g_t^r \eqdef \frac{1}{n}\sum_{i\in A_t} g_i(w_{i,t}^r,\xi_{i,t}^r).
\end{equation}
\end{lemma}

\begin{proof}
Averaging the corrected merge formula gives
\begin{eqnarray*}
\bar x^{r+1}
\overset{\eqref{eq:averages}}{=}   \frac{1}{n}\sum_{i=1}^n x_i^{r+1}
&\overset{\eqref{eq:corrected_merge}}{=}& \frac{1}{n}\sum_{i=1}^n \Bigl(z_i^r+\Proj_{S_r}(\bar y^r-y_i^r)\Bigr) \\
&=& \frac{1}{n}\sum_{i=1}^n z_i^r + \frac{1}{n}\sum_{i=1}^n  \Proj_{S_r}(\bar y^r-y_i^r)\\
&=&
\bar z^r + \Proj_{S_r}\left(\frac{1}{n}\sum_{i=1}^n (\bar y^r- y_i^r)\right)\\
&=&
\bar z^r + \Proj_{S_r}(0) =\bar z^r.
\end{eqnarray*}
Next,
\begin{eqnarray}
z_i^r
 \overset{\eqref{eq:ni90yfd87y98fdig}}{=} w_{i,H_i}^r  \overset{\eqref{eq:j099u0f9d8y98gfg}}{=} 
x_i^r-\eta_r\sum_{t=0}^{H_i-1} g_i(w_{i,t}^r,\xi_{i,t}^r). \label{eq:k-0-f9ud09yu0fd}
\end{eqnarray}
Averaging over \(i\), we get
\begin{eqnarray*}
\bar z^r &\overset{\eqref{eq:averages}}{=} & \frac{1}{n}\sum_{i=1}^n z_i^r  \\
&\overset{\eqref{eq:k-0-f9ud09yu0fd}}{=}&
 \frac{1}{n}\sum_{i=1}^n x_i^r - \eta_r\frac{1}{n}\sum_{i=1}^n\sum_{t=0}^{H_i-1} g_i(w_{i,t}^r,\xi_{i,t}^r) \\
&\overset{\eqref{eq:averages}}{=} &
\bar x^r - \eta_r\frac{1}{n}\sum_{i=1}^n\sum_{t=0}^{H_i-1} g_i(w_{i,t}^r,\xi_{i,t}^r) \\
&\overset{\eqref{eq:A_t-89yf908df}}{=}&
\bar x^r - \eta_r\sum_{t=0}^{H_{\max}-1}\frac{1}{n}\sum_{i\in A_t} g_i(w_{i,t}^r,\xi_{i,t}^r) \\
&\overset{\eqref{eq:oi-f9ud0fdg}}{=}&
\bar x^r - \eta_r G_r.
\end{eqnarray*}
Combining the two identities proves the claim.
\end{proof}

\subsection{Variance reduction by worker averaging}

For \(i\in A_t\), define \begin{equation} \label{eq:-0-89u09980gf_0fgg}
\varepsilon_{i,t}^r
\eqdef
g_i(w_{i,t}^r,\xi_{i,t}^r)-\nabla f(w_{i,t}^r),
\end{equation}
and further let
\begin{eqnarray}
\bar\varepsilon_t^r
& \eqdef &  \frac{1}{n}\sum_{i\in A_t} \varepsilon_{i,t}^r. 
\label{eq:-8-f9d--08g} 
\end{eqnarray}

Plugging  \eqref{eq:-0-89u09980gf_0fgg} into \eqref{eq:-8-f9d--08g}  gives
\begin{eqnarray}
\bar\varepsilon_t^r
&\overset{\eqref{eq:-0-89u09980gf_0fgg}+\eqref{eq:-8-f9d--08g}}{=}&
\frac{1}{n}\sum_{i\in A_t}\Bigl(g_i(w_{i,t}^r,\xi_{i,t}^r)-\nabla f(w_{i,t}^r)\Bigr) \notag \\
& \overset{\eqref{eq:oi-f9ud0fdg}}{=} &
\bar g_t^r-\frac{1}{n}\sum_{i\in A_t}\nabla f(w_{i,t}^r). \notag
\end{eqnarray}

\begin{lemma}[Variance reduction at local time \(t\)]
\label{lem:variance_reduction}Under \Cref{ass:independence} (independent samples) and \Cref{ass:bounded_variance} (bounded variance), for every round \(r\) and local time \(t\), we have
\[
\ExpCond{\|\bar\varepsilon_t^r\|^2}{\cF_{r,t}}
\le
\frac{m_t}{n^2}\sigma^2
\]
and
\[
\sum_{t=0}^{H_{\max}-1}
\ExpCond{\|\bar\varepsilon_t^r\|^2}{\{x_i^r\}_{i=1}^n}
\le
\frac{\bar H}{n}\sigma^2.
\]
\end{lemma}

\begin{proof}
Conditioned on \(\cF_{r,t}\), the vectors
\(
\{\varepsilon_{i,t}^r, i\in A_t\},
\)
are independent and conditionally mean zero.
Therefore,
\begin{equation}\label{eq:i-09yp9dh8fgfg}
\ExpCond{\|\bar\varepsilon_t^r\|^2 }{\cF_{r,t}}
\overset{\eqref{eq:-8-f9d--08g} }{=}
\frac1{n^2}
\ExpCond{\left\|\sum_{i\in A_t}\varepsilon_{i,t}^r\right\|^2 }{ \cF_{r,t} } =
\frac1{n^2}
\sum_{i\in A_t}
\ExpCond{\|\varepsilon_{i,t}^r\|^2 }{ \cF_{r,t}},
\end{equation}
because the cross terms vanish.
By \Cref{ass:bounded_variance}, we have
\[
\ExpCond{\|\varepsilon_{i,t}^r\|^2}{ \cF_{r,t} }
\le
\sigma^2.
\]
Plugging this into \eqref{eq:i-09yp9dh8fgfg}, and recalling that  $m_t = |A_t|$, we get
\begin{equation}\label{eq:8g87fytd98yf8dg}
\ExpCond{\|\bar\varepsilon_t^r\|^2 }{ \cF_{r,t} }
\le
\frac{m_t}{n^2}\sigma^2.
\end{equation}
Summing the above inequality over \(t \in \{0,\dots, H_{\max}-1\}\) and using
\begin{eqnarray}
\sum_{t=0}^{H_{\max}-1} m_t & \overset{\eqref{eq:A_t-89yf908df}}{=} & \sum_{t=0}^{H_{\max}-1} |A_t|  \notag \\
& \overset{\eqref{eq:A_t-89yf908df}}{=} &\sum_{t=0}^{H_{\max}-1} \left| \{i\in\{1,\dots,n\}: H_i>t\} \right|  \notag \\
&=& \sum_{t=0}^{H_{\max}-1} \sum_{i=1}^{n} {\bf 1}(H_i > t)  \notag \\ 
&=& \sum_{i=1}^{n}  \sum_{t=0}^{H_{\max}-1} {\bf 1}(H_i > t) \notag \\ 
&=& \sum_{i=1}^n H_i \overset{\eqref{eq:-0-0-098uy8ygugf}}{=} n\bar H,\label{eq:bo87v8yfd7y87tf8d}
\end{eqnarray}
we obtain
\[
\sum_{t=0}^{H_{\max}-1}
\ExpCond{\|\bar\varepsilon_t^r\|^2}{ \cF_{r,t}  }
\overset{\eqref{eq:8g87fytd98yf8dg}}{\leq}
\frac{1}{n^2}\left(\sum_{t=0}^{H_{\max}-1} m_t\right)\sigma^2
\overset{\eqref{eq:bo87v8yfd7y87tf8d}}{=}
\frac{\bar H}{n}\sigma^2.
\]
Taking conditional expectation with respect to \(\{x_i^r\}_{i=1}^n\) and using the tower property, we obtain
\[
\sum_{t=0}^{H_{\max}-1}
\ExpCond{\|\bar\varepsilon_t^r\|^2}{\{x_i^r\}_{i=1}^n}
=
\sum_{t=0}^{H_{\max}-1}
\ExpCond{
\ExpCond{\|\bar\varepsilon_t^r\|^2}{\cF_{r,t}}
}{\{x_i^r\}_{i=1}^n}
\le
\frac{\bar H}{n}\sigma^2.
\]
\end{proof}

\subsection{Within-round drift bounds}

Define
\begin{equation}\label{eq:4_defs}
u_i^r \eqdef y_i^r-x_i^r,
\qquad
v_i^r \eqdef z_i^r-y_i^r,
\qquad
\bar u^r \eqdef \frac{1}{n}\sum_{i=1}^n u_i^r,
\qquad
\bar v^r \eqdef \frac{1}{n}\sum_{i=1}^n v_i^r.
\end{equation}

\begin{lemma}[Within-round local drift]\label{lem:within-round-drift}Let \Cref{ass:independence} (independent samples) and \Cref{ass:bounded_second_moment} (bounded second moment) hold. Then, for every round \(r\), worker \(i\), and local time \(t\in\{0,\dots,H_i\}\), we have
\begin{equation} \label{eq:W_bound_delay}
\ExpCond{ \|w_{i,t}^r-x_i^r\|^2 }{ \{x_j^r\}_{j=1}^n }
\leq
\eta_r^2 t^2 G^2,
\end{equation}
\begin{align}
\ExpCond{\sum_{i=1}^n \|u_i^r-\bar u^r\|^2 }{ \{x_j^r\}_{j=1}^n }
&\leq
\eta_r^2 G^2 S_N,
\label{eq:U_bound_delay}
\\
\ExpCond{\sum_{i=1}^n \|v_i^r-\bar v^r\|^2 }{ \{x_j^r\}_{j=1}^n }
&\leq
\eta_r^2 G^2 S_Q.
\label{eq:V_bound_delay}
\end{align}
Moreover,
\begin{equation}
\sum_{i=1}^n\sum_{t=0}^{H_i-1}
\ExpCond{\|w_{i,t}^r-x_i^r\|^2  }{ \{x_j^r\}_{j=1}^n }
\le
\eta_r^2 G^2 \Psi_H. \label{eq:double_W_bound_delay}
\end{equation}
\end{lemma}

\begin{proof}
We begin with the identity
\begin{equation}
w_{i,t}^r - x_i^r
\overset{\eqref{eq:j099u0f9d8y98gfg}}{=}
-\eta_r \sum_{s=0}^{t-1} g_i(w_{i,s}^r,\xi_{i,s}^r).
\label{eq:w-drift-sum}
\end{equation}
Using Cauchy-Schwarz, we obtain
\begin{equation}
\|w_{i,t}^r - x_i^r\|^2
\overset{\eqref{eq:w-drift-sum}}{=}
\eta_r^2 \left\|\sum_{s=0}^{t-1} g_i(w_{i,s}^r,\xi_{i,s}^r)\right\|^2
\leq 
\eta_r^2\, t \sum_{s=0}^{t-1} \|g_i(w_{i,s}^r,\xi_{i,s}^r)\|^2.
\label{eq:w-drift-preexp}
\end{equation}
Taking conditional expectation and using the tower property together with \Cref{ass:bounded_second_moment}, we get
\begin{equation}
\ExpCond{\|g_i(w_{i,s}^r,\xi_{i,s}^r)\|^2}{\{x_j^r\}_{j=1}^n}
=
\ExpCond{\ExpCond{\|g_i(w_{i,s}^r,\xi_{i,s}^r)\|^2}{\cF_{r,s}}
}{\{x_j^r\}_{j=1}^n}
\leq
G^2.
\label{eq:bounded-second-moment-used}
\end{equation}
Substituting into \eqref{eq:w-drift-preexp}, we obtain
\begin{align}
\ExpCond{\|w_{i,t}^r - x_i^r\|^2}{\{x_j^r\}_{j=1}^n}
&\overset{\eqref{eq:w-drift-preexp} + \eqref{eq:bounded-second-moment-used}}{\le}
\eta_r^2 t^2 G^2.
\label{eq:w-drift-final}
\end{align}

Summing \eqref{eq:w-drift-final} over \(i\) and \(t\), we get
\begin{equation*}
\sum_{i=1}^n \sum_{t=0}^{H_i-1}
\ExpCond{\|w_{i,t}^r - x_i^r\|^2}{\{x_j^r\}_{j=1}^n}
\overset{\eqref{eq:w-drift-final}}{\le}
\eta_r^2 G^2 \sum_{i=1}^n \sum_{t=0}^{H_i-1} t^2
\overset{\eqref{eq:Psi_H-defn}}{=}
\eta_r^2 G^2 \Psi_H.
\end{equation*}

Next, by definition,
\begin{equation}
u_i^r
\overset{\eqref{eq:4_defs}}{=} 
y_i^r-x_i^r 
\overset{\eqref{eq:ni90yfd87y98fdig}}{=}
w_{i,N_i}^r - x_i^r
\overset{\eqref{eq:j099u0f9d8y98gfg}}{=}
-\eta_r \sum_{t=0}^{N_i-1} g_i(w_{i,t}^r,\xi_{i,t}^r).
\label{eq:u-def-sum}
\end{equation}
Using Cauchy-Schwarz again, we obtain
\begin{eqnarray}
\|u_i^r\|^2
\overset{\eqref{eq:u-def-sum}}{=}
\eta_r^2 \left\|\sum_{t=0}^{N_i-1} g_i(w_{i,t}^r,\xi_{i,t}^r)\right\|^2
\leq
\eta_r^2 N_i \sum_{t=0}^{N_i-1} \|g_i(w_{i,t}^r,\xi_{i,t}^r)\|^2.
\label{eq:u-preexp}
\end{eqnarray}
Taking conditional expectation and using \eqref{eq:bounded-second-moment-used}, we get
\begin{align}
\ExpCond{\|u_i^r\|^2}{\{x_j^r\}_{j=1}^n}
&\overset{\eqref{eq:u-preexp}+\eqref{eq:bounded-second-moment-used}}{\le}
\eta_r^2 N_i^2 G^2.
\label{eq:u-bound}
\end{align}
Recall that \(\bar u^r \eqdef \frac{1}{n}\sum_{i=1}^n u_i^r\). Since the variance of the vectors \(\{u_i^r\}_{i=1}^n\) can be bounded by their second moment, i.e., 
\begin{equation}
\frac{1}{n}\sum_{i=1}^n \|u_i^r-\bar u^r\|^2
\leq
\frac{1}{n}\sum_{i=1}^n \|u_i^r\|^2,
\label{eq:-9b0fdyfdg}
\end{equation}
we obtain
\begin{eqnarray*}
\ExpCond{\sum_{i=1}^n \|u_i^r-\bar u^r\|^2}{\{x_j^r\}_{j=1}^n}
&\overset{\eqref{eq:-9b0fdyfdg}+\eqref{eq:u-bound}}{\le}&
\eta_r^2 G^2 \sum_{i=1}^n N_i^2
\notag\\
&\overset{\eqref{eq:98yf9d8yf0d_90uf09d}}{=}&
\eta_r^2 G^2 S_N.
\end{eqnarray*}

Finally, by definition,
\begin{equation}
v_i^r
\overset{\eqref{eq:4_defs}}{=} z_i^r-y_i^r \overset{\eqref{eq:ni90yfd87y98fdig}}{=}  w_{i,H_i}^r - w_{i,N_i}^r \overset{\eqref{eq:j099u0f9d8y98gfg}}{=}
-\eta_r \sum_{t=N_i}^{H_i-1} g_i(w_{i,t}^r,\xi_{i,t}^r),
\label{eq:v-def-sum}
\end{equation}
which contains exactly \(Q_i = H_i-N_i\) terms. Using \eqref{eq:v-def-sum} and Cauchy-Schwarz, we obtain
\begin{equation}
\|v_i^r\|^2
\overset{\eqref{eq:v-def-sum}}{=} 
\eta_r^2 \left\|\sum_{t=N_i}^{H_i-1} g_i(w_{i,t}^r,\xi_{i,t}^r)\right\|^2
\leq 
\eta_r^2 Q_i \sum_{t=N_i}^{H_i-1} \|g_i(w_{i,t}^r,\xi_{i,t}^r)\|^2.
\label{eq:=-=-0-yhfdig}
\end{equation}
Taking conditional expectation and using \eqref{eq:bounded-second-moment-used} and \eqref{eq:Q_i},
we get
\begin{align}
\ExpCond{\|v_i^r\|^2}{\{x_j^r\}_{j=1}^n}
&\overset{\eqref{eq:=-=-0-yhfdig}+\eqref{eq:bounded-second-moment-used}+\eqref{eq:Q_i}}{\leq}
\eta_r^2 Q_i^2 G^2.
\label{eq:9=-fdy9gfdf=-0fdg}
\end{align}
Using 
\begin{equation}
\sum_{i=1}^n \|v_i^r-\bar v^r\|^2
\le
\sum_{i=1}^n \|v_i^r\|^2,
\label{eq:mnoi-8ytf87gdi8y98fdgfgf}
\end{equation}
which is an analogue of the bound \eqref{eq:-9b0fdyfdg} applied to the vectors \(\{v_i^r\}_{i=1}^n\),
we obtain
\begin{eqnarray*}
\ExpCond{\sum_{i=1}^n \|v_i^r-\bar v^r\|^2}{\{x_j^r\}_{j=1}^n}
&\overset{\eqref{eq:mnoi-8ytf87gdi8y98fdgfgf}+\eqref{eq:9=-fdy9gfdf=-0fdg}}{\le}&
\eta_r^2 G^2 \sum_{i=1}^n Q_i^2
\notag\\
&\overset{\eqref{eq:98yf9d8yf0d_90uf09d}}{=}&
\eta_r^2 G^2 S_Q.
\end{eqnarray*}
This proves \eqref{eq:W_bound_delay}--\eqref{eq:double_W_bound_delay}.
\end{proof}

\subsection{Disagreement recursion}

Define the {\em disagreement} related to the vectors $\{x_i^r\}_{i=1}^n$ and $\{y_i^r\}_{i=1}^n$ as
\begin{equation}\label{eq:-09y9yf8dy8fy098f+-0=-9g}
X^r \eqdef \sum_{i=1}^n \|x_i^r-\bar x^r\|^2, \qquad Y^r \eqdef \sum_{i=1}^n \|y_i^r-\bar y^r\|^2.
\end{equation}
Analogously,  define the disagreement quantities
\begin{equation}\label{eq:three_defs}
U^r \eqdef \sum_{i=1}^n \|u_i^r-\bar u^r\|^2,
\qquad
V^r \eqdef \sum_{i=1}^n \|v_i^r-\bar v^r\|^2,
\qquad
Z^r \eqdef \sum_{i=1}^n \|z_i^r-\bar z^r\|^2.
\end{equation}

\begin{lemma}[Disagreement recursion]\label{lem:edr}Fix any parameters \(\alpha>0\) and \(\beta>0\), and define
\[
c \eqdef q(1+\alpha)(1+\beta),
\qquad
B \eqdef q(1+\beta)\left(1+\frac1\alpha\right),
\qquad
D \eqdef 1+\frac{q}{\beta}.
\]
Then, under \Cref{ass:independence} (mutual independence) and \Cref{ass:bounded_second_moment} (bounded second moment), for every round \(r\), we have
\begin{align}
\ExpCond{ X^{r+1} }{ \{x_i^r\}_{i=1}^n }
&\leq
c X^{r}
+
\eta_r^2 G^2\bigl(BS_N + DS_Q\bigr).
\label{eq:delay-corrected-disagreement-recursion}
\end{align}
\end{lemma}

\begin{proof}
By \Cref{lem:avg_evol_09vdf}, we have
\(
\bar x^{r+1}=\bar z^r.
\)
Therefore, for each worker \(i\),
\begin{eqnarray}
x_i^{r+1}-\bar x^{r+1}
&\overset{\eqref{eq:corrected_merge}}{=}&
z_i^r+\Proj_{S_r}(\bar y^r-y_i^r)-\bar z^r
\notag\\
&=&
z_i^r-\bar z^r-\Proj_{S_r}(y_i^r-\bar y^r)
\notag\\
&\overset{\eqref{eq:4_defs}}{=}&
(y_i^r-\bar y^r)+(v_i^r-\bar v^r)-\Proj_{S_r}(y_i^r-\bar y^r)
\notag\\
&=&
(I-\Proj_{S_r})(y_i^r-\bar y^r)+(v_i^r-\bar v^r).
\label{eq:corrected-disagreement-coordinate}
\end{eqnarray}
Because $y_i^r$ and $z_i^r$ are deterministic functions of the initial points $x_i^0$, historical masks, and data samples up to round $r$, \Cref{ass:independence} ensures that the mask $S_r$ is independent of $\{y_i^r, z_i^r\}_{i=1}^n$.
Its conditional distribution remains uniform and each coordinate is not selected with
probability \(q=1-p\). 
For any fixed vectors \(a,b\), we have
\[
\Exp{\|(I-\Proj_{S_r})a+b\|^2}
=
q\|a\|^2+2q\inner{a}{b}+\|b\|^2.
\]
Applying this with
\[
a=y_i^r-\bar y^r,
\qquad
b=v_i^r-\bar v^r,
\]
and using \eqref{eq:corrected-disagreement-coordinate}, we get
\begin{eqnarray}
\ExpCond{\|x_i^{r+1}-\bar x^{r+1}\|^2}{\{y_i^r,z_i^r\}_{i=1}^n}
&=&
q\|y_i^r-\bar y^r\|^2
+
2q\inner{y_i^r-\bar y^r}{v_i^r-\bar v^r}
+
\|v_i^r-\bar v^r\|^2
\notag\\
&=&
q\|(y_i^r-\bar y^r)+(v_i^r-\bar v^r)\|^2
+
p\|v_i^r-\bar v^r\|^2
\notag\\
&=&
q\|z_i^r-\bar z^r\|^2
+
p\|v_i^r-\bar v^r\|^2.
\label{eq:one-worker-mask-expectation}
\end{eqnarray}
Summing \eqref{eq:one-worker-mask-expectation} over \(i\), we obtain
\begin{eqnarray}
\ExpCond{X^{r+1}}{\{y_i^r,z_i^r\}_{i=1}^n}
&=&
q\sum_{i=1}^n\|z_i^r-\bar z^r\|^2
+
p\sum_{i=1}^n\|v_i^r-\bar v^r\|^2
\notag\\
&\overset{\eqref{eq:three_defs}}{=}&
qZ^r+pV^r.
\label{eq:mask-disagreement-exact}
\end{eqnarray}

Next, using the identity
\[
z_i^r-\bar z^r=(y_i^r-\bar y^r)+(v_i^r-\bar v^r)
\]
which follows from \eqref{eq:4_defs}, 
we estimate \(qZ^r+pV^r\). For any vectors \(a,b\), we can estimate \begin{eqnarray}
q\|a+b\|^2+p\|b\|^2
&=&
q\|a\|^2+2q\inner{a}{b}+(q+p)\|b\|^2
\notag\\
&=&
q\|a\|^2+2q\inner{a}{b}+\|b\|^2
\notag\\
&\leq&
q(1+\beta)\|a\|^2+\left(1+\frac{q}{\beta}\right)\|b\|^2,
\label{eq:ab-beta-bound}
\end{eqnarray}
where we used the inequality
\[
2q\inner{a}{b}
\le
q\beta\|a\|^2+\frac{q}{\beta}\|b\|^2.
\]
Applying \eqref{eq:ab-beta-bound} with
\[
a=y_i^r-\bar y^r,
\qquad
b=v_i^r-\bar v^r,
\]
and summing over \(i\), gives
\begin{eqnarray}
qZ^r+pV^r
&\leq&
q(1+\beta)Y^r+\left(1+\frac{q}{\beta}\right)V^r
\notag\\
&=&
q(1+\beta)Y^r+DV^r.
\label{eq:gamma-v-bound}
\end{eqnarray}

It remains to control \(Y^r\). Since
\[
y_i^r-\bar y^r
\overset{\eqref{eq:4_defs}}{=}
(x_i^r-\bar x^r)+(u_i^r-\bar u^r),
\]
using Young's inequality leads to \begin{eqnarray}
Y^r
&\overset{\eqref{eq:three_defs}}{=}&
\sum_{i=1}^n\|y_i^r-\bar y^r\|^2
\notag\\
&=&
\sum_{i=1}^n
\|(x_i^r-\bar x^r)+(u_i^r-\bar u^r)\|^2
\notag\\
&\leq &
\sum_{i=1}^n \left[
(1+\alpha)\|x_i^r-\bar x^r\|^2 +\left(1+\frac1\alpha\right) \|u_i^r-\bar u^r\|^2 \right]
\notag\\
&\overset{\eqref{eq:-09y9yf8dy8fy098f+-0=-9g}+\eqref{eq:three_defs}}{=}&
(1+\alpha)X^{r}
+
\left(1+\frac1\alpha\right)
U^r.
\label{eq:Y-bound-alpha}
\end{eqnarray}
Combining \eqref{eq:mask-disagreement-exact}, \eqref{eq:gamma-v-bound}, and
\eqref{eq:Y-bound-alpha}, we obtain
\begin{eqnarray}
\ExpCond{X^{r+1}}{\{y_i^r,z_i^r\}_{i=1}^n}
&\overset{\eqref{eq:mask-disagreement-exact}}{=}&
qZ^r+pV^r
\notag\\
&\overset{\eqref{eq:gamma-v-bound}}{\leq}&
q(1+\beta)Y^r+DV^r
\notag\\
&\overset{\eqref{eq:Y-bound-alpha}}{\leq}&
q(1+\beta)(1+\alpha)X^{r}
+
q(1+\beta)\left(1+\frac1\alpha\right)
U^r
+
DV^r
\notag\\
&=&
cX^{r}
+
B U^r
+
DV^r.
\label{eq:disagreement-before-expectation}
\end{eqnarray}
Finally, taking conditional expectation with respect to \(\{x_i^r\}_{i=1}^n\) and using
\eqref{eq:U_bound_delay}--\eqref{eq:V_bound_delay}, we get
\begin{eqnarray}
\ExpCond{X^{r+1}}{\{x_i^r\}_{i=1}^n}
&\overset{\eqref{eq:disagreement-before-expectation}}{\leq}&
cX^{r}
+
B\ExpCond{U^r}{\{x_i^r\}_{i=1}^n}
+
D\ExpCond{V^r}{\{x_i^r\}_{i=1}^n}
\notag\\
&\overset{\eqref{eq:U_bound_delay}+\eqref{eq:V_bound_delay}}{\leq}&
cX^{r}+\eta_r^2G^2(BS_N+DS_Q).
\end{eqnarray}
This proves \eqref{eq:delay-corrected-disagreement-recursion}.
\end{proof}

\begin{corollary}[Accumulated disagreement bound]
\label{cor:sum_disagreement}
Assume \(x_i^0=x^0\) for all \(i\), so that \(X^0=0\), and assume \(\eta_r\equiv \eta\) for all $r$.
If
\[
c=q(1+\alpha)(1+\beta)<1,
\]
then
\begin{align}
\sum_{r=0}^{R-1}\Exp{X^{r}}
&\leq
R\,\frac{\eta^2G^2\bigl(BS_N + DS_Q\bigr)}{1-c}
\label{eq:sum_disagreement_corrected_delay}
\end{align}
for every \(R\ge 1\).
\end{corollary}

\begin{proof}
Let
\(
\delta^r \eqdef \Exp{X^r}\).
Taking full expectation in \eqref{eq:delay-corrected-disagreement-recursion}, and using
\(\eta_r\equiv \eta\), gives
\begin{eqnarray}
\delta^{r+1}
\overset{\eqref{eq:delay-corrected-disagreement-recursion}}{\leq}
c\delta^r + C,
\label{eq:delta-recursion}
\end{eqnarray}
where \(C \eqdef \eta^2G^2(BS_N+DS_Q)\).
Since \(x_i^0=x^0\) for all \(i\), we have
\[
X^0=0,
\qquad
\delta^0=\Exp{X^0}=0.
\]
We now prove by induction that, for every \(r\geq 0\),
\begin{equation}
\delta^r
\leq
C\sum_{j=0}^{r-1}c^j,
\label{eq:delta-geometric-bound}
\end{equation}
where the sum is interpreted as \(0\) when \(r=0\).
For \(r=0\), \eqref{eq:delta-geometric-bound} gives
\(
\delta^0\leq 0,
\)
which holds since \(\delta^0=0\). Suppose now that \eqref{eq:delta-geometric-bound} holds for some \(r\geq 0\). Then
\begin{eqnarray*}
\delta^{r+1}
\overset{\eqref{eq:delta-recursion}}{\leq}
c\delta^r+C
\overset{\eqref{eq:delta-geometric-bound}}{\leq}
cC\sum_{j=0}^{r-1}c^j+C
=
C\sum_{j=1}^{r}c^j+C
=
C\sum_{j=0}^{r}c^j.
\end{eqnarray*}
Thus \eqref{eq:delta-geometric-bound} holds for all \(r\geq 0\). Since \(c<1\), we have
\[
\sum_{j=0}^{r-1}c^j
\leq
\frac{1}{1-c},
\]
and therefore
\begin{equation}
\delta^r
\overset{\eqref{eq:delta-geometric-bound}}{\leq}
\frac{C}{1-c}
=
\frac{\eta^2G^2(BS_N+DS_Q)}{1-c}.
\label{eq:delta-uniform-bound}
\end{equation}
Summing \eqref{eq:delta-uniform-bound} over \(r=0,\dots,R-1\), we get
\begin{eqnarray*}
\sum_{r=0}^{R-1}\Exp{X^r}
&=&
\sum_{r=0}^{R-1}\delta^r
\\
&\overset{\eqref{eq:delta-uniform-bound}}{\leq}&
R\,\frac{\eta^2G^2(BS_N+DS_Q)}{1-c}.
\end{eqnarray*}
This proves \eqref{eq:sum_disagreement_corrected_delay}.
\end{proof}

\subsection{One-round descent inequality}

Define
\begin{equation}\label{eq:three_defs_final_lemma}
\bar\nabla_t^r \eqdef \frac{1}{n}\sum_{i\in A_t}\nabla f(w_{i,t}^r),
\qquad
\bar g_t^r \eqdef \frac{1}{n}\sum_{i\in A_t} g_i(w_{i,t}^r,\xi_{i,t}^r),
\qquad
G_r \eqdef \sum_{t=0}^{H_{\max}-1}\bar g_t^r.
\end{equation}

\begin{lemma}[One-round descent]\label{lem:explicit_one_round}Let \Cref{ass:smoothness} (smoothness), \Cref{ass:independence} (independent samples), \Cref{ass:unbiased} (unbiasedness), \Cref{ass:bounded_variance} (bounded variance), and \Cref{ass:bounded_second_moment} (bounded second moment) be satisfied. Then for every round \(r\), we have
\begin{align}
\ExpCond{ f(\bar x^{r+1}) }{ \{x_i^r\}_{i=1}^n }
\leq
f(\bar x^r)
&-
\eta_r\left(\frac{\bar H}{2}-2L\eta_r \bar H H_{\max}\right)\|\nabla f(\bar x^r)\|^2
\notag\\
&
+
\eta_r\frac{L^2H_{\max}}{n}\bigl(1+4L\eta_r H_{\max}\bigr)X^{r} \notag \\
&
+
L\eta_r^2\bar H\frac{\sigma^2}{n}
\notag\\
&
+
L^2\eta_r^3G^2\frac{\Psi_H}{n}\bigl(1+4L\eta_r H_{\max}\bigr).
\label{eq:98y80_098yu0fd_0yhgfdf}
\end{align}
\end{lemma}

\begin{proof}
By \Cref{lem:avg_evol_09vdf}, we have
\begin{equation*}
\bar x^{r+1}
\overset{\eqref{eq:no8yfd9y89fd08ygf}}{=}
\bar x^r-\eta_r G_r.
\end{equation*}
Hence, by \(L\)-smoothness of \(f\) (\Cref{ass:smoothness}),
\begin{eqnarray}
f(\bar x^{r+1})
&\le&
f(\bar x^r)
-
\eta_r \inner{\nabla f(\bar x^r)}{G_r}
+
\frac{L\eta_r^2}{2}\|G_r\|^2.
\label{eq:smoothness_step_98}
\end{eqnarray}
Taking conditional expectation given \(\{x_i^r\}_{i=1}^n\), we obtain
\begin{eqnarray}
\ExpCond{ f(\bar x^{r+1}) }{ \{x_i^r\}_{i=1}^n }
\overset{\eqref{eq:smoothness_step_98}}{\leq}
f(\bar x^r)
-
\eta_r \ExpCond{ \inner{\nabla f(\bar x^r)}{G_r} }{ \{x_i^r\}_{i=1}^n }
+
\frac{L\eta_r^2}{2}\ExpCond{ \|G_r\|^2 }{ \{x_i^r\}_{i=1}^n }.
\label{eq:start_one_round_descent_0980fd}
\end{eqnarray}

\paragraph{Step 1: lower bound the inner-product term.}
Using unbiasedness, we get
\begin{equation}\label{eq:h87t7fgdufiu_098uyff}
\ExpCond{ \bar g_t^r }{ \cF_{r,t} } \overset{\Cref{ass:unbiased}}{=} \bar\nabla_t^r,
\end{equation}
and hence
\begin{eqnarray}
\ExpCond{ \inner{\nabla f(\bar x^r)}{G_r} }{ \{x_i^r\}_{i=1}^n }
&\overset{\eqref{eq:h87t7fgdufiu_098uyff}+\eqref{eq:three_defs_final_lemma}}{=}&
\sum_{t=0}^{H_{\max}-1}
\ExpCond{ \inner{\nabla f(\bar x^r)}{\bar\nabla_t^r} }{ \{x_i^r\}_{i=1}^n }.
\label{eq:inner-expand}
\end{eqnarray}
Further, 
\begin{equation}\label{eq:-uyhfd89y8gfg}
\inner{\nabla f(\bar x^r)}{\bar\nabla_t^r}
\overset{\eqref{eq:three_defs_final_lemma}}{=}
\frac{1}{n}\sum_{i\in A_t}
\inner{\nabla f(\bar x^r)}{\nabla f(w_{i,t}^r)}.
\end{equation}
Using the inequality
\[
\inner{a}{b}
\ge
\frac{1}{2}\|a\|^2-\frac{1}{2}\|a-b\|^2,
\]
with
\(
a=\nabla f(\bar x^r),\;
b=\nabla f(w_{i,t}^r),
\)
and then using \(L\)-Lipschitz continuity of \(\nabla f\) (cf. \Cref{ass:smoothness}), we obtain
\begin{equation}\label{eq:u08yfod9y98gfgf}
\inner{\nabla f(\bar x^r)}{\nabla f(w_{i,t}^r)}
\ge
\frac{1}{2}\|\nabla f(\bar x^r)\|^2
-
\frac{L^2}{2}\|w_{i,t}^r-\bar x^r\|^2.
\end{equation}
Plugging \eqref{eq:u08yfod9y98gfgf} into \eqref{eq:-uyhfd89y8gfg}, and using $m_t = |A_t|$, we get
\begin{eqnarray}
\inner{\nabla f(\bar x^r)}{\bar\nabla_t^r}
&\ge&
\frac{m_t}{2n}\|\nabla f(\bar x^r)\|^2
-
\frac{L^2}{2n}\sum_{i\in A_t}\|w_{i,t}^r-\bar x^r\|^2.
\label{eq:inner-lower}
\end{eqnarray}
Taking conditional expectation and summing over \(t\),
\begin{eqnarray}
\ExpCond{ \inner{\nabla f(\bar x^r)}{G_r} }{ \{x_i^r\}_{i=1}^n }
&\overset{\eqref{eq:inner-expand}+\eqref{eq:inner-lower}}{\ge}&
\left(\frac1{2n}\sum_{t=0}^{H_{\max}-1}m_t\right)\|\nabla f(\bar x^r)\|^2
\notag\\
&&
-
\frac{L^2}{2n}
\sum_{i=1}^n\sum_{t=0}^{H_i-1}
\ExpCond{ \|w_{i,t}^r-\bar x^r\|^2 }{ \{x_i^r\}_{i=1}^n } \notag \\
&=&\frac{\bar H}{2}\|\nabla f(\bar x^r)\|^2
\notag\\
&&
-
\frac{L^2}{2n}
\sum_{i=1}^n\sum_{t=0}^{H_i-1}
\ExpCond{ \|w_{i,t}^r-\bar x^r\|^2 }{ \{x_i^r\}_{i=1}^n } ,
\label{eq:inner-before-distance}
\end{eqnarray}
where in the last step we have used the identity
\[
\sum_{t=0}^{H_{\max}-1}m_t
\overset{\eqref{eq:A_t-89yf908df}}{=}
\sum_{i=1}^n H_i
\overset{\eqref{eq:-0-0-098uy8ygugf}}{=}
n\bar H.
\]

Next, using the decomposition
\begin{equation*}
w_{i,t}^r-\bar x^r
=
(x_i^r-\bar x^r)+(w_{i,t}^r-x_i^r),
\end{equation*}
Young's inequality
\[
\|a+b\|^2 \le 2\|a\|^2+2\|b\|^2,
\]
and \Cref{lem:within-round-drift}, we get
\begin{equation}\label{eq:-0jhfd7y98y9fgf}
\ExpCond{ \|w_{i,t}^r-\bar x^r\|^2 }{ \{x_i^r\}_{i=1}^n }
\le
2\|x_i^r-\bar x^r\|^2 + 2\eta_r^2 t^2 G^2.
\end{equation}
Summing over \(i\) and \(t\), we obtain
\begin{eqnarray}
\sum_{i=1}^n\sum_{t=0}^{H_i-1}
\ExpCond{ \|w_{i,t}^r-\bar x^r\|^2 }{ \{x_i^r\}_{i=1}^n }
&\overset{\eqref{eq:-0jhfd7y98y9fgf}}{\leq}&
\sum_{i=1}^n\sum_{t=0}^{H_i-1}
\left(
2\|x_i^r-\bar x^r\|^2 + 2\eta_r^2 t^2 G^2
\right)
\notag\\
&=&
2\sum_{i=1}^n\sum_{t=0}^{H_i-1}\|x_i^r-\bar x^r\|^2
+
2\eta_r^2 G^2\sum_{i=1}^n\sum_{t=0}^{H_i-1} t^2
\notag\\
&\overset{\eqref{eq:Psi_H-defn}}{=}&
2\sum_{i=1}^n H_i \|x_i^r-\bar x^r\|^2
+
2\eta_r^2 G^2 \Psi_H
\notag\\
&\leq&
2H_{\max}\sum_{i=1}^n \|x_i^r-\bar x^r\|^2
+
2\eta_r^2 G^2 \Psi_H
\notag\\
&\overset{\eqref{eq:-09y9yf8dy8fy098f+-0=-9g}}{=}&
2H_{\max}X^r + 2\eta_r^2 G^2\Psi_H.
\label{eq:distance-bound}
\end{eqnarray}

Substituting \eqref{eq:distance-bound} into \eqref{eq:inner-before-distance}, we get
\begin{eqnarray}
\ExpCond{ \inner{\nabla f(\bar x^r)}{G_r} }{ \{x_i^r\}_{i=1}^n }
&\ge&
\frac{\bar H}{2}\|\nabla f(\bar x^r)\|^2
-
\frac{L^2H_{\max}}{n}X^r
-
\frac{L^2\eta_r^2G^2}{n}\Psi_H.
\label{eq:inner-term-final}
\end{eqnarray}

\paragraph{Step 2: upper bound the quadratic term.}
Let us write
\[
G_r
\overset{\eqref{eq:three_defs_final_lemma}+\eqref{eq:-8-f9d--08g} }{=}
\sum_{t=0}^{H_{\max}-1}\bar\nabla_t^r
+
\sum_{t=0}^{H_{\max}-1}\bar\varepsilon_t^r.
\]
Using Young's inequality \(\|a+b\|^2\le 2\|a\|^2+2\|b\|^2\), we get
\begin{eqnarray}
\ExpCond{ \|G_r\|^2 }{ \{x_i^r\}_{i=1}^n }
\le
2\ExpCond{ \left\|\sum_{t=0}^{H_{\max}-1}\bar\nabla_t^r\right\|^2 }{ \{x_i^r\}_{i=1}^n }
+
2\ExpCond{ \left\|\sum_{t=0}^{H_{\max}-1}\bar\varepsilon_t^r\right\|^2 }{ \{x_i^r\}_{i=1}^n }.
\label{eq:G-split}
\end{eqnarray}

First, let's bound the first term on the right-hand side of \eqref{eq:G-split}. Using the Jensen's inequality  bound
\[
\left\|\sum_{t=0}^{H_{\max}-1} a_t\right\|^2
\le
H_{\max}\sum_{t=0}^{H_{\max}-1}\|a_t\|^2,
\]
and arguments identical to those leading to \eqref{eq:distance-bound} (we repeat the same decomposition and apply \Cref{lem:within-round-drift}), we obtain the bound
\begin{eqnarray}
\ExpCond{ \left\|\sum_{t=0}^{H_{\max}-1}\bar\nabla_t^r\right\|^2 }{ \{x_i^r\}_{i=1}^n }
\leq
2\bar H H_{\max}\|\nabla f(\bar x^r)\|^2
+
\frac{4L^2H_{\max}^2}{n}X^r
+
\frac{4L^2\eta_r^2G^2H_{\max}}{n}\Psi_H.
\label{eq:mean-term}
\end{eqnarray}

Second, we bound the second term on the right-hand side of \eqref{eq:G-split}. We first explain why the cross terms vanish. Recall that
\[
\bar\varepsilon_t^r
=
\frac{1}{n}\sum_{i\in A_t}
\bigl(g_i(w_{i,t}^r,\xi_{i,t}^r)-\nabla f(w_{i,t}^r)\bigr).
\]
By \Cref{ass:unbiased}, and since \(w_{i,t}^r\) is \(\cF_{r,t}\)-measurable, we have
\begin{equation}
\ExpCond{\bar\varepsilon_t^r}{\cF_{r,t}} = 0.
\label{eq:noise-zero-mean}
\end{equation}
Moreover, for any \(s<t\), the vector \(\bar\varepsilon_s^r\) is \(\cF_{r,t}\)-measurable, since it depends only on randomness revealed before time \(t\). Therefore,
\begin{eqnarray}
\ExpCond{\inner{\bar\varepsilon_s^r}{\bar\varepsilon_t^r}}{\{x_i^r\}_{i=1}^n}
&=&
\ExpCond{
\ExpCond{\inner{\bar\varepsilon_s^r}{\bar\varepsilon_t^r}}{\cF_{r,t}}
}{\{x_i^r\}_{i=1}^n}
\notag\\
&=&
\ExpCond{
\inner{\bar\varepsilon_s^r}{\ExpCond{\bar\varepsilon_t^r}{\cF_{r,t}}}
}{\{x_i^r\}_{i=1}^n}
\notag\\
&\overset{\eqref{eq:noise-zero-mean}}{=}&
0.
\label{eq:cross-terms-zero}
\end{eqnarray}
This shows that the sequence \(\{\bar\varepsilon_t^r\}\) is a martingale difference sequence, and hence different-time noise terms are orthogonal in conditional expectation.

Now expand the squared norm:
\begin{eqnarray}
\ExpCond{ \left\|\sum_{t=0}^{H_{\max}-1}\bar\varepsilon_t^r\right\|^2 }{ \{x_i^r\}_{i=1}^n }
&=&
\ExpCond{
\sum_{t=0}^{H_{\max}-1}\|\bar\varepsilon_t^r\|^2
+
2\sum_{0\le s<t\le H_{\max}-1}\inner{\bar\varepsilon_s^r}{\bar\varepsilon_t^r}
}{ \{x_i^r\}_{i=1}^n }
\notag\\
&=&
\sum_{t=0}^{H_{\max}-1}
\ExpCond{\|\bar\varepsilon_t^r\|^2}{ \{x_i^r\}_{i=1}^n } \notag\\
&&\; + \;
2\sum_{0\le s<t\le H_{\max}-1}
\ExpCond{\inner{\bar\varepsilon_s^r}{\bar\varepsilon_t^r}}{ \{x_i^r\}_{i=1}^n }
\notag\\
&\overset{\eqref{eq:cross-terms-zero}}{=}&
\sum_{t=0}^{H_{\max}-1}
\ExpCond{\|\bar\varepsilon_t^r\|^2}{ \{x_i^r\}_{i=1}^n }.
\label{eq:noise-no-cross}
\end{eqnarray}
Finally, applying \Cref{lem:variance_reduction}, we obtain
\begin{eqnarray}
\ExpCond{ \left\|\sum_{t=0}^{H_{\max}-1}\bar\varepsilon_t^r\right\|^2 }{ \{x_i^r\}_{i=1}^n }
&\overset{\eqref{eq:noise-no-cross}}{=}&
\sum_{t=0}^{H_{\max}-1}
\ExpCond{\|\bar\varepsilon_t^r\|^2}{ \{x_i^r\}_{i=1}^n }
\notag\\
&\overset{\Cref{lem:variance_reduction}}{\leq}& \sum_{t=0}^{H_{\max}-1}
\frac{m_t}{n^2}\sigma^2
\notag\\
&=&
\frac{\sigma^2}{n^2}\sum_{t=0}^{H_{\max}-1} m_t
\notag\\
&\overset{\eqref{eq:bo87v8yfd7y87tf8d}}{=}&
\frac{\sigma^2}{n^2}\cdot n\bar H
=
\frac{\bar H}{n}\sigma^2.
\label{eq:noise-term}
\end{eqnarray}
Combining \eqref{eq:G-split}, \eqref{eq:mean-term}, and \eqref{eq:noise-term}, we obtain
\begin{eqnarray}
\ExpCond{ \|G_r\|^2 }{ \{x_i^r\}_{i=1}^n }
&\leq&
4\bar H H_{\max}\|\nabla f(\bar x^r)\|^2
+
\frac{8L^2H_{\max}^2}{n}X^r \\
&+& 
\frac{8L^2\eta_r^2G^2H_{\max}}{n}\Psi_H
+
2\frac{\bar H}{n}\sigma^2.
\label{eq:G-final}
\end{eqnarray}
\paragraph{Step 3: conclude.}
Substituting \eqref{eq:inner-term-final} and \eqref{eq:G-final} into \eqref{eq:start_one_round_descent_0980fd}, and collecting terms, yields
\begin{align*}
\ExpCond{ f(\bar x^{r+1}) }{ \{x_i^r\}_{i=1}^n }
&\leq
f(\bar x^r)
-
\eta_r\left(\frac{\bar H}{2}-2L\eta_r \bar H H_{\max}\right)\|\nabla f(\bar x^r)\|^2
\\
&\quad
+
\eta_r\frac{L^2H_{\max}}{n}\bigl(1+4L\eta_r H_{\max}\bigr)X^{r}
+
L\eta_r^2\bar H\frac{\sigma^2}{n}
\\
&\quad
+
L^2\eta_r^3G^2\frac{\Psi_H}{n}\bigl(1+4L\eta_r H_{\max}\bigr),
\end{align*}
which is \eqref{eq:98y80_098yu0fd_0yhgfdf}.
\end{proof}

\subsection{Proof of \Cref{thm:main}} 
\label{sub:proof_of_thm:main}

\maintheorem*

\begin{proof}
    Taking full expectation in \Cref{lem:explicit_one_round}, and using
    \(\eta_r\equiv\eta\), gives
    \begin{eqnarray}
    \Exp{f(\bar x^{r+1})}
    &\overset{\eqref{eq:98y80_098yu0fd_0yhgfdf}}{\leq}&
    \Exp{f(\bar x^r)}
    -
    \eta\left(\frac{\bar H}{2}-2L\eta\bar H H_{\max}\right)
    \Exp{\|\nabla f(\bar x^r)\|^2}
    \notag\\
    &&+
    \eta\frac{L^2H_{\max}}{n}
    \bigl(1+4L\eta H_{\max}\bigr)\Exp{X^r}
    +
    L\eta^2\bar H\frac{\sigma^2}{n}
    \notag\\
    &&+
    L^2\eta^3G^2\frac{\Psi_H}{n}
    \bigl(1+4L\eta H_{\max}\bigr).
    \label{eq:n98yf09_09u8f09dug}
    \end{eqnarray}
    Since
    \(
    \eta \leq \frac{1}{8LH_{\max}},
    \)
    we have
    \begin{eqnarray}
    4L\eta H_{\max} &\leq & \frac{1}{2},
    \label{eq:stepsize_907f090f}
    \end{eqnarray}
    and therefore,
    \begin{eqnarray}
    1+4L\eta H_{\max}
    &\overset{\eqref{eq:stepsize_907f090f}}{\leq}&
    \frac{3}{2}
    \label{eq:stepsize-factor-bound}
    \end{eqnarray}
    and
    \begin{eqnarray}
    \frac{\bar H}{2}-2L\eta\bar H H_{\max}
    =
    \bar H\left(\frac{1}{2}-2L\eta H_{\max}\right)
    \overset{\eqref{eq:stepsize_907f090f}}{\geq}
    \frac{\bar H}{4}.
    \label{eq:stepsize-gradient-coeff}
    \end{eqnarray}

    Using \eqref{eq:stepsize-gradient-coeff} and \eqref{eq:stepsize-factor-bound} in
    \eqref{eq:n98yf09_09u8f09dug}, we obtain
    \begin{eqnarray}
    \Exp{f(\bar x^{r+1})}
    &\leq&
    \Exp{f(\bar x^r)}
    -
    \frac{\eta\bar H}{4}
    \Exp{\|\nabla f(\bar x^r)\|^2}
    \notag\\
    &&+
    \frac{3}{2}\,\eta\frac{L^2H_{\max}}{n}\Exp{X^r}
    +
    L\eta^2\bar H\frac{\sigma^2}{n}
    +
    \frac{3}{2}\,L^2\eta^3G^2\frac{\Psi_H}{n}.
    \label{eq:expected-one-round-descent-simplified-thm}
    \end{eqnarray}
    Rearranging \eqref{eq:expected-one-round-descent-simplified-thm}, we get
    \begin{eqnarray}
    \frac{\eta\bar H}{4}
    \Exp{\|\nabla f(\bar x^r)\|^2}
    &\leq&
    \Exp{f(\bar x^r)}-\Exp{f(\bar x^{r+1})}
    +
    \frac{3}{2}\,\eta\frac{L^2H_{\max}}{n}\Exp{X^r}
    \notag\\
    &&+
    L\eta^2\bar H\frac{\sigma^2}{n}
    +
    \frac{3}{2}\,L^2\eta^3G^2\frac{\Psi_H}{n}.
    \label{eq:rearranged-simplified-thm}
    \end{eqnarray}
    Summing \eqref{eq:rearranged-simplified-thm} over \(r=0,\dots,R-1\), we obtain
    \begin{eqnarray}
    \frac{\eta\bar H}{4}
    \sum_{r=0}^{R-1}\Exp{\|\nabla f(\bar x^r)\|^2}
    &\overset{\eqref{eq:rearranged-simplified-thm}}{\leq}&
    \sum_{r=0}^{R-1}
    \left(
    \Exp{f(\bar x^r)}-\Exp{f(\bar x^{r+1})}
    \right)
    \notag\\
    &&+
    \frac{3}{2}\,\eta\frac{L^2H_{\max}}{n}
    \sum_{r=0}^{R-1}\Exp{X^r}
    \notag\\
    &&+
    R L\eta^2\bar H\frac{\sigma^2}{n}
    +
    \frac{3}{2}\,R L^2\eta^3G^2\frac{\Psi_H}{n}
    \notag\\
    &=&
    f(x^0)-\Exp{f(\bar x^R)}
    \notag\\
    &&+
    \frac{3}{2}\,\eta\frac{L^2H_{\max}}{n}
    \sum_{r=0}^{R-1}\Exp{X^r}
    \notag\\
    &&+
    R L\eta^2\bar H\frac{\sigma^2}{n}
    +
    \frac{3}{2}\,R L^2\eta^3G^2\frac{\Psi_H}{n}.
    \label{eq:summed-simplified-thm}
    \end{eqnarray}
    By \Cref{ass:lower_bounded} (lower-boundedness), we get
    \begin{eqnarray}
    f(x^0)-\Exp{f(\bar x^R)}
    &\leq&
    f(x^0)-f^\star.
    \label{eq:lower-bound-final-iterate}
    \end{eqnarray}
    Moreover, since \(X^0=0\), \(\eta_r\equiv\eta\), and \(c<1\), \Cref{cor:sum_disagreement} gives
    \begin{eqnarray}
    \sum_{r=0}^{R-1}\Exp{X^r}
    &\overset{\Cref{cor:sum_disagreement}}{\leq}&
    R\,\frac{\eta^2G^2(BS_N+DS_Q)}{1-c}.
    \label{eq:use-sum-disagreement}
    \end{eqnarray}
    Substituting \eqref{eq:lower-bound-final-iterate} and
    \eqref{eq:use-sum-disagreement} into \eqref{eq:summed-simplified-thm}, we get
    \begin{eqnarray}
    \frac{\eta\bar H}{4}
    \sum_{r=0}^{R-1}\Exp{\|\nabla f(\bar x^r)\|^2}
    &\leq&
    f(x^0)-f^\star
    +
    \frac{3}{2}\,R\eta^3\frac{L^2G^2H_{\max}}{n}
    \frac{BS_N+DS_Q}{1-c}
    \notag\\
    &&+
    R L\eta^2\bar H\frac{\sigma^2}{n}
    +
    \frac{3}{2}\,R L^2\eta^3G^2\frac{\Psi_H}{n}.
    \label{eq:before-final-division-thm}
    \end{eqnarray}
    Dividing \eqref{eq:before-final-division-thm} by \(R\eta\bar H/4\), we obtain
    \begin{eqnarray*}
    \frac1R\sum_{r=0}^{R-1}\Exp{\|\nabla f(\bar x^r)\|^2}
    &\leq&
    \frac{4(f(x^0)-f^\star)}{\eta\bar H R}
    +
    \frac{6L^2\eta^2G^2H_{\max}}{n\bar H}
    \frac{BS_N+DS_Q}{1-c}
    \\
    &&+
    \frac{4L\eta\sigma^2}{n}
    +
    \frac{6L^2\eta^2G^2\Psi_H}{n\bar H}.
    \end{eqnarray*}
    Finally, reordering the terms gives \eqref{eq:main_rate}.
\end{proof}
\subsection{Proof of \Cref{corollary:time_complexity}}\label{proof:time_complexity}
Here we provide the proof of \Cref{corollary:time_complexity}.
\begin{proof}
Under the assumptions of \Cref{thm:main}, we have
\begin{equation*}
        \Exp{\norm{\nabla f(\bar x^{\hat{r}})}^2} = \frac{1}{R} \sum_{r=0}^{R-1} \Exp{\norm{\nabla f(\bar x^r)}^2} \le c_0 \cdot \frac{\Delta}{\eta \bar H R} + c_1 \cdot \frac{\eta L \sigma^2}{n} + c_2 \cdot \eta^2 \frac{L^2 G^2 X}{n \bar H} .
\end{equation*}

For the variance term to satisfy $c_1 \frac{\eta L \sigma^2}{n} \le \frac{\varepsilon}{3}$, we require $\eta \le \frac{1}{3 c_1} \frac{\varepsilon n}{L \sigma^2}$ when $\sigma^2 > 0$. 
For the disagreement drift term to satisfy $c_2 \eta^2 \frac{L^2 G^2 X}{n \bar H} \le \frac{\varepsilon}{3}$, we require $\eta \le \frac{1}{\sqrt{3 c_2}} \frac{\sqrt{\varepsilon n \bar H}}{L G \sqrt{X}}$ when $X > 0$.
The largest stepsize satisfying these conditions as well as the stepsize bound from \Cref{thm:main} is
\begin{equation*}
    \hat\eta = \min\nbr{\frac{1}{4 c_1} \frac{\varepsilon n}{L \sigma^2}, \frac{1}{\sqrt{4 c_2}} \frac{\sqrt{\varepsilon n \bar H}}{L G \sqrt{X}}, \frac{1}{8 L H_{\max}}},
\end{equation*}
where the terms involving $\sigma^2$ and $X$ are understood to evaluate to $+ \infty$ when these quantities are zero.

To ensure the remaining term satisfies $c_0 \frac{\Delta}{\eta \bar H R} \le \frac{\varepsilon}{3}$, we must have $R \ge \frac{3 c_0 \Delta}{\varepsilon \hat\eta \bar H}$. 
Substituting the evaluated components of $\hat\eta$ directly yields the explicit round boundary:
\begin{equation*}
    R \ge c_R \left( \frac{\Delta L \sigma^2}{n \varepsilon \bar H} + \frac{\Delta L G \sqrt{X}}{\varepsilon^{3/2} \bar H \sqrt{n \bar H}} + \frac{\Delta L H_{\max}}{\varepsilon \bar H} \right)
\end{equation*}
for an appropriate absolute constant $c_R > 0$.

To convert this round complexity to wall-clock time complexity, we multiply the total elapsed rounds $R$ by the round duration, $T = M \tau + \zeta$. 
Using $\frac{T}{\bar H} = \tau_H \coloneqq \frac{n}{\sum_{i=1}^n \tau_i^{-1}}$, we obtain the stated bound:
\begin{equation*}
    \mathcal{O}\rbr{\frac{\Delta L \sigma^2}{n \varepsilon} \cdot \tau_H + \frac{\Delta L G \sqrt{X}}{\varepsilon^{3/2} \sqrt{n \bar H}} \cdot \tau_H + \frac{\Delta L H_{\max}}{\varepsilon} \cdot \tau_H}.
\end{equation*}
\end{proof}
\section{Additional Experimental Details}
\label{sec:appendix-experiments}

We empirically evaluate sparse local SGD with communication--computation overlap mainly on binary classification problems from LIBSVM, and additionally on CIFAR-10 and Tiny ImageNet image classification with neural networks. 
The goal of the experimental section is to quantify the separate effects of local computation, sparse synchronization, and overlap under heterogeneous worker speeds. 
We compare five methods:
\begin{itemize}
    \item \textbf{Sync SGD}: no local drift, no sparsity, and no overlap.
    \item \textbf{FedAvg-Full}: local drift with full model averaging.
    \item \textbf{Local Sparse}: local drift with sparse model averaging.
    \item \textbf{Overlap overwrite}: local drift, sparse model averaging, and communication--computation overlap.
    \item \textbf{Overlap delay-corrected}: local drift, sparse model averaging, overlap, and corrected merging.
\end{itemize}
The last two methods are our proposed overlap variants. 
The dense baselines, \textbf{Sync SGD} and \textbf{FedAvg-Full}, provide reference points for the cost of local drift and sparse communication, while \textbf{Local Sparse} isolates the additional effect of blocking sparse synchronization.

\paragraph{Setup.}
We use binary logistic regression with labels mapped to $\{-1,+1\}$. 
Given $n$ workers, the training data are split into worker shards using a fixed seed when partitioning is enabled. 
All methods are run with the same initialization, data split, worker partition, and random seeds.

For an example $(x,y)$, the logistic loss is
\begin{equation}
    \ell(w;x,y)
    =
    \log\left(1+\exp\left(-y x^\top w\right)\right),
\end{equation}
where $w \in \mathbb{R}^{d}$ denotes the model parameter. 
We consider the finite-sum objective
\begin{equation}
    F(w)
    =
    \sum_{i=1}^{n}
    \frac{m_i}{m}
    F_i(w),
    \qquad
    m = \sum_{i=1}^{n} m_i,
\end{equation}
where
\begin{equation}
    F_i(w)
    =
    \frac{1}{m_i}
    \sum_{r=1}^{m_i}
    \ell(w;x_{i,r},y_{i,r})
    +
    R(w).
\end{equation}
When studying the nonconvex setting, we use the coordinate-wise Geman--McClure regularizer
\begin{equation}
    R(w)
    =
    \lambda \sum_{j=1}^{d}
    \frac{w_j^2}{w_j^2+\theta^2},
\end{equation}
with $\lambda$ and $\theta$ specified in each experiment. 
This regularizer introduces nonconvexity while preserving a simple logistic-regression structure, allowing us to evaluate the methods using both predictive and stationarity-based metrics. 
We also ran the same comparisons in this nonconvex regularized setting, but do not report those plots because they were qualitatively similar to the logistic-regression results and did not add separate conclusions.
Features are standardized using training-set statistics when normalization is enabled.

\paragraph{Metrics.}
We report training accuracy, validation accuracy, and the full training gradient norm $\|\nabla F(w)\|$. 
Accuracy measures predictive performance, while the gradient norm measures optimization progress and is especially relevant in the nonconvex regularized setting. 
We plot performance as a function of communication rounds and cumulative logical time. 
Round-based plots measure progress per synchronization event, whereas logical-time plots capture the main motivation for overlap: reducing idle communication time by allowing workers to compute while sparse messages are in flight. 
For the CIFAR-10 and Tiny ImageNet experiments, we also plot against cumulative processed examples.
This axis ignores simulated communication time and counts how many training examples have been consumed by the logical workers.
It therefore serves as a proxy for the number of stochastic-gradient steps and, approximately, for the amount of local computation or floating-point work used by each method.
These plots answer a different question from the logical-time plots: whether the overlap methods only look better because they can process more mini-batches during communication delays, or whether those extra mini-batch updates remain useful when methods are compared after consuming the same amount of data.
For multi-seed experiments, curves show the mean across seeds, with shaded regions indicating variability across runs.

\paragraph{Implementation.}
The experiments are implemented in a PyTorch-based logical-time simulation framework.
The framework is not a real distributed system: workers, communication, and delays are simulated deterministically inside one process.
This design lets us isolate algorithmic effects such as local computation, sparsity, communication delay, worker-speed heterogeneity, and communication--computation overlap without noise from hardware scheduling or networking.
Each worker is represented by a separate PyTorch model copy, and the simulator explicitly tracks the round states $x_i^r$, $y_i^r$, and $z_i^r$.
In round $r$, worker $i$ starts from $x_i^r$, performs $N_i$ local SGD steps to obtain $y_i^r$, sends a sparse masked model vector, and, for the overlap methods, continues for $Q_i$ additional steps to obtain $z_i^r$.
The server averages the sparse messages on a shared random mask, and workers then merge the incoming sparse average either by overwrite or by delay correction.

Logical time is controlled by integer worker step times $\tau_i$.
For overlap experiments, the simulator computes $\tau=\operatorname{lcm}(\tau_1,\ldots,\tau_n)$, uses a local compute window of length $M\tau$, and takes the communication delay $\zeta$ to be a multiple of $\tau$.
This gives per-worker step counts $N_i=M\tau/\tau_i$ before communication and $Q_i=\zeta/\tau_i$ during communication, so faster workers naturally perform more local computation within the same logical-time window.
The codebase separates timing, mask sampling, compression, communication accounting, and merge rules, and includes engines for the proposed overlap methods, blocking \textbf{Local Sparse}, dense \textbf{FedAvg-Full}, and \textbf{Sync SGD}.
Experiments are driven by YAML configurations and record optimization metrics, validation metrics, worker disagreement, processed examples, cumulative logical time, and logical communication cost in both coordinates and bits.

\paragraph{Experimental organization.}
We first compare all five methods on LIBSVM binary classification tasks to quantify the cost of local drift, the cost of sparse synchronization, and the benefit of overlap. 
We then focus on the sparse methods, \textbf{Local Sparse}, \textbf{Overlap overwrite}, and \textbf{Overlap delay-corrected}, to isolate the impact of the two proposed overlap mechanisms. 
Finally, we perform ablations over the sparsity level, local computation budget, communication delay, and worker heterogeneity.
We also include CIFAR-10 and Tiny ImageNet neural-network experiments to check that the same qualitative behavior appears beyond logistic regression.

\subsection{External experiments}
\label{subsec:a9a-sparse-overlap}

We begin with a controlled sparse-communication experiment on \texttt{a9a}, comparing only 
\textbf{Local Sparse}, \textbf{Overlap overwrite}, and \textbf{Overlap delay-corrected}. 
This experiment isolates the effect of using the communication window for additional local computation, since all three methods use the same sparsity level and the same logical round duration.

We use $n=4$ workers with heterogeneous step times
\begin{equation}
    (\tau_1,\tau_2,\tau_3,\tau_4) = (1,2,3,6),
\end{equation}
so that $\tau=\operatorname{lcm}(\tau_1,\ldots,\tau_n)=6$. 
All methods are run for $20$ communication rounds with learning rate $\eta=0.1$, batch size $256$, sparsity level $p=0.3$, and no regularization, i.e., $\lambda=0$. 
The training set is normalized using training-set statistics, and $10\%$ of the training data is used for validation.

For \textbf{Local Sparse}, each round consists of a compute window of length $18$ followed by a communication window of length $6$, during which workers block. 
For the two overlap methods, we set $M=3$ and $\zeta=6$, giving the same round duration
\begin{equation}
    M\tau+\zeta = 3\cdot 6 + 6 = 24.
\end{equation}
In this controlled comparison, the three sparse methods also communicate the same number of coordinates per round. 
Since they have the same round duration and the same per-round communication volume, comparing performance as a function of communication rounds is already sufficient: logical time and logical communication cost are common rescalings of the round index. 
We nevertheless include time- and communication-based plots as consistency checks.
Thus, the pre-communication compute phase is identical across the three methods, while the overlap methods additionally exploit the communication window for local updates.

Under these timings, the pre-communication local steps are
\begin{equation}
    (N_1,N_2,N_3,N_4)
    =
    \left(
    \left\lfloor \frac{18}{1} \right\rfloor,
    \left\lfloor \frac{18}{2} \right\rfloor,
    \left\lfloor \frac{18}{3} \right\rfloor,
    \left\lfloor \frac{18}{6} \right\rfloor
    \right)
    =
    (18,9,6,3).
\end{equation}
During communication, the overlap methods perform the additional steps
\begin{equation}
    (Q_1,Q_2,Q_3,Q_4)
    =
    \left(
    \left\lfloor \frac{6}{1} \right\rfloor,
    \left\lfloor \frac{6}{2} \right\rfloor,
    \left\lfloor \frac{6}{3} \right\rfloor,
    \left\lfloor \frac{6}{6} \right\rfloor
    \right)
    =
    (6,3,2,1).
\end{equation}
Therefore, any improvement of the overlap methods over \textbf{Local Sparse} comes from replacing idle communication time by useful local computation. 
The comparison between \textbf{Overlap overwrite} and \textbf{Overlap delay-corrected} then measures the benefit of correcting the merge rather than simply overwriting the synchronized coordinates.

Across all reported metrics and plotting parametrizations considered in this experiment, we observe the same clear ordering:
\[
    \textbf{Overlap delay-corrected}
    \quad \text{outperforms} \quad
    \textbf{Overlap overwrite}
    \quad \text{outperforms} \quad
    \textbf{Local Sparse}.
\]
The separation is significant: for accuracy metrics the delay-corrected variant reaches higher values earlier, while for loss and gradient-norm metrics it decreases the objective and stationarity measure faster. 
The same qualitative behavior has also been observed in many additional settings, including nonconvex regularized objectives, very heterogeneous worker-speed profiles, compute-dominated regimes in which local steps are expensive relative to communication, and communication-dominated regimes in which communication delay is the main bottleneck.

This ordering has a simple interpretation. 
\textbf{Local Sparse} wastes the communication window by waiting, so it performs fewer useful local updates within the same logical round duration. 
\textbf{Overlap overwrite} uses this window for additional computation, which explains its consistent improvement over the blocking baseline, but it then discards part of this progress on the synchronized coordinates by overwriting them with a delayed sparse average. 
\textbf{Overlap delay-corrected} keeps the useful overlap drift and only corrects the disagreement induced by sparse synchronization, so it benefits from the extra computation without paying the full cost of stale overwriting. 
The empirical hierarchy therefore matches the mechanism predicted by the method: overlap is valuable, and delay correction is needed to fully exploit it.
The accuracy plots in \Cref{fig:exp1-accuracy} show this ordering on both training and validation accuracy, \Cref{fig:exp1-gradient} shows the same behavior for stationarity, and \Cref{fig:exp1-loss-rescalings} confirms that the loss comparison is unchanged when rounds are rescaled by logical time or communication cost.

\begin{figure}[H]
    \centering
    \begin{minipage}{0.48\linewidth}
        \centering
        \includegraphics[width=\linewidth]{images/exp1/train_loss_vs_rounds.png}\\
        {\small (a) Training accuracy versus round.}
    \end{minipage}\hfill
    \begin{minipage}{0.48\linewidth}
        \centering
        \includegraphics[width=\linewidth]{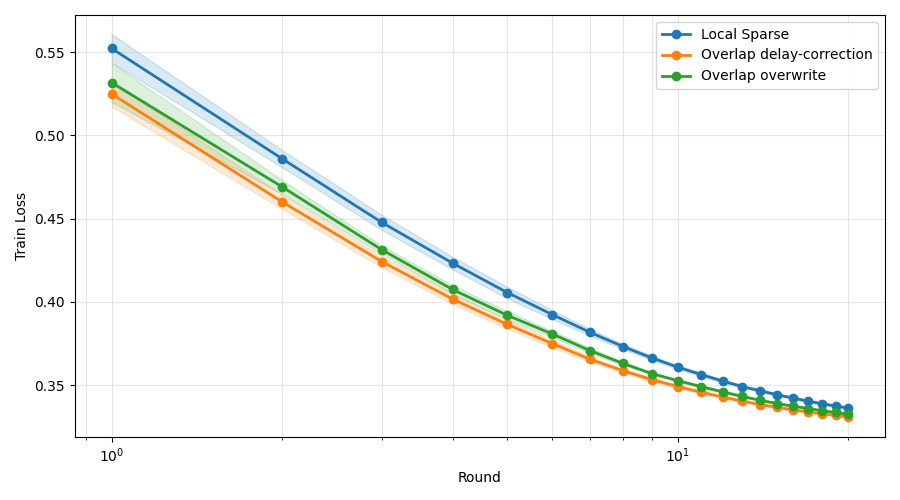}\\
        {\small (b) Training accuracy versus round, log scale.}
    \end{minipage}

    \begin{minipage}{0.48\linewidth}
        \centering
        \includegraphics[width=\linewidth]{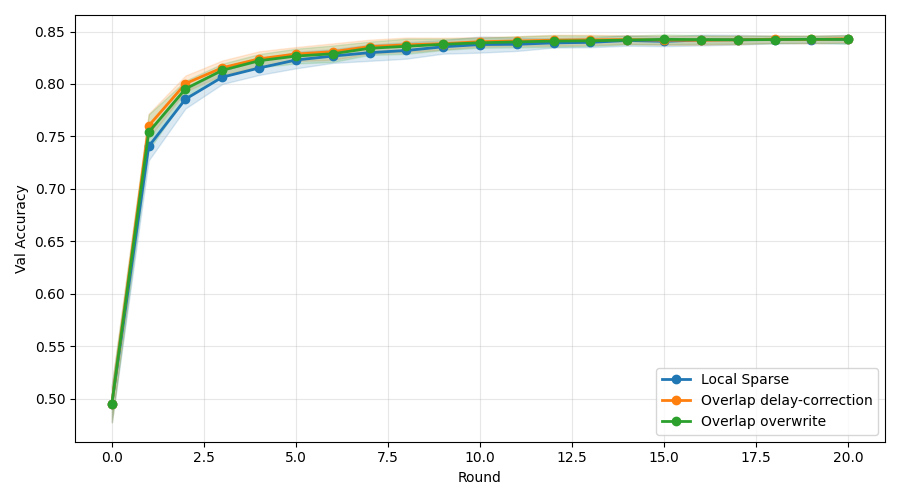}\\
        {\small (c) Validation accuracy versus round.}
    \end{minipage}\hfill
    \begin{minipage}{0.48\linewidth}
        \centering
        \includegraphics[width=\linewidth]{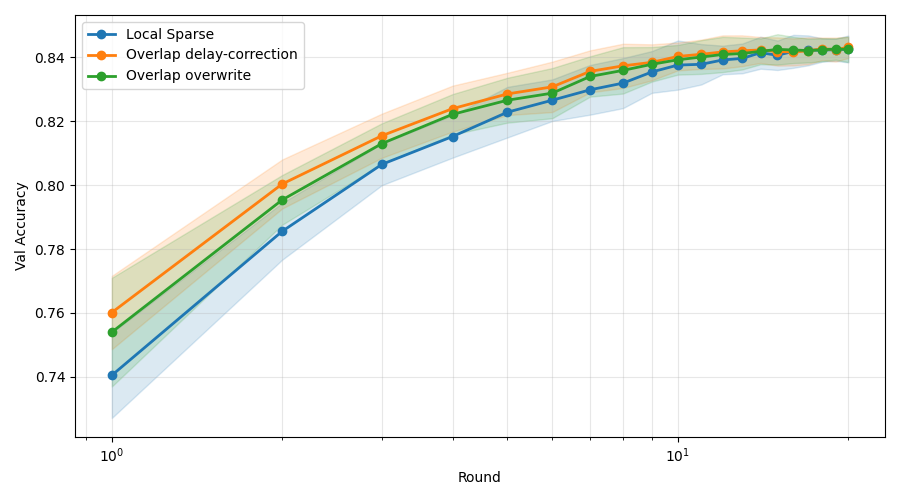}\\
        {\small (d) Validation accuracy versus round, log scale.}
    \end{minipage}
    \caption{Accuracy comparison for the controlled sparse-overlap experiment on \texttt{a9a}. 
    The delay-corrected overlap rule reaches higher accuracy earlier than overwrite, and both overlap rules improve over the blocking local-sparse baseline.}
    \label{fig:exp1-accuracy}
\end{figure}

\begin{figure}[H]
    \centering
    \includegraphics[width=0.62\linewidth]{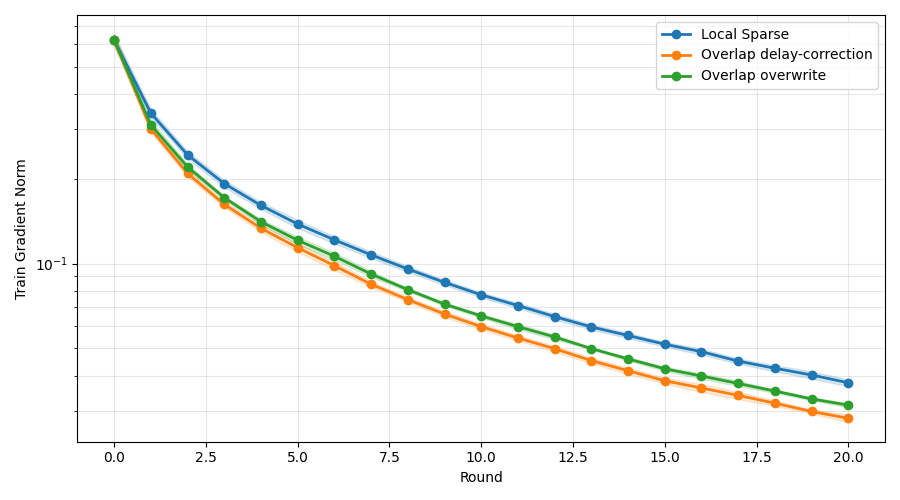}
    \caption{Training gradient norm for the controlled sparse-overlap experiment on \texttt{a9a}. 
    The stationarity metric follows the same ordering as the accuracy plots: delay correction makes the fastest progress, followed by overwrite and then blocking local sparse averaging.}
    \label{fig:exp1-gradient}
\end{figure}

\begin{figure}[H]
    \centering
    \begin{minipage}{0.48\linewidth}
        \centering
        \includegraphics[width=\linewidth]{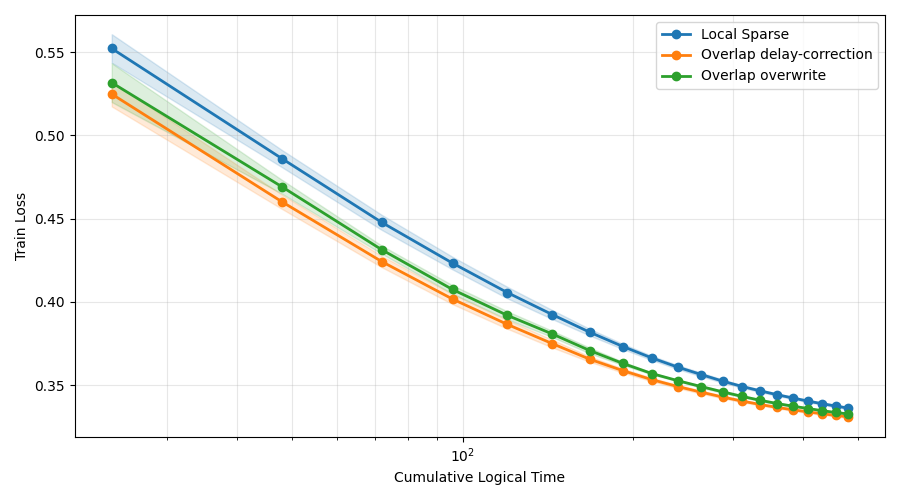}\\
        {\small (a) Training loss versus logical time.}
    \end{minipage}\hfill
    \begin{minipage}{0.48\linewidth}
        \centering
        \includegraphics[width=\linewidth]{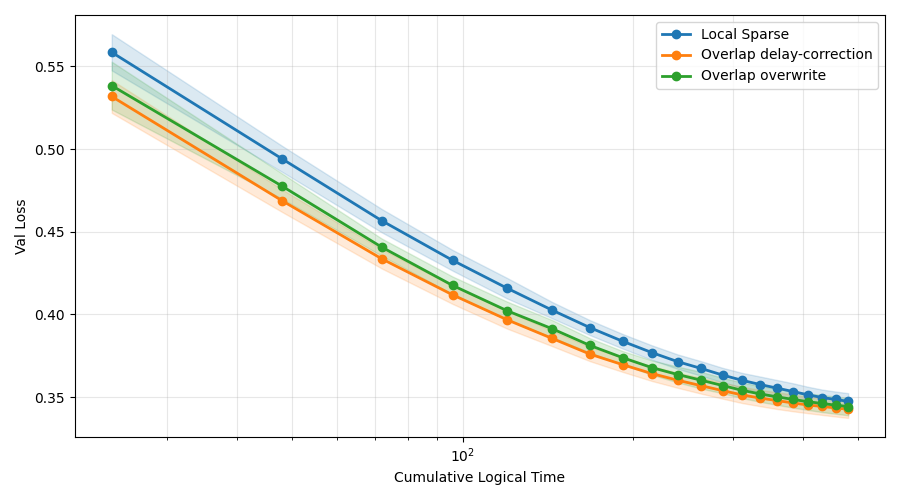}\\
        {\small (b) Validation loss versus logical time.}
    \end{minipage}

    \begin{minipage}{0.48\linewidth}
        \centering
        \includegraphics[width=\linewidth]{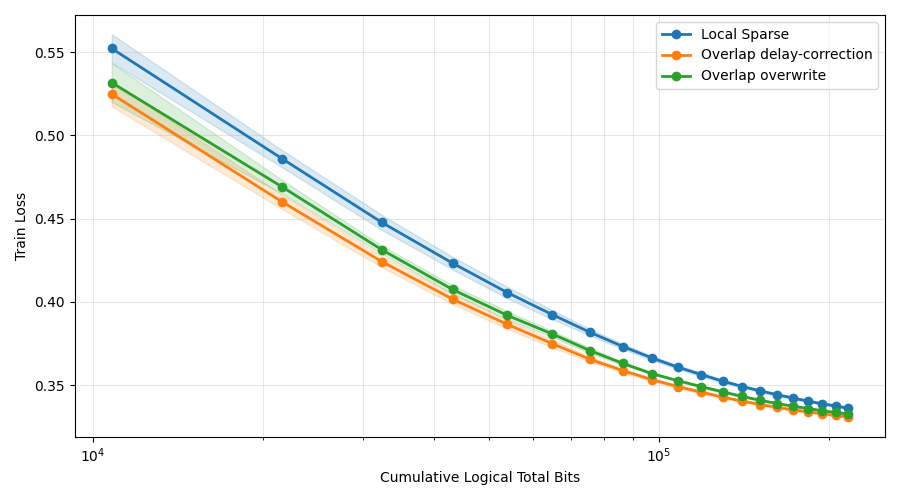}\\
        {\small (c) Training loss versus communication cost.}
    \end{minipage}\hfill
    \begin{minipage}{0.48\linewidth}
        \centering
        \includegraphics[width=\linewidth]{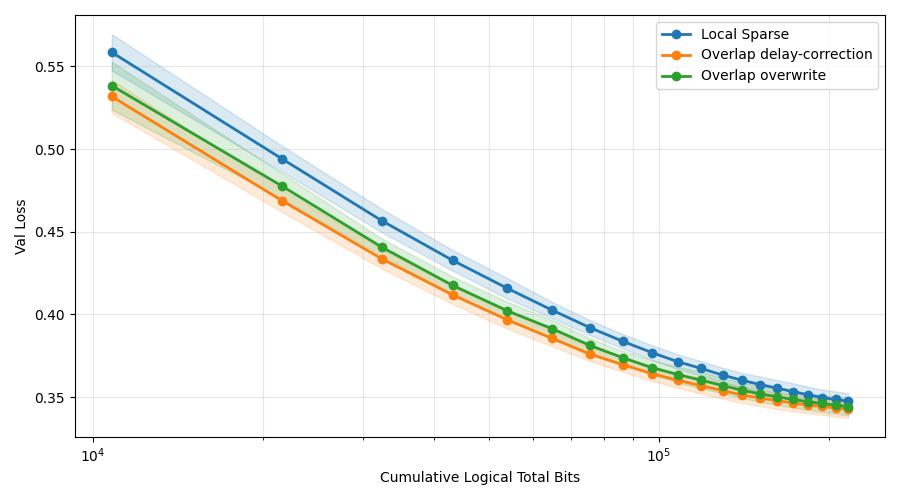}\\
        {\small (d) Validation loss versus communication cost.}
    \end{minipage}
    \caption{Loss curves under equivalent resource parametrizations for the controlled sparse-overlap experiment. 
    Since all three sparse methods use the same round duration and per-round communication volume, logical time and communication cost preserve the round-based ordering.}
    \label{fig:exp1-loss-rescalings}
\end{figure}

\subsection{Merge-rule comparison across regimes}
\label{subsec:merge-rule-comparison}

We next isolate the effect of the delayed merge rule by comparing only the two overlap variants, \textbf{Overlap overwrite} and \textbf{Overlap delay-corrected}, across three additional settings. 
Within each setting, the two methods use the same data split, sparsity level, communication schedule, worker speeds, and random seeds; the only difference is how the delayed sparse average is merged after the overlap phase. 
This comparison is designed to answer a direct question: when communication-computation overlap is used, should the incoming delayed average overwrite the synchronized coordinates, or should it be corrected to preserve the local progress made while communication was in flight?

All three settings use \texttt{a9a} logistic regression with $20$ rounds, learning rate $\eta=0.1$, batch size $256$, sparsity level $p=0.3$, normalized features, and no regularization. 
The settings differ only in the local-computation budget, communication delay, and worker-speed profile:
\[
\begin{array}{c|c|c|c}
\text{setting} & M & \zeta & (\tau_1,\tau_2,\tau_3,\tau_4) \\
\hline
\text{long compute, short delay} & 8 & 6 & (1,2,3,6) \\
\text{short compute, long delay} & 2 & 24 & (1,2,3,6) \\
\text{very heterogeneous speeds} & 2 & 20 & (1,2,10,20).
\end{array}
\]

Across all three settings, delay correction gives lower training and validation loss than overwrite. 
The size of the advantage is not identical across regimes: in the long-compute, short-delay setting the gain is small but persistent; in the short-compute, long-delay setting it is large and visible throughout training; and under very heterogeneous worker speeds it is moderate. 
This variation is useful, because it shows that delay correction is not only helpful in an especially favorable configuration. 
Even when the two merge rules are close, the corrected rule consistently avoids the loss incurred by discarding overlap-phase progress on synchronized coordinates.
The three regimes are shown in \Cref{fig:exp2-long-compute-short-delay,fig:exp2-short-compute-long-delay,fig:exp2-very-heterogeneous}; each figure places the training and validation losses side by side so that the text and plots can be read together.

\begin{figure}[H]
    \centering
    \begin{minipage}{0.48\linewidth}
        \centering
        \includegraphics[width=\linewidth]{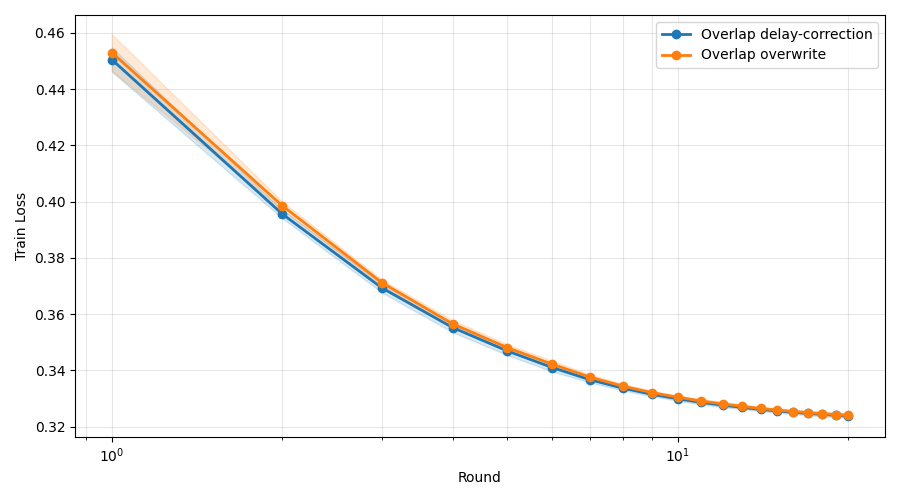}\\
        {\small (a) Training loss.}
    \end{minipage}\hfill
    \begin{minipage}{0.48\linewidth}
        \centering
        \includegraphics[width=\linewidth]{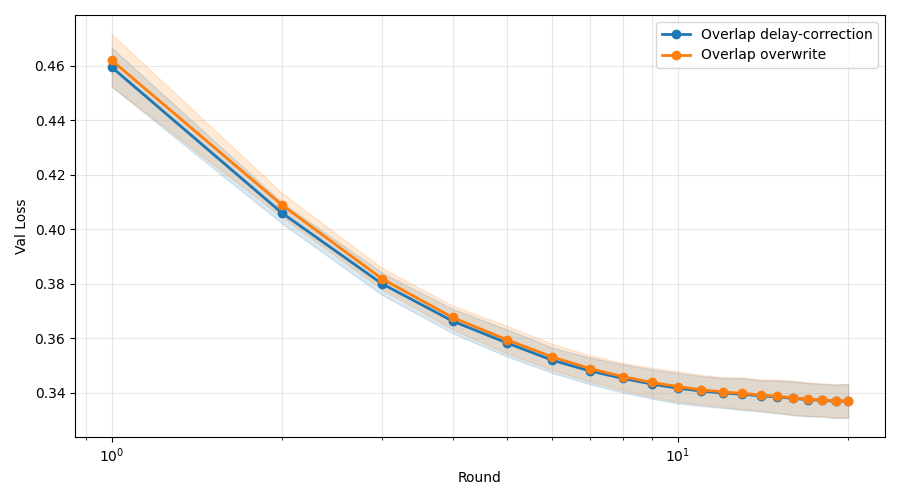}\\
        {\small (b) Validation loss.}
    \end{minipage}
    \caption{Long-compute, short-delay regime ($M=8$, $\zeta=6$, worker times $(1,2,3,6)$). 
    Delay correction gives a small but persistent improvement over overwrite on both training and validation loss.}
    \label{fig:exp2-long-compute-short-delay}
\end{figure}

\begin{figure}[H]
    \centering
    \begin{minipage}{0.48\linewidth}
        \centering
        \includegraphics[width=\linewidth]{images/exp2_2/train_loss_vs_rounds.png}\\
        {\small (a) Training loss.}
    \end{minipage}\hfill
    \begin{minipage}{0.48\linewidth}
        \centering
        \includegraphics[width=\linewidth]{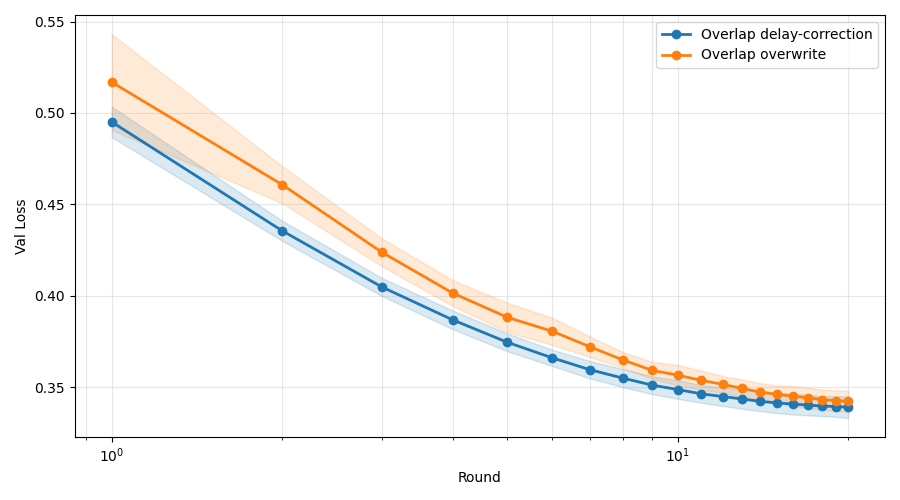}\\
        {\small (b) Validation loss.}
    \end{minipage}
    \caption{Short-compute, long-delay regime ($M=2$, $\zeta=24$, worker times $(1,2,3,6)$). 
    This is the regime with the largest gap: overwrite discards a substantial amount of overlap-phase progress, while delay correction preserves it.}
    \label{fig:exp2-short-compute-long-delay}
\end{figure}

\begin{figure}[H]
    \centering
    \begin{minipage}{0.48\linewidth}
        \centering
        \includegraphics[width=\linewidth]{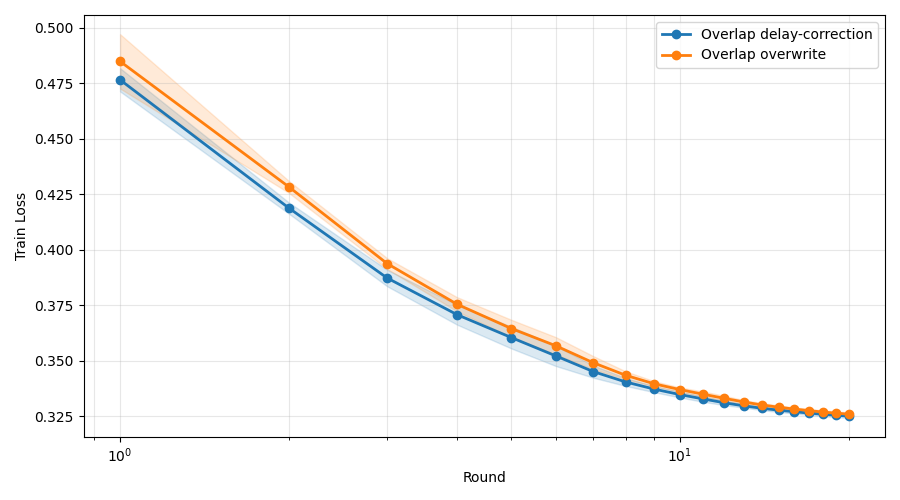}\\
        {\small (a) Training loss.}
    \end{minipage}\hfill
    \begin{minipage}{0.48\linewidth}
        \centering
        \includegraphics[width=\linewidth]{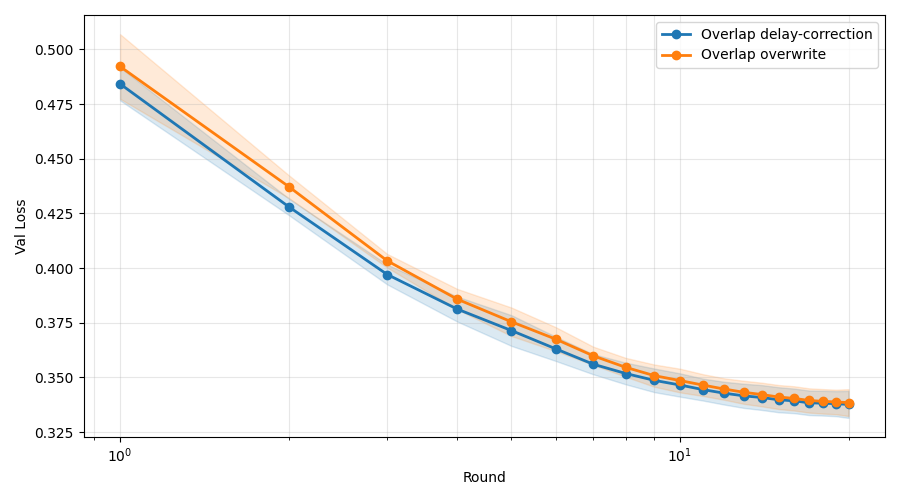}\\
        {\small (b) Validation loss.}
    \end{minipage}
    \caption{Very heterogeneous worker-speed regime ($M=2$, $\zeta=20$, worker times $(1,2,10,20)$). 
    Delay correction keeps a moderate advantage over overwrite, showing that the corrected merge remains useful when worker speeds vary substantially.}
    \label{fig:exp2-very-heterogeneous}
\end{figure}

\subsection{Sparsity-level ablation}
\label{subsec:sparsity-ablation}

We now vary the sparsity level $p$, which controls the fraction of model coordinates communicated at each synchronization.
The goal is to understand the communication-accuracy tradeoff induced by sparse parameter averaging.
We use the same \texttt{a9a} logistic-regression setup with $n=4$ workers, worker times $(1,2,3,6)$, $20$ rounds, learning rate $\eta=0.1$, batch size $256$, normalized features, and no regularization.
The local-computation and delay parameters are fixed to $M=4$ and $\zeta=6$, while the sparsity level is varied over
\[
    p \in \{0.001,0.01,0.1,1.0\}.
\]
We report training loss as a function of cumulative logical total bits in \Cref{fig:exp4-sparsity-ablation}, separately for \textbf{Local Sparse}, \textbf{Overlap overwrite}, and \textbf{Overlap delay-corrected}.

This ablation is best read horizontally, at a fixed communication budget.
Reducing $p$ shifts the curves left by orders of magnitude, because each communication round transmits fewer coordinates.
At the same time, the shape of the loss curve changes only mildly: even very sparse communication continues to make steady progress.
Thus, on this homogeneous-data \texttt{a9a} experiment, aggressive sparsification is highly communication-efficient.
For example, the curves with $p=0.001$ and $p=0.01$ reach losses close to the denser settings while using far fewer total bits.

There is still a tradeoff.
Larger $p$ communicates more of the model and therefore can slightly reduce the final loss after the same number of rounds, especially when comparing the endpoints of the curves.
However, this improvement is expensive in communication: the $p=1.0$ curve corresponds to dense parameter averaging and lies far to the right.
The main conclusion is therefore that sparsity buys a large reduction in communication cost with only a modest optimization penalty in this setting.
The conclusion is consistent across the three sparse methods, and the delay-corrected overlap rule remains at least as good as overwrite while preserving the same communication savings.

\begin{figure}[H]
    \centering
    \begin{minipage}{0.48\linewidth}
        \centering
        \includegraphics[width=\linewidth]{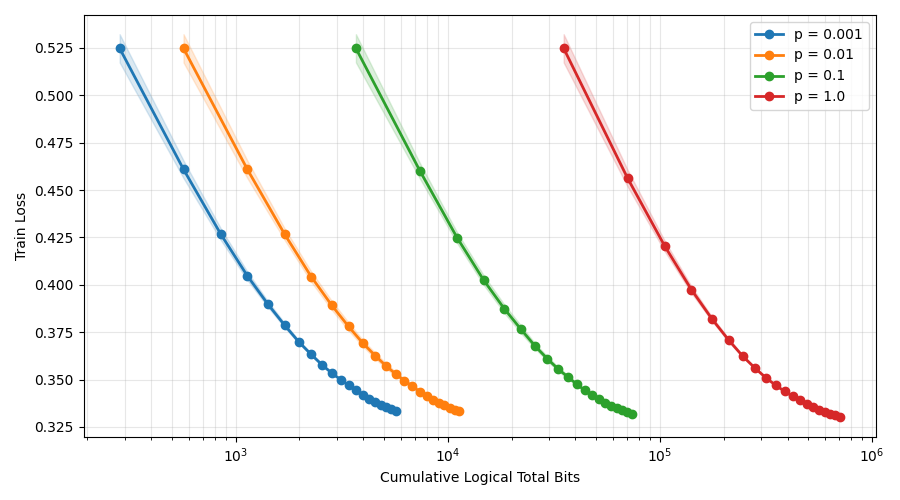}\\
        {\small (a) Local Sparse.}
    \end{minipage}\hfill
    \begin{minipage}{0.48\linewidth}
        \centering
        \includegraphics[width=\linewidth]{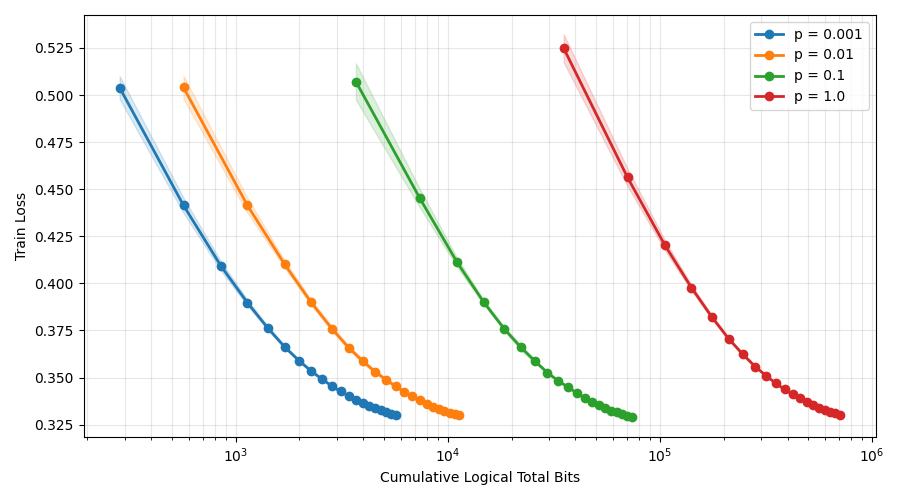}\\
        {\small (b) Overlap overwrite.}
    \end{minipage}

    \begin{minipage}{0.62\linewidth}
        \centering
        \includegraphics[width=\linewidth]{images/exp4/overlap_delay_corrected_train_loss_vs_bits.png}\\
        {\small (c) Overlap delay-corrected.}
    \end{minipage}
    \caption{Sparsity-level ablation on \texttt{a9a}.
    We vary $p\in\{0.001,0.01,0.1,1.0\}$ with $M=4$, $\zeta=6$, worker times $(1,2,3,6)$, and otherwise identical training settings.
    Smaller $p$ dramatically reduces the cumulative number of communicated bits while preserving a similar loss trajectory, showing that sparse parameter averaging gives a strong communication-accuracy tradeoff in this homogeneous-data regime.}
    \label{fig:exp4-sparsity-ablation}
\end{figure}

\subsection{Ablation on the local computation budget}
\label{subsec:ablation-M}
We next study the effect of the local computation budget $M$.
We vary $M \in \{1,4,16,64\}$ and plot the training loss against cumulative
logical time in \Cref{fig:exp5-m-ablation}.
All runs use \texttt{a9a} logistic regression with $20$ rounds, learning rate
$\eta=0.1$, sparsity level $p=0.3$, communication delay $\zeta=6$, batch size
$256$, normalized features, no regularization, and worker times $(1,2,3,6)$.
All curves start from the same initial model at logical time zero.
This initial point is not shown because the horizontal axis is logarithmic,
and $\log(0)$ is undefined.
We keep the raw time values rather than artificially shifting or transforming
the data to display the initial point.
Consequently, the first visible point for larger values of $M$ appears later,
since one communication round has a larger logical duration.

\begin{figure}[H]
    \centering
    \begin{minipage}{0.48\linewidth}
        \centering
        \includegraphics[width=\linewidth]{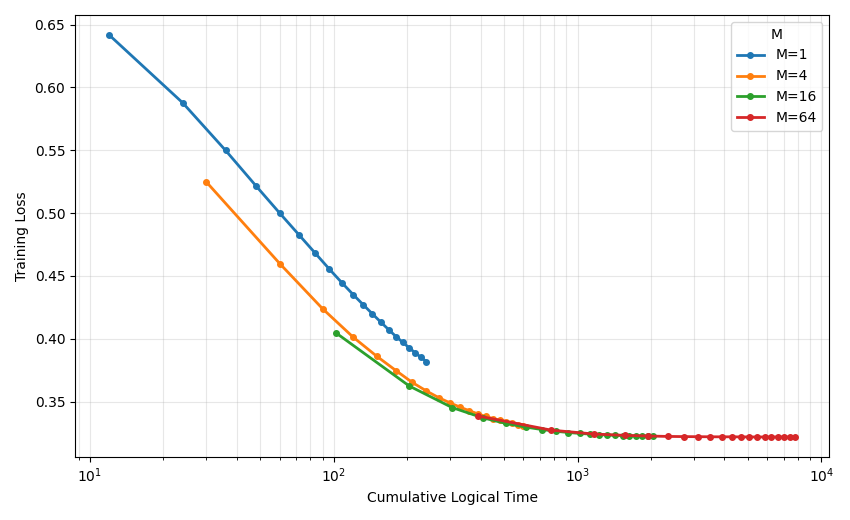}\\
        {\small (a) Local Sparse.}
    \end{minipage}\hfill
    \begin{minipage}{0.48\linewidth}
        \centering
        \includegraphics[width=\linewidth]{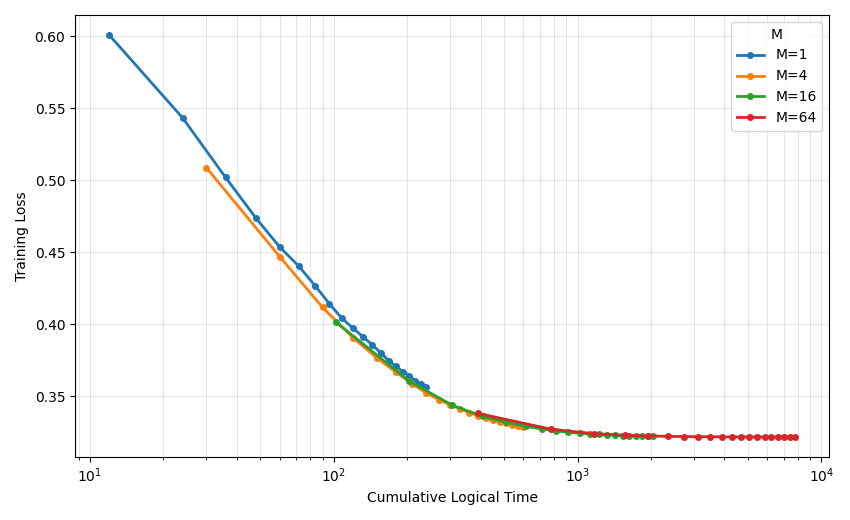}\\
        {\small (b) Overlap overwrite.}
    \end{minipage}

    \begin{minipage}{0.62\linewidth}
        \centering
        \includegraphics[width=\linewidth]{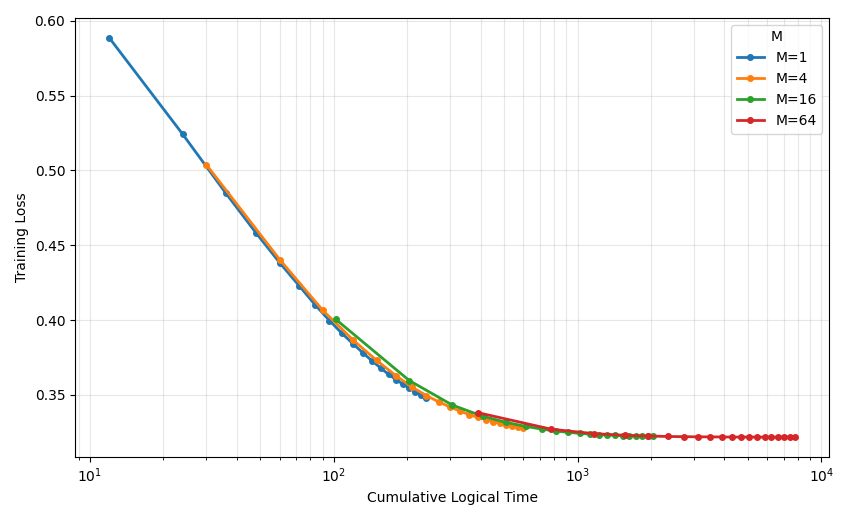}\\
        {\small (c) Overlap delay-corrected.}
    \end{minipage}
    \caption{Ablation on the local computation budget.
    The blocking local-sparse baseline is sensitive to the local computation budget in early logical time, while the overlap methods are more stable.
    The delay-corrected rule gives the strongest collapse across the tested budgets.}
    \label{fig:exp5-m-ablation}
\end{figure}

For the blocking \textbf{Local Sparse} baseline, the early training-loss
trajectory is sensitive to $M$.
Small values of $M$ synchronize more frequently and therefore spend a larger
fraction of logical time in blocking communication.
Larger values, especially $M=16$ and $M=64$, amortize this communication cost
over more local computation and reach lower training loss faster in logical
time.
However, all tested values eventually approach a similar loss floor.

The overlap methods show a weaker dependence on $M$.
For \textbf{Overlap overwrite}, the curves for different values of $M$ are much
closer than for the blocking baseline.
This suggests that once communication is overlapped with local computation, the
local computation budget has a smaller effect on training-loss decrease per
unit logical time.
The remaining differences are mostly visible in the early phase, while the
curves nearly coincide as training progresses.

The strongest invariance is observed for \textbf{Overlap delay-corrected}.
Across $M \in \{1,4,16,64\}$, the loss curves almost collapse when plotted
against cumulative logical time.
This behavior is consistent with the delay-corrected merge preserving the
useful local progress made during the communication window.
In this regime, small values of $M$ do not suffer the same logical-time penalty
observed in the blocking baseline, while larger values of $M$ do not introduce
a visible training-loss penalty.

Overall, this ablation suggests that the blocking sparse method is sensitive to
$M$ mainly through communication amortization.
Communication--computation overlap reduces this sensitivity, and the
delay-corrected rule gives the most stable training-loss behavior across local
computation budgets.
In all cases, the methods reach nearly the same final training-loss floor, so
the main effect of $M$ is on early progress in logical time rather than final
training loss.
This conclusion is for the data-homogeneous regime studied here.
In strongly data-heterogeneous regimes, a large value of $M$ can be harmful for
all methods, because workers may take many consecutive steps toward different
local objectives before their models are realigned.
This client-drift effect is especially important when the data split is highly
non-i.i.d., and it should be controlled separately when choosing the local
computation budget.

\subsection{Effect of communication delay}
\label{subsec:communication-delay-ablation}

We next vary the communication delay $\zeta$ to test when overlap becomes most valuable.
This experiment compares the same three sparse methods as in \Cref{subsec:a9a-sparse-overlap}: \textbf{Local Sparse}, \textbf{Overlap overwrite}, and \textbf{Overlap delay-corrected}.
The point is different from the merge-rule-only comparison in \Cref{subsec:merge-rule-comparison}.
Here we ask how the amount of time spent in communication changes the value of overlapping computation with communication, and how much of that value is lost if the delayed sparse average is merged by overwriting.

We do not plot the case $\zeta=0$.
When $\zeta=0$, communication is instantaneous, so the overlap phase has length zero and the overlap methods perform no extra local steps while messages are in flight.
In that regime, overwrite and delay correction coincide with the blocking local-sparse update, and all three curves are identical under the same seed.
The informative cases are therefore the positive-delay regimes shown in \Cref{fig:exp3-delay-ablation-train}, where the communication window creates both an opportunity for useful extra computation and a stale-merge effect that must be handled.

For the moderate-delay case $\zeta=12$, the two overlap methods already improve over \textbf{Local Sparse} throughout most of training.
The delay-corrected rule is consistently best, although its margin over overwrite is relatively small by the end of training.
This suggests that when the communication window is long enough to make overlap useful but not dominant, the main gain comes from avoiding idle time, while delay correction provides an additional but more modest improvement.

The larger-delay case $\zeta=48$ makes the mechanism much clearer.
\textbf{Local Sparse} falls behind because it waits through a long communication window at every round, whereas both overlap methods turn that same time into additional local optimization.
Among the overlap methods, delay correction gives the lowest training loss across the run and has a visibly larger advantage over overwrite than in the $\zeta=12$ case.
This is consistent with the interpretation of the corrected merge: as $\zeta$ increases, workers make more progress during communication, so a naive overwrite discards more useful work on the synchronized coordinates.
Delay correction preserves this overlap-phase progress and therefore benefits more strongly when communication delay is large.

\begin{figure}[H]
    \centering
    \begin{minipage}{0.48\linewidth}
        \centering
        \includegraphics[width=\linewidth]{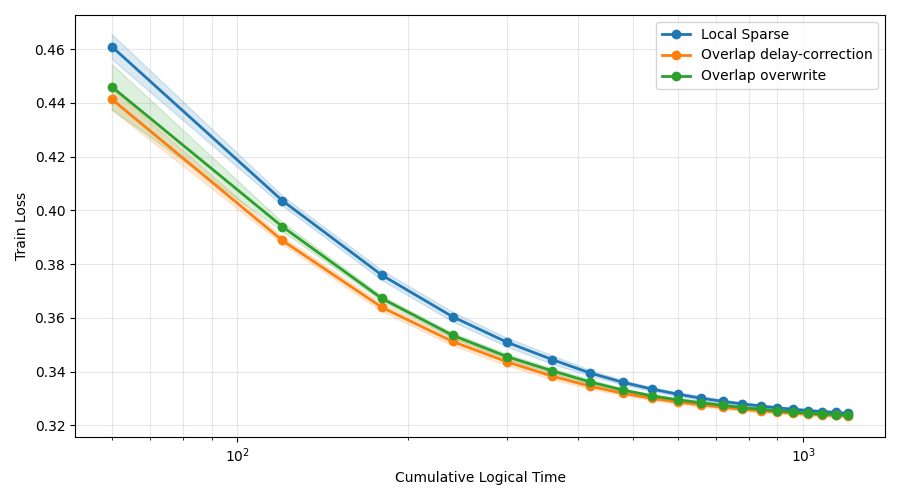}\\
        {\small (a) Moderate delay, $\zeta=12$.}
    \end{minipage}\hfill
    \begin{minipage}{0.48\linewidth}
        \centering
        \includegraphics[width=\linewidth]{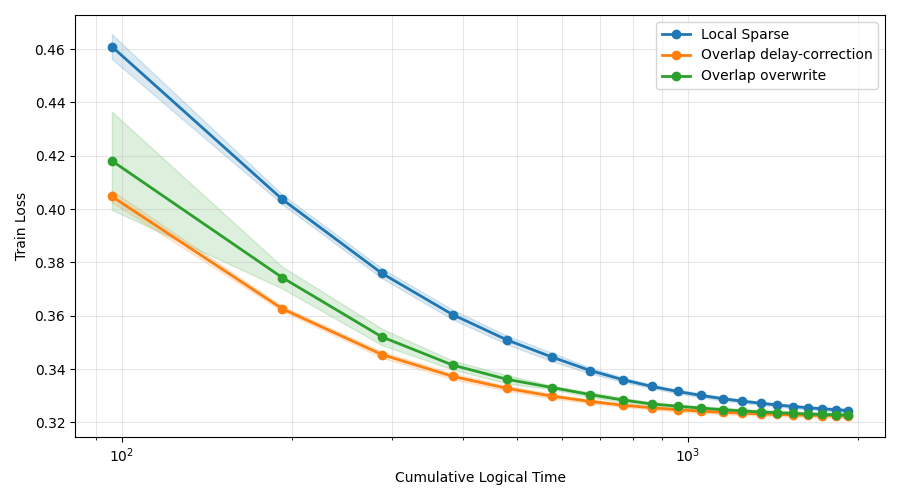}\\
        {\small (b) Large delay, $\zeta=48$.}
    \end{minipage}
    \caption{Training loss versus cumulative logical time for the communication-delay ablation.
    The zero-delay case is omitted because all three methods coincide when $\zeta=0$.
    For positive delay, overlap improves over blocking local sparse averaging, and the delay-corrected merge is consistently strongest; the advantage is especially pronounced when the communication delay is large.}
    \label{fig:exp3-delay-ablation-train}
\end{figure}

\paragraph{A heterogeneous-data stress test.}
The preceding delay ablation should not be read as saying that longer communication delay is always beneficial.
To illustrate the limitation, we also run an intentionally difficult non-i.i.d.\ experiment with a label-Dirichlet partition with concentration parameter $\alpha=0.03$.
This creates strongly heterogeneous worker objectives.
The run uses \texttt{a9a} logistic regression with $40$ rounds, learning rate $\eta=0.3$, sparsity level $p=0.01$, $M=1$, $\zeta=96$, and worker times $(1,2,3,6)$.
This is an extreme communication-dominated regime: before communication, the workers take
\[
    (N_1,N_2,N_3,N_4) = (6,3,2,1)
\]
local steps, while during the communication window the overlap methods take
\[
    (Q_1,Q_2,Q_3,Q_4) = (96,48,32,16)
\]
additional local steps.

The result in \Cref{fig:exp3-heterogeneous-negative} is negative for overlap.
Both overlap methods improve very quickly at the beginning, because the long communication window gives them many extra local updates.
However, with such heterogeneous data these extra updates are not simply useful progress toward the global objective: each worker is also pulled strongly toward its own local objective.
After this fast initial phase, the overlap curves plateau below \textbf{Local Sparse}; in particular, delay correction performs worst because it is designed to preserve the progress made during communication, and in this non-i.i.d.\ setting that progress can be biased toward local optima.
The blocking local-sparse method is slower initially, but by waiting through communication it takes far fewer local steps between synchronizations and eventually reaches better train and validation accuracy.

This experiment is outside the scope of our theory, which is developed for the data-homogeneous setting where all workers optimize the same objective.
Its role is therefore mainly diagnostic.
It shows that aggressive overlap can fail when communication is extremely slow, sparsity is very high, and the data partition is strongly heterogeneous.
In such regimes, one would need additional mechanisms, such as smaller learning rates, less aggressive overlap, stronger synchronization, or an analysis that explicitly controls client drift under heterogeneous objectives.

\begin{figure}[H]
    \centering
    \begin{minipage}{0.48\linewidth}
        \centering
        \includegraphics[width=\linewidth]{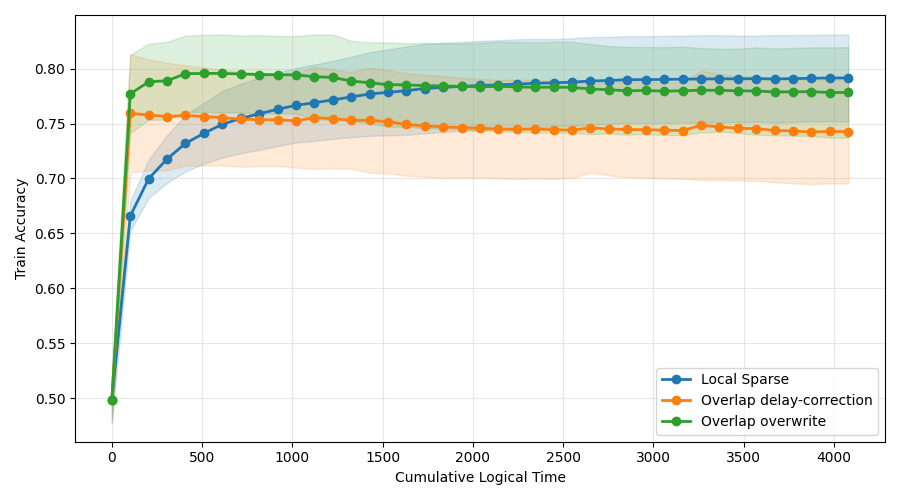}\\
        {\small (a) Training accuracy.}
    \end{minipage}\hfill
    \begin{minipage}{0.48\linewidth}
        \centering
        \includegraphics[width=\linewidth]{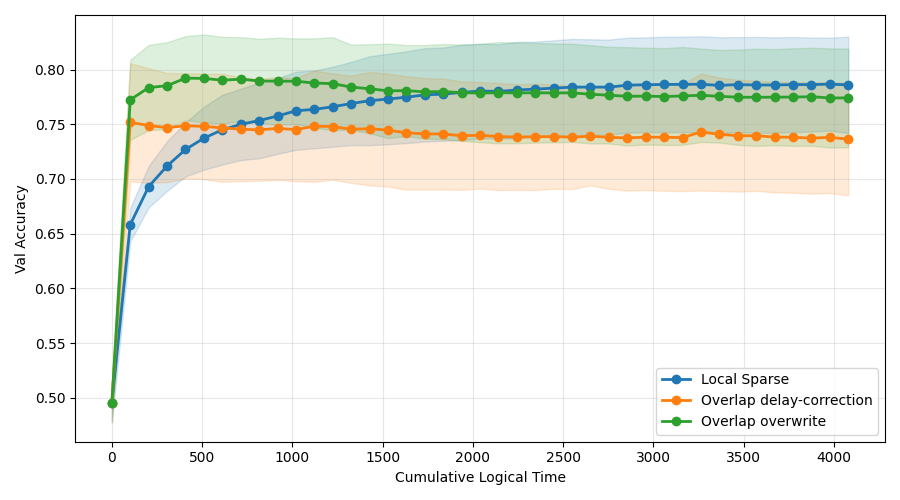}\\
        {\small (b) Validation accuracy.}
    \end{minipage}
    \caption{Negative stress test with strongly heterogeneous data and very slow communication.
    The data are partitioned by a label-Dirichlet split with $\alpha=0.03$, and the overlap methods use $M=1$, $\zeta=96$, and $p=0.01$.
    The overlap methods make rapid initial progress but then plateau below the blocking local-sparse baseline, indicating that too many local steps during communication can drive the methods toward poor local optima when worker objectives are highly heterogeneous.}
    \label{fig:exp3-heterogeneous-negative}
\end{figure}

\subsection{CIFAR-10 neural-network experiments}
\label{subsec:cifar-neural-network}

We also test the sparse overlap methods on CIFAR-10 image classification \citep{krizhevsky2009learning}.
The goal is to check whether the behavior observed for logistic regression persists for a non-linear model trained with the same logical-time simulation framework.
The model is a convolutional classifier trained with cross-entropy loss.
We use the same three sparse methods as above and run all methods with matched initialization, data split, and random seed.
The CIFAR-10 data are split into training, validation, and test sets, with validation held out from the training set.
To focus on optimization and communication effects rather than client drift, all logical workers sample from the same training distribution.
Training uses random crops with padding \(4\), random horizontal flips, weight decay \(5\cdot 10^{-4}\), dropout \(0.1\), batch size \(256\), evaluation batch size \(512\), and \(10\%\) validation data.
The runs use four logical workers with step times \((1,2,3,6)\), learning rate \(0.05\), seed \(11\), target budget \(400\) epochs, and simulated 32-bit communicated values.

We consider two communication regimes.
The normal regime uses sparsity \(p=0.25\), local-computation parameter \(M=4\), and communication delay \(\zeta=6\).
The communication-stress regime uses more aggressive sparsity and a longer delay, with \(p=0.10\), \(M=2\), and \(\zeta=24\).
The stress regime is designed to make communication a larger part of the logical-time budget.

\Cref{fig:cifar-normal,fig:cifar-comm-stress} show the CIFAR-10 results.
In both regimes, overlap improves training and validation accuracy as a function of logical time.
The advantage is clearer in the communication-stress setting, where blocking communication is more expensive.
The processed-example plots provide a complementary check.
Because overlap workers keep taking mini-batch steps while sparse communication is in flight, an overlap method can process more examples by a fixed logical time than blocking \textbf{Local Sparse}.
Plotting against cumulative processed examples removes this timing advantage and compares methods after they have consumed the same amount of training data, which also roughly equalizes the number of local SGD steps and the amount of local compute.
On this axis, the methods are much closer, especially in the normal regime.
Importantly, the overlap variants do not lag behind or become unstable when compared this way.
Thus the additional local steps taken during communication are not merely buying speed by spending extra computation in a harmful way; they appear to remain useful optimization steps.
Together with the logical-time plots, this supports the same interpretation as in the logistic-regression experiments: overlap turns communication waiting time into productive local training, and the main gain is in elapsed logical time rather than in changing the amount of data needed to reach a given accuracy.

\begin{figure}[H]
    \centering
    \begin{minipage}{0.48\linewidth}
        \centering
        \includegraphics[width=\linewidth]{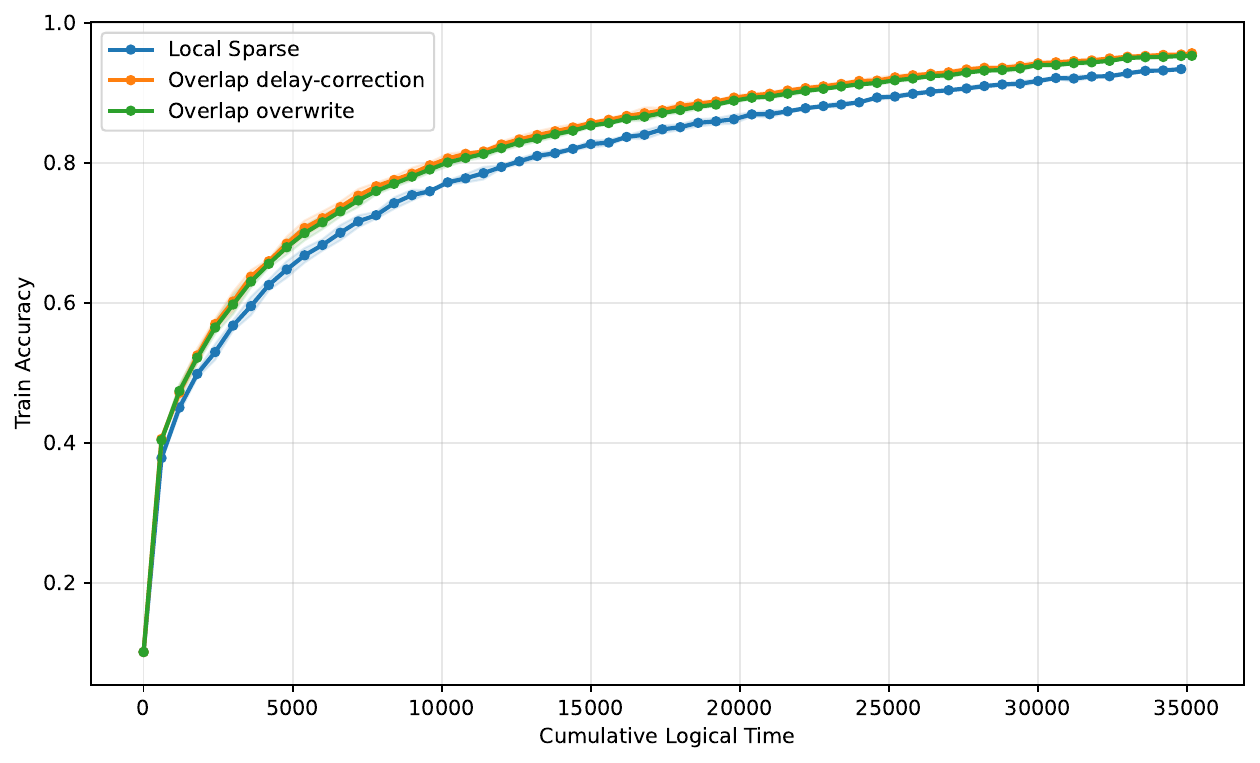}\\
        {\small (a) Training accuracy versus logical time.}
    \end{minipage}\hfill
    \begin{minipage}{0.48\linewidth}
        \centering
        \includegraphics[width=\linewidth]{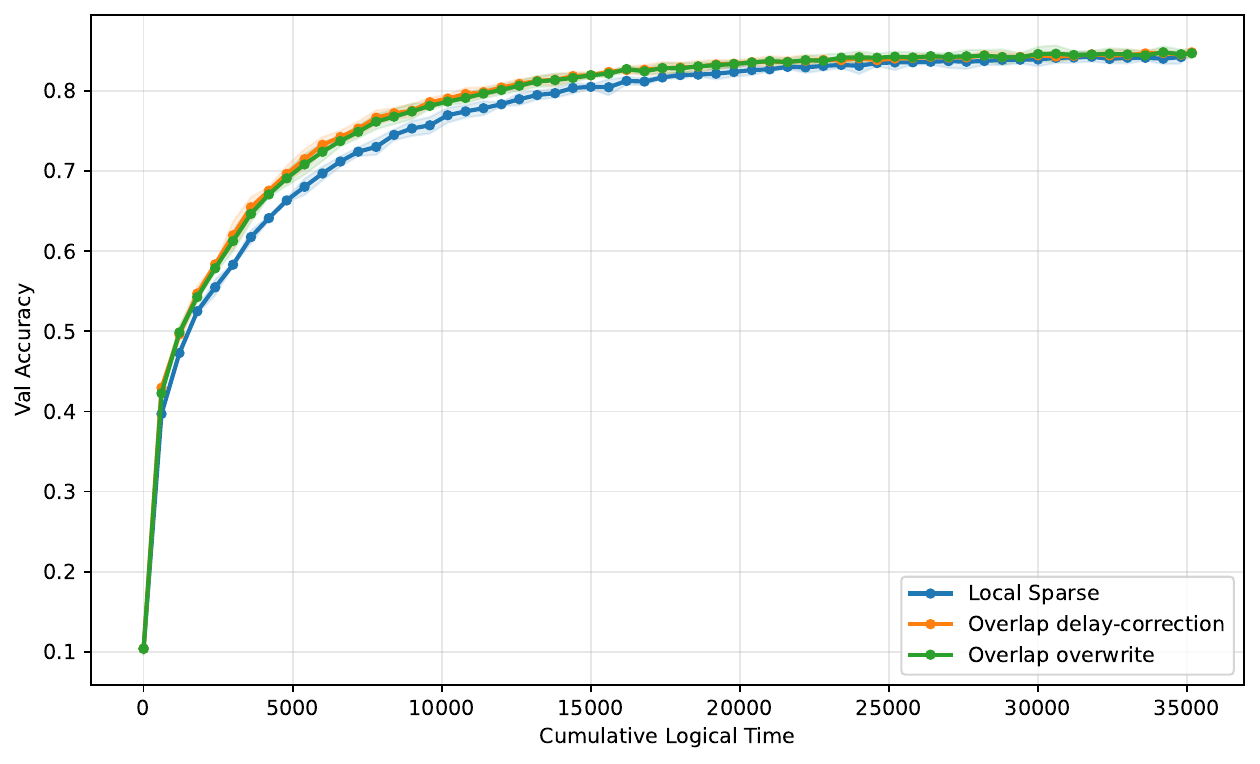}\\
        {\small (b) Validation accuracy versus logical time.}
    \end{minipage}

    \begin{minipage}{0.62\linewidth}
        \centering
        \includegraphics[width=\linewidth]{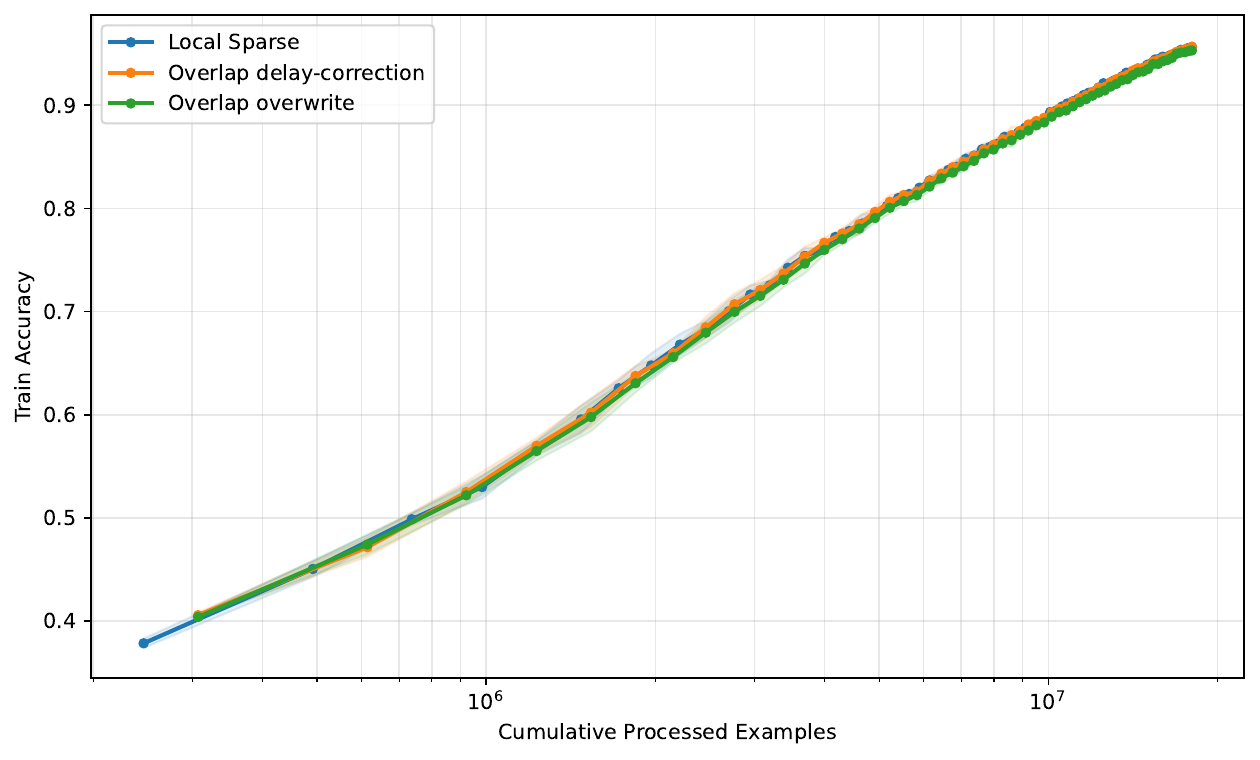}\\
        {\small (c) Training accuracy versus processed examples.}
    \end{minipage}
    \caption{CIFAR-10 convolutional-network experiment in the normal communication regime, with \(p=0.25\), \(M=4\), and \(\zeta=6\).
    The overlap methods improve over blocking sparse averaging in logical time, while the processed-example curves show that this extra overlapped computation does not reduce accuracy at a fixed amount of data processed.}
    \label{fig:cifar-normal}
\end{figure}

\begin{figure}[H]
    \centering
    \begin{minipage}{0.48\linewidth}
        \centering
        \includegraphics[width=\linewidth]{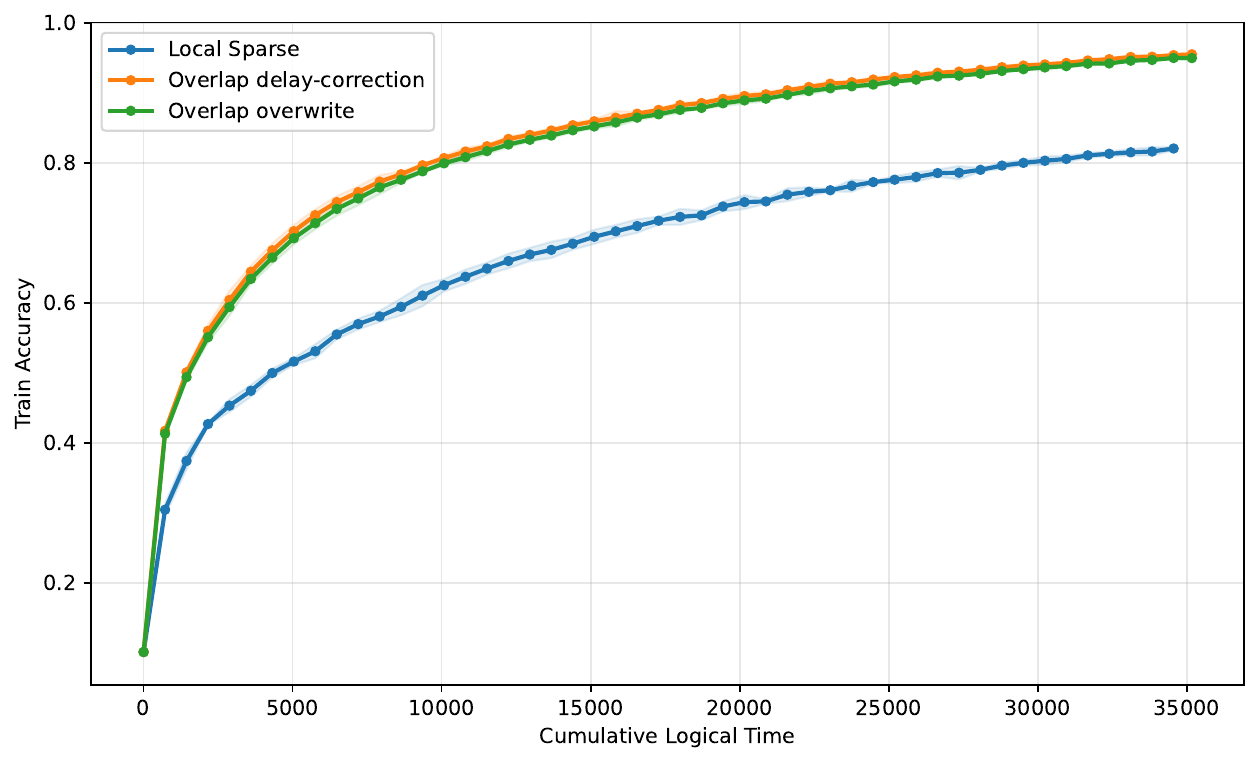}\\
        {\small (a) Training accuracy versus logical time.}
    \end{minipage}\hfill
    \begin{minipage}{0.48\linewidth}
        \centering
        \includegraphics[width=\linewidth]{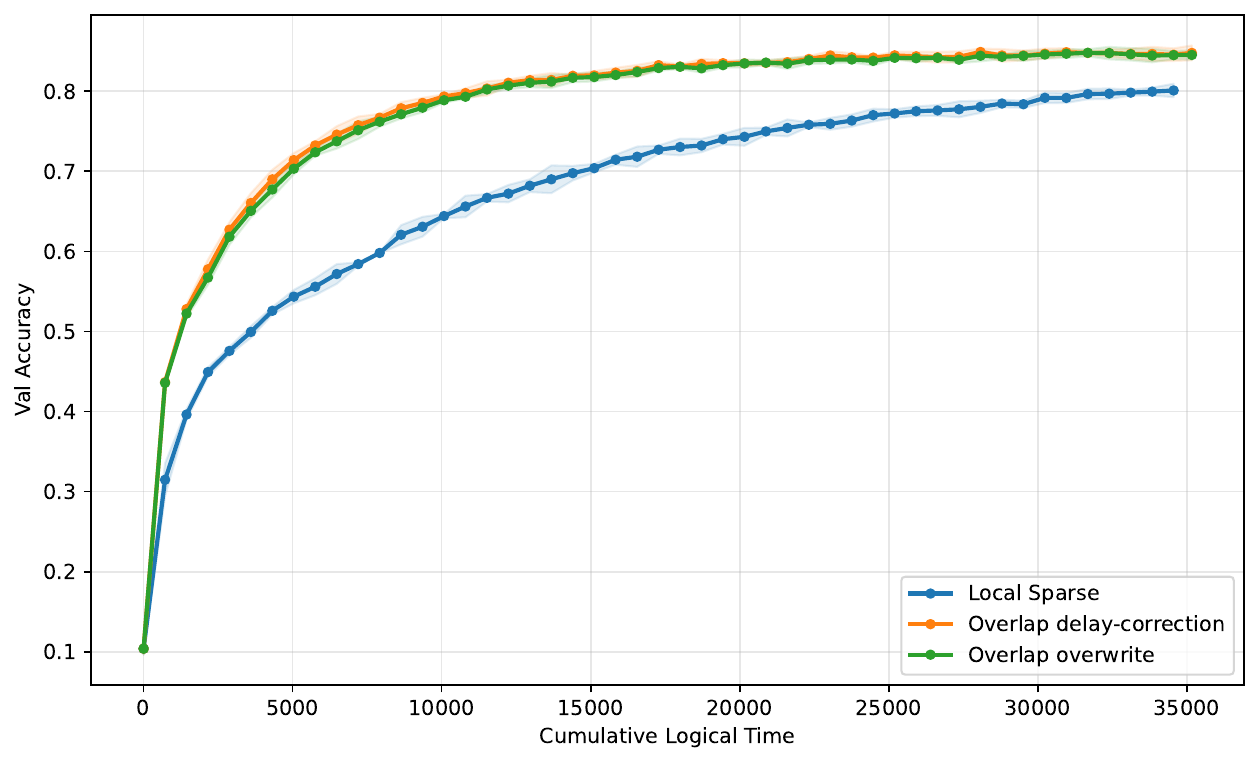}\\
        {\small (b) Validation accuracy versus logical time.}
    \end{minipage}

    \begin{minipage}{0.62\linewidth}
        \centering
        \includegraphics[width=\linewidth]{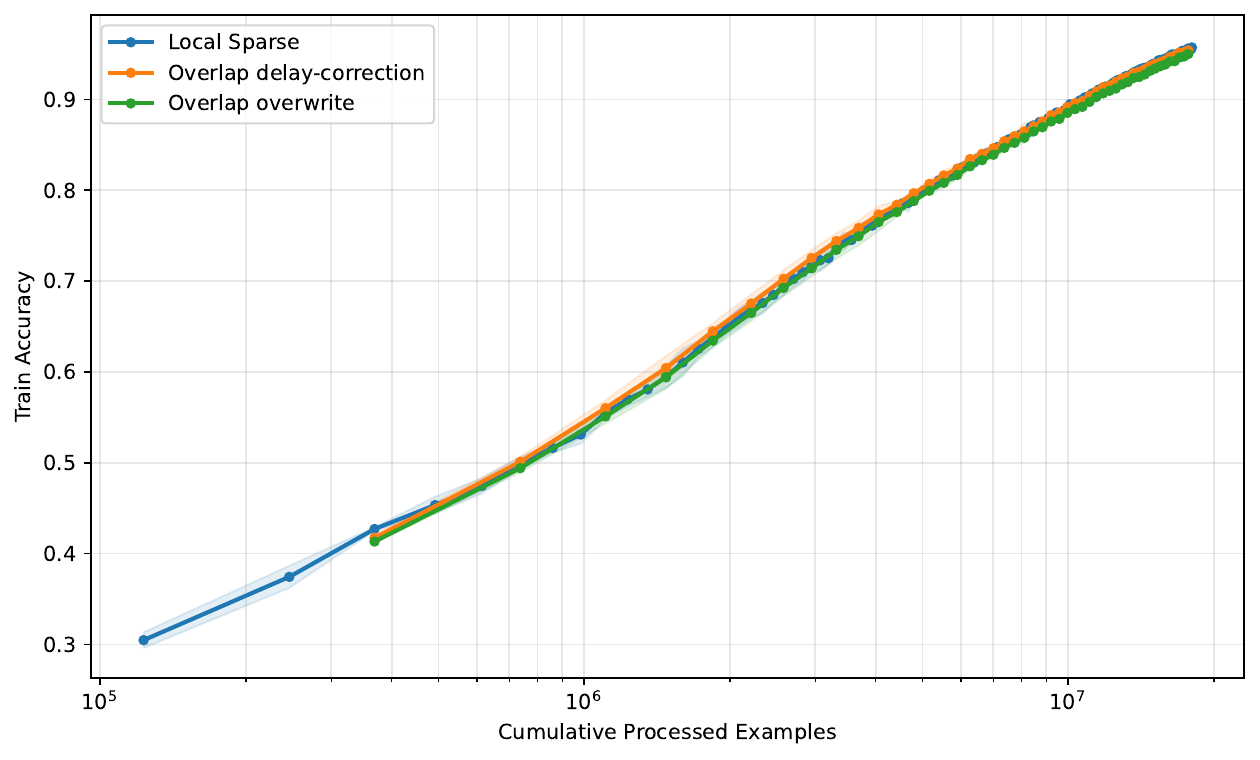}\\
        {\small (c) Training accuracy versus processed examples.}
    \end{minipage}
    \caption{CIFAR-10 convolutional-network experiment in the communication-stress regime, with \(p=0.10\), \(M=2\), and \(\zeta=24\).
    The larger communication delay makes the logical-time benefit of overlap more pronounced, while the processed-example plot checks that the additional local steps remain useful when compared at the same number of consumed training examples.}
    \label{fig:cifar-comm-stress}
\end{figure}

\subsection{Tiny ImageNet neural-network experiments}
\label{subsec:tiny-imagenet-neural-network}

We additionally evaluate the sparse overlap methods on Tiny ImageNet to test the same mechanisms on a larger image-classification problem.
Tiny ImageNet contains \(200\) image classes at resolution \(64\times 64\).
We use the official training split for optimization, hold out \(10\%\) of it for validation, and reserve the official validation split for final test evaluation.
All images are normalized using ImageNet statistics.

The model is a ResNet-18-style convolutional network trained from scratch \citep{he2016deep}.
Because the input images are \(64\times 64\), we replace the standard ImageNet stem by a \(3\times 3\) convolution with stride \(1\) and remove the initial max-pooling layer.
We use GroupNorm instead of BatchNorm, since BatchNorm running statistics are not naturally synchronized by our sparse parameter-averaging implementation.
The final classifier is replaced by a \(200\)-class linear layer, and dropout is disabled.

All methods use the same training recipe: \(100\) effective epochs, batch size \(128\), SGD with momentum \(0.9\), weight decay \(5\cdot 10^{-4}\), initial learning rate \(0.1\), and cosine decay to zero.
Training uses random crops with padding \(4\) and random horizontal flips; evaluation uses deterministic \(64\times 64\) images without augmentation.
As in the CIFAR-10 experiments, we use four logical workers with step times \((1,2,3,6)\) and a shared data partition, so each worker samples from the same training distribution.
We compare \textbf{Local Sparse}, \textbf{Overlap overwrite}, and \textbf{Overlap delay-corrected} under two communication regimes.
The normal regime uses \(p=0.25\), overlap parameters \(M=4\) and \(\zeta=6\), and the corresponding blocking baseline uses compute window \(24\) and communication time \(6\). 
The communication-stress regime uses stronger sparsification and delay, with \(p=0.10\), overlap parameters \(M=2\) and \(\zeta=24\), and the corresponding blocking baseline uses compute window \(12\) and communication time \(24\).

\Cref{fig:tin-normal,fig:tin-comm-stress} show the results.
The main conclusion is again that overlap is the dominant effect.
In the normal regime, both overlap variants improve substantially over blocking sparse averaging in logical time and in rounds, while the processed-example plot shows that the methods have similar sample efficiency once the extra computation performed during communication is accounted for.
In the communication-stress regime, the logical-time gap becomes much larger: blocking \textbf{Local Sparse} spends so much time waiting that it reaches much lower training and validation accuracy within the same time budget.
The processed-example plot is more surprising.
This axis removes the wall-clock advantage of overlap, but it does not equalize the number of sparse merge operations.
In the stress regime, the overlap methods perform the same pre-communication steps as \textbf{Local Sparse}, and then two more such blocks of local computation during communication.
Thus, for a comparable number of processed examples, \textbf{Local Sparse} has applied roughly three times as many sparse averaging operations.
We did not expect \textbf{Local Sparse} to be worse on this axis, and we do not yet fully understand the effect.
One plausible explanation is that, with only \(p=0.10\) communicated coordinates and a deep network, very frequent partial-coordinate merges can perturb the optimization trajectory, while the overlap methods implicitly use a larger local-computation window between sparse synchronizations.

Unlike the logistic-regression merge-rule ablations, Tiny ImageNet shows little separation between \textbf{Overlap overwrite} and \textbf{Overlap delay-corrected}.
The two overlap curves are nearly indistinguishable in the stress regime, and overwrite is slightly ahead in some normal-regime curves.
Thus, these experiments should be read primarily as evidence for the robustness of communication--computation overlap on a more demanding neural-network task, rather than as evidence that one sparse merge rule dominates on every architecture and training recipe.

\begin{figure}[H]
    \centering
    \begin{minipage}{0.48\linewidth}
        \centering
        \includegraphics[width=\linewidth]{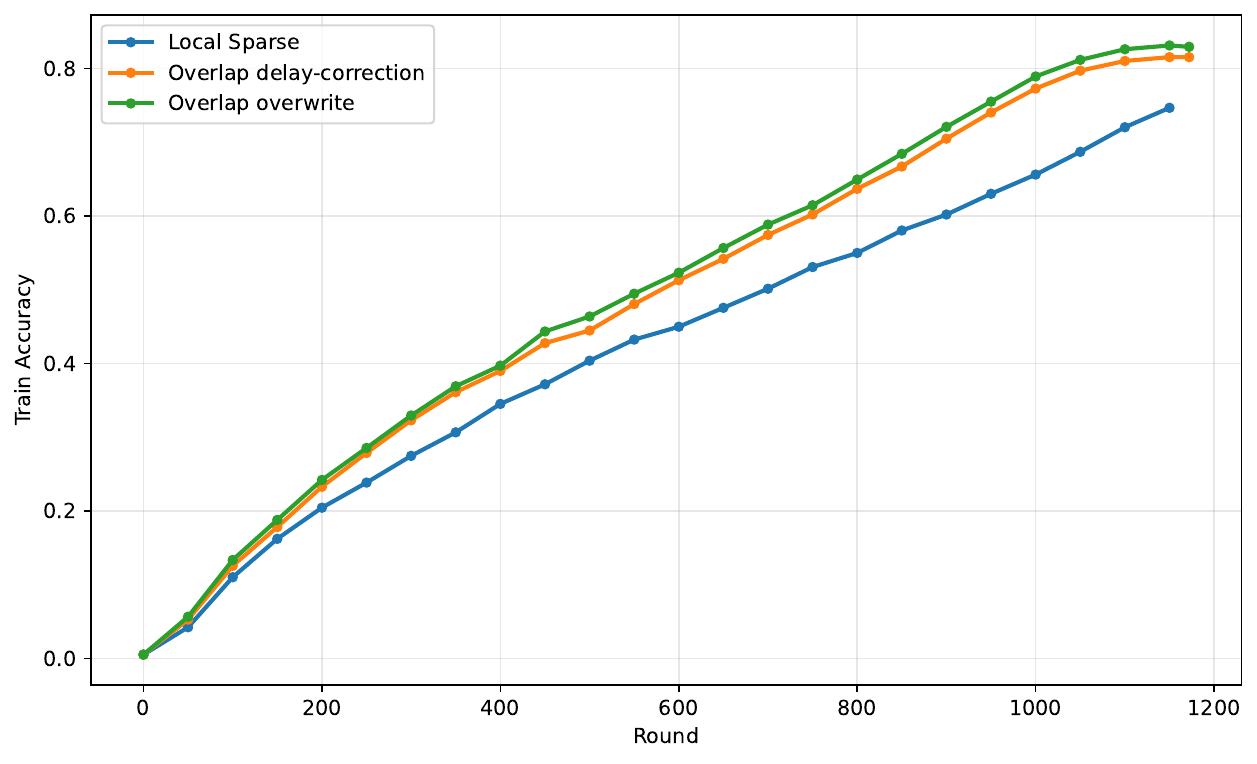}\\
        {\small (a) Training accuracy versus round.}
    \end{minipage}\hfill
    \begin{minipage}{0.48\linewidth}
        \centering
        \includegraphics[width=\linewidth]{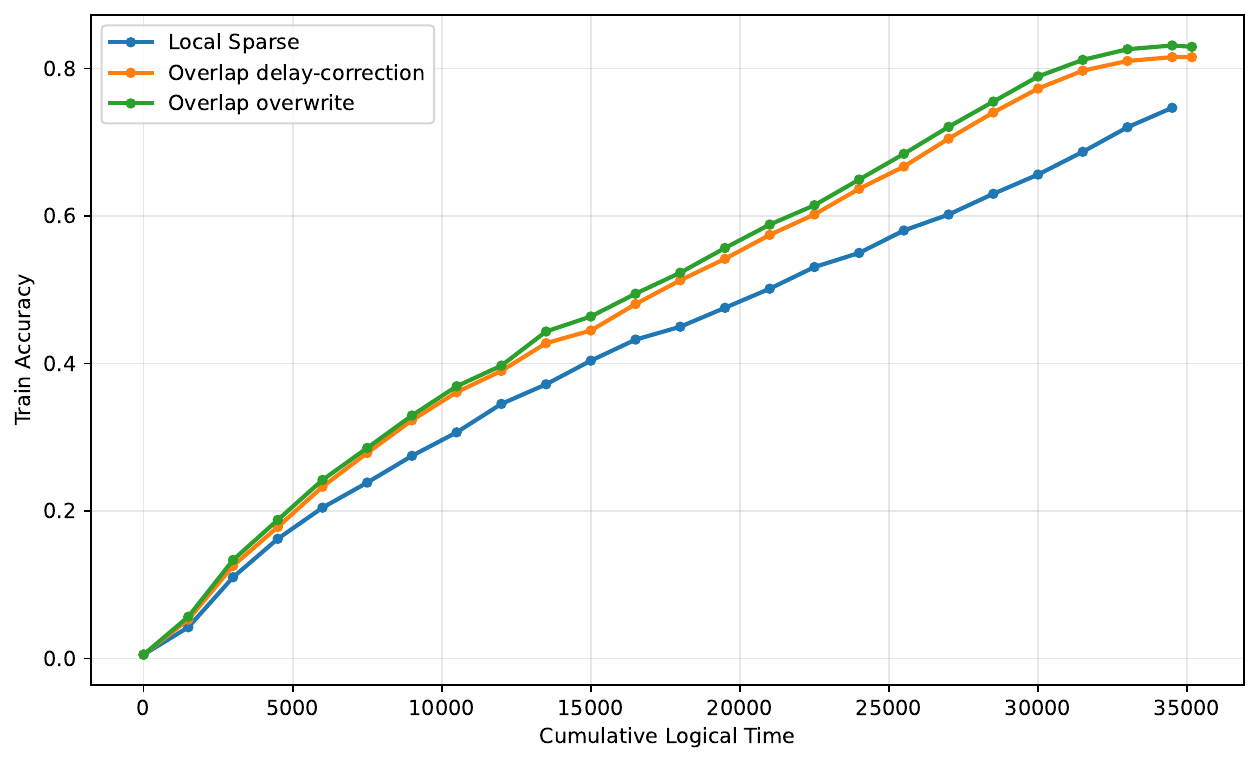}\\
        {\small (b) Training accuracy versus logical time.}
    \end{minipage}

    \begin{minipage}{0.48\linewidth}
        \centering
        \includegraphics[width=\linewidth]{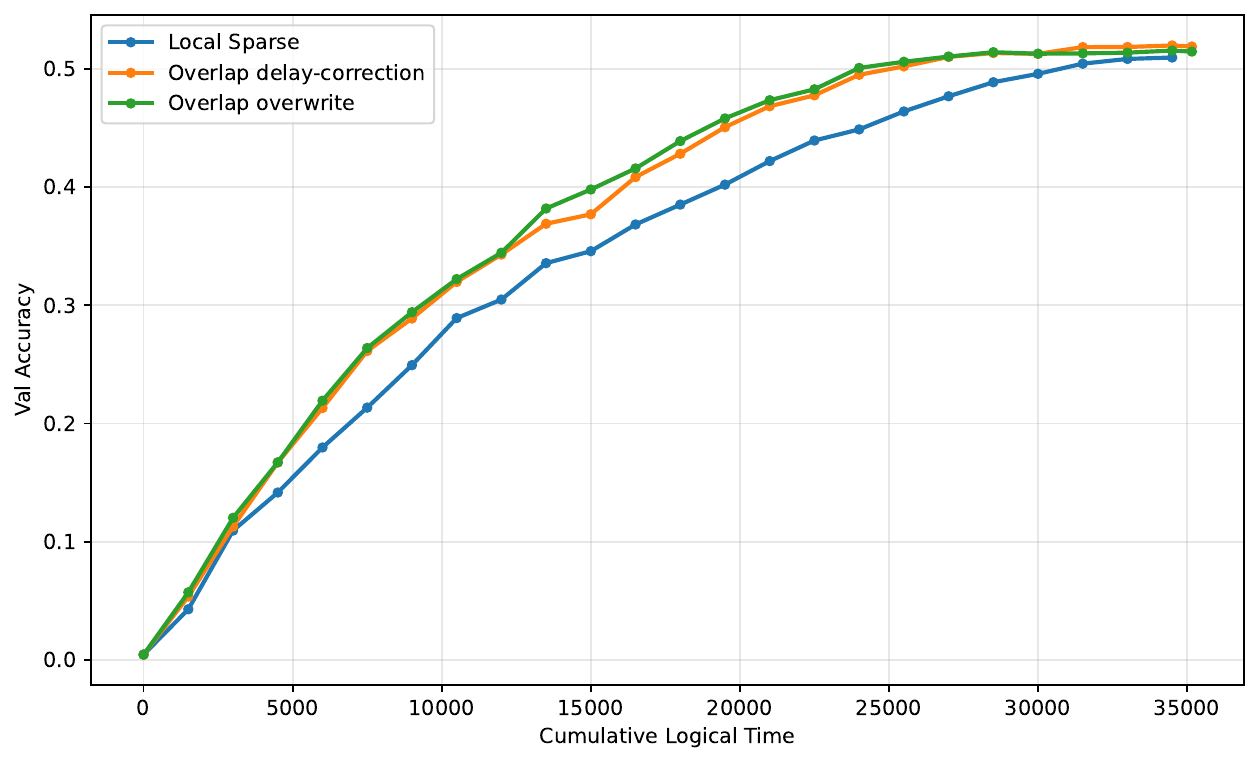}\\
        {\small (c) Validation accuracy versus logical time.}
    \end{minipage}\hfill
    \begin{minipage}{0.48\linewidth}
        \centering
        \includegraphics[width=\linewidth]{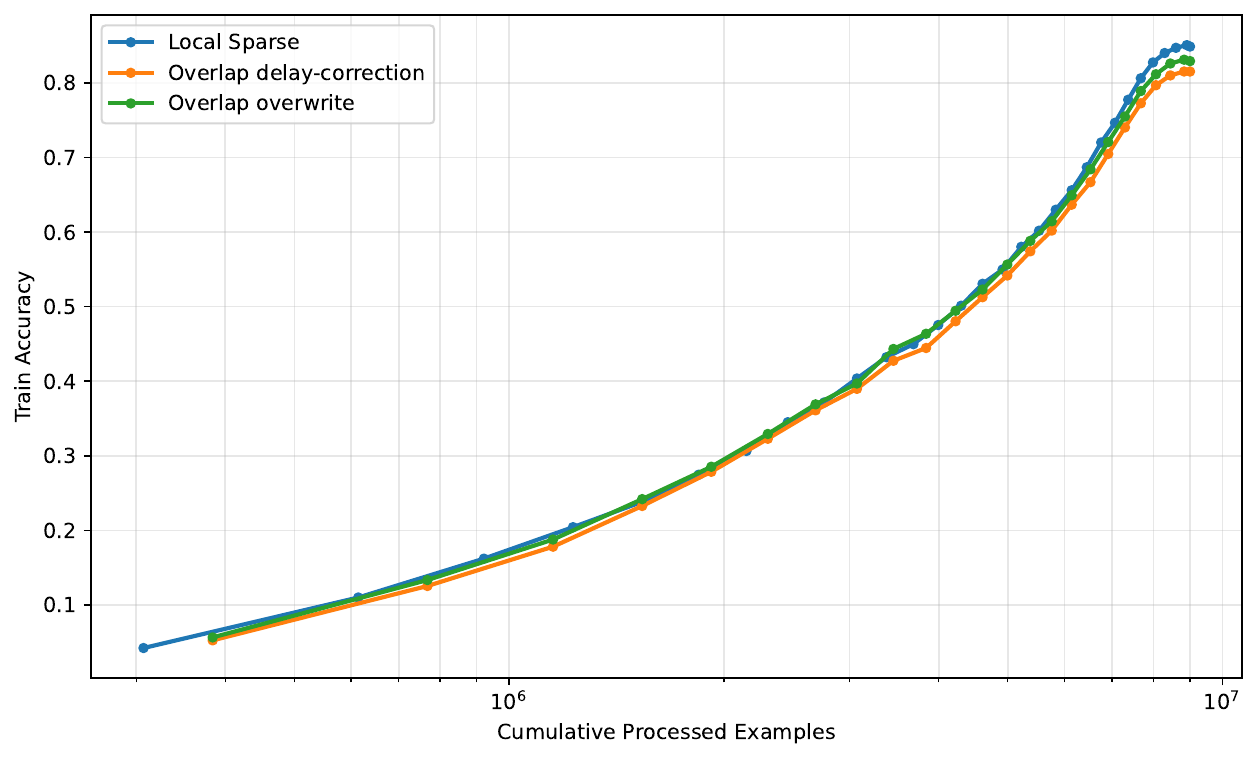}\\
        {\small (d) Training accuracy versus processed examples.}
    \end{minipage}
    \caption{Tiny ImageNet experiment in the normal communication regime, with \(p=0.25\), \(M=4\), and \(\zeta=6\).
    Both overlap variants improve over blocking sparse averaging in logical time and in rounds.
    When plotted against processed examples, the curves are much closer, indicating that the main benefit is the use of communication time for additional useful computation.}
    \label{fig:tin-normal}
\end{figure}

\begin{figure}[H]
    \centering
    \begin{minipage}{0.48\linewidth}
        \centering
        \includegraphics[width=\linewidth]{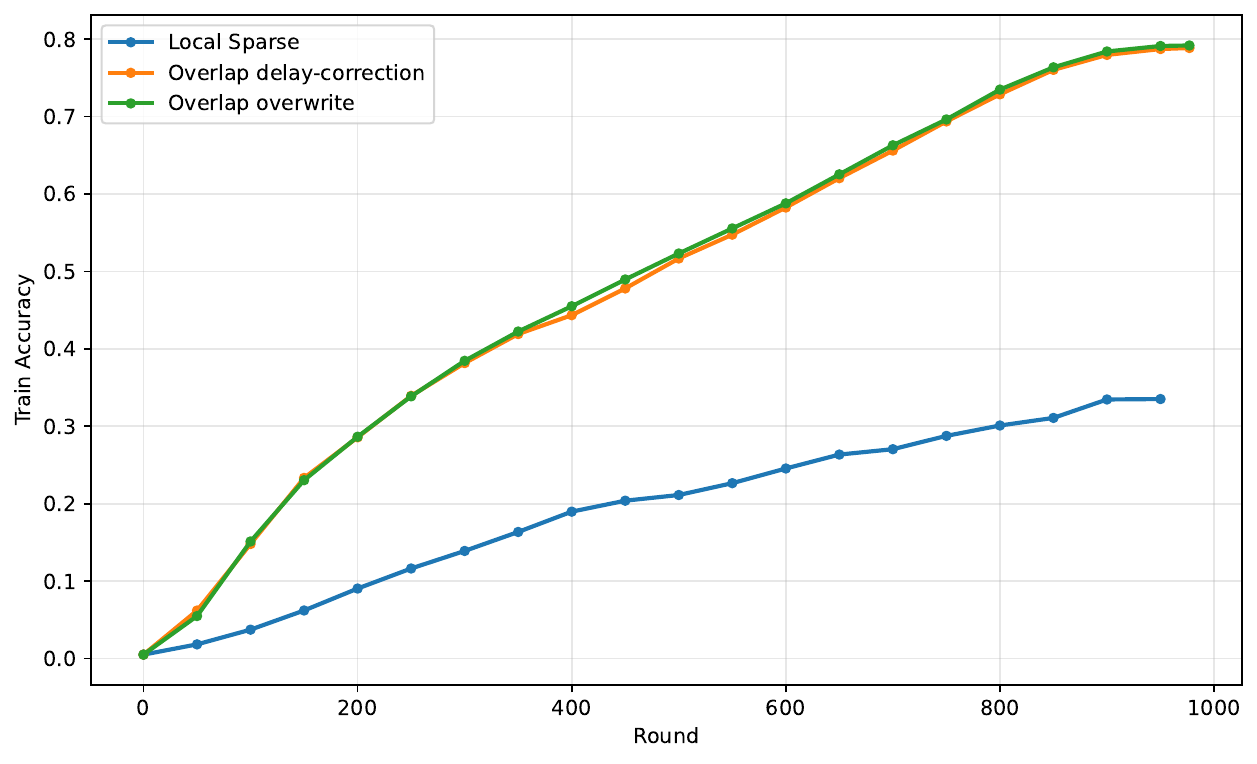}\\
        {\small (a) Training accuracy versus round.}
    \end{minipage}\hfill
    \begin{minipage}{0.48\linewidth}
        \centering
        \includegraphics[width=\linewidth]{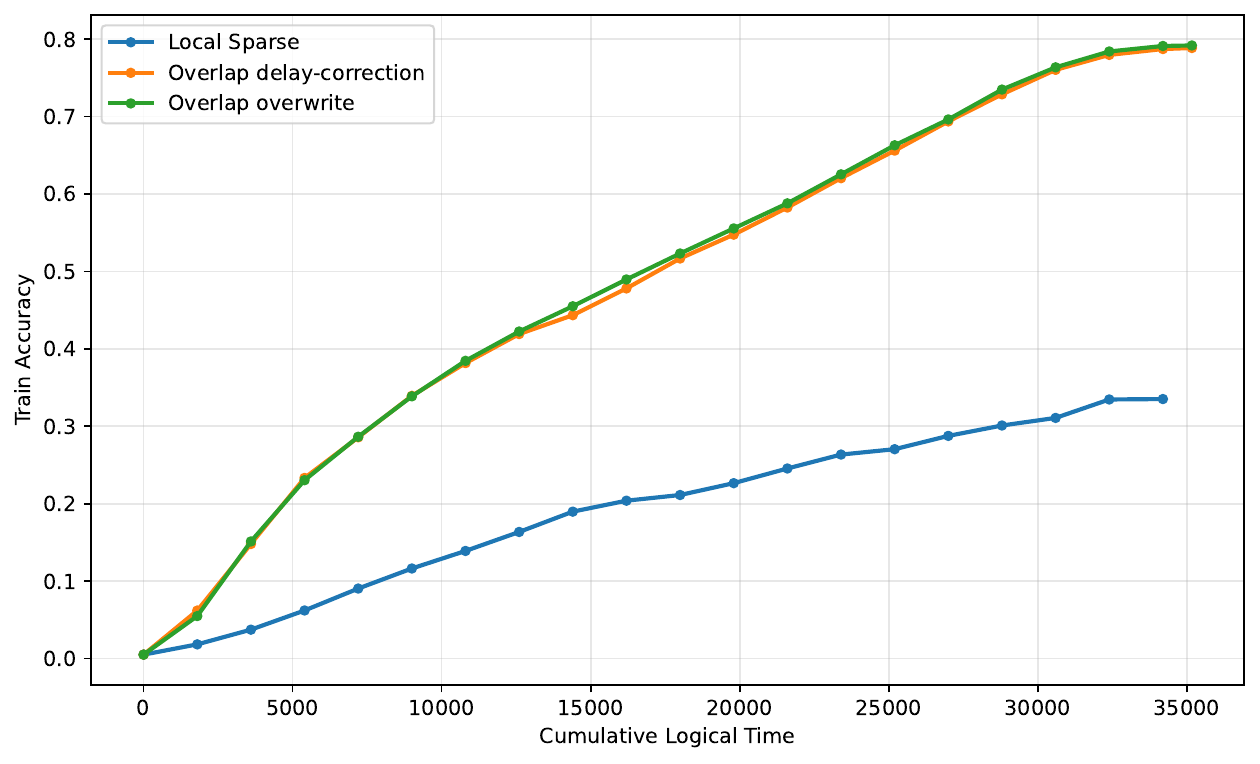}\\
        {\small (b) Training accuracy versus logical time.}
    \end{minipage}

    \begin{minipage}{0.48\linewidth}
        \centering
        \includegraphics[width=\linewidth]{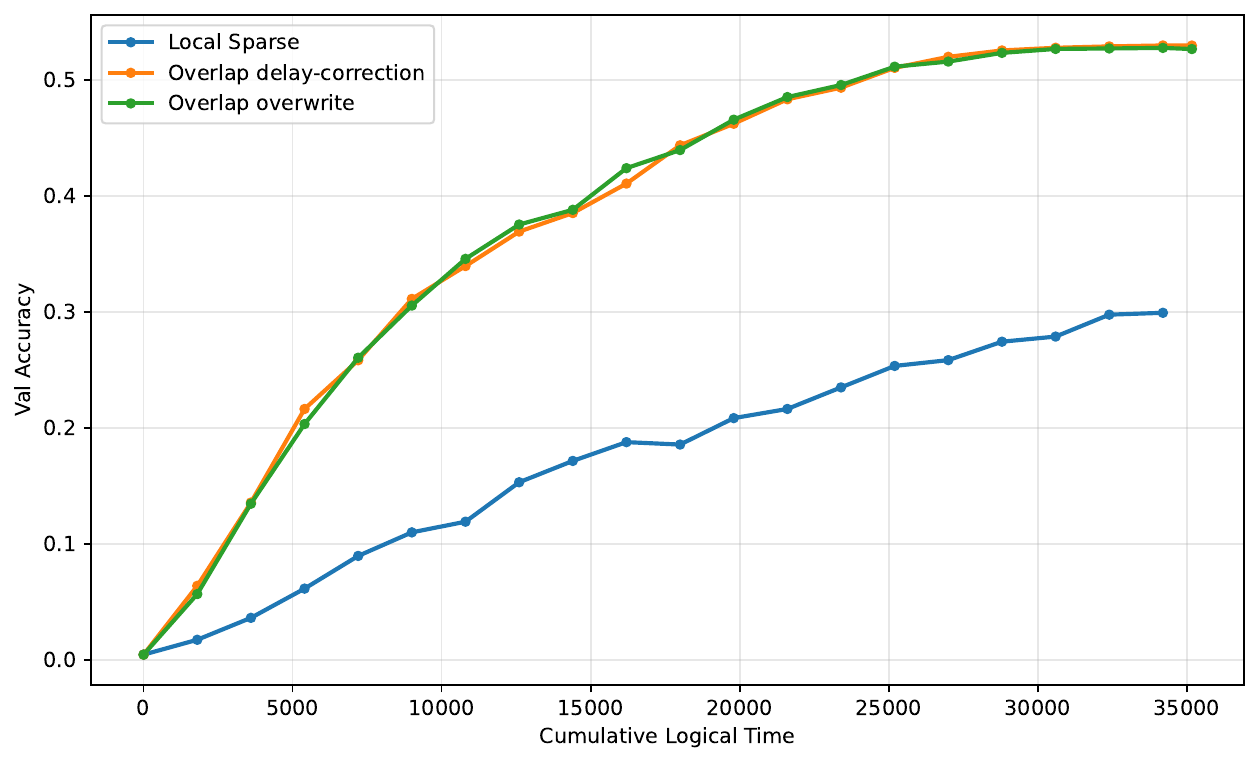}\\
        {\small (c) Validation accuracy versus logical time.}
    \end{minipage}\hfill
    \begin{minipage}{0.48\linewidth}
        \centering
        \includegraphics[width=\linewidth]{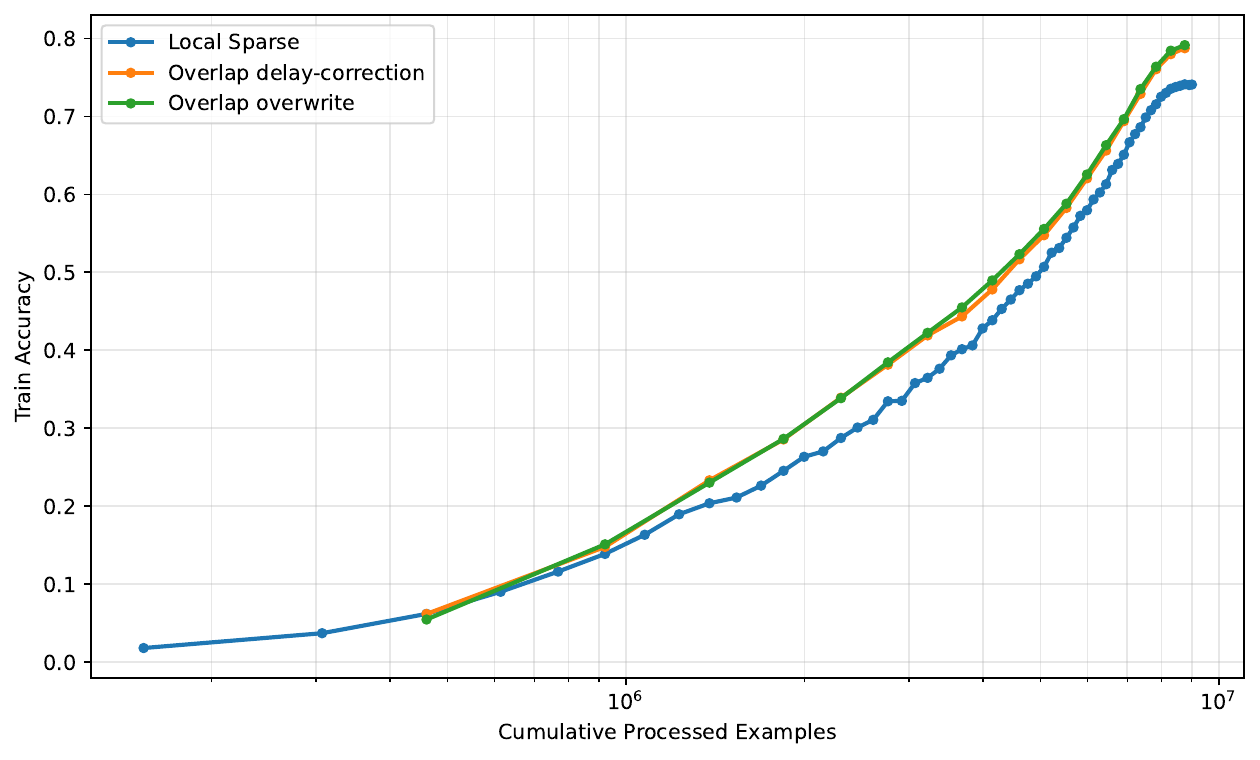}\\
        {\small (d) Training accuracy versus processed examples.}
    \end{minipage}
    \caption{Tiny ImageNet experiment in the communication-stress regime, with \(p=0.10\), \(M=2\), and \(\zeta=24\).
    The longer communication delay makes the advantage of overlap much larger: both overlap variants substantially outperform blocking local sparse averaging in logical time and also remain ahead when compared by processed examples.}
    \label{fig:tin-comm-stress}
\end{figure}

\end{document}